\documentclass[11pt]{article}
\usepackage[utf8]{inputenc}

\usepackage[margin=1in]{geometry}
\usepackage{xcolor}
\usepackage{amsmath}
\usepackage{amssymb}
\usepackage{amsthm}
\usepackage{algorithm}
\usepackage{algpseudocode}
\usepackage{hyperref}
\usepackage{graphicx}
\usepackage[font={small}]{caption}
\usepackage{thm-restate}
\usepackage[normalem]{ulem}
\usepackage{xspace}
\usepackage{arydshln}
\usepackage{wrapfig}
\usepackage{enumitem}

\newcommand{\nc}{\newcommand}
\newcommand{\DMO}{\DeclareMathOperator}

\newcommand{\remove}[1]{}

\algnotext{EndFor}
\algnotext{EndIf}
\algnotext{EndWhile}
\algnotext{EndProcedure}

\allowdisplaybreaks

\newcommand{\poly}{\mathrm{poly}}
\nc{\MS}{\mathcal{S}}
\nc{\MP}{\mathcal{P}}
\nc{\MR}{\mathcal{R}}

\nc{\MZ}{\mathcal{Z}}
\DMO{\Binom}{Binom}
\newcommand{\E}{\mathbb{E}}
\DMO{\Var}{Var}
\nc{\Geom}{\text{Geom}}
\newcommand{\ti}{\tilde{i}}
\newcommand{\tx}{\tilde{x}}
\newcommand{\talpha}{\tilde{\alpha}}
\newcommand{\hr}{\hat{r}}
\newcommand{\tr}{\tilde{r}}
\newcommand{\tbX}{\tilde{\bX}}
\newcommand{\tpsi}{\tilde{\psi}}
\newcommand{\tc}{\tilde{c}}

\newcommand{\btc}{\mathbf{\tc}}
\newcommand{\bc}{\mathbf{c}}
\newcommand{\bX}{\mathbf{X}}
\newcommand{\bC}{\mathbf{C}}

\newcommand{\bw}{w}

\newcommand{\bx}{x}

\newcommand{\by}{y}

\newcommand{\bv}{v}

\newcommand{\bS}{\mathbf{S}}

\newcommand{\bz}{z}

\newcommand{\bb}{b}

\newcommand{\bB}{B}

\newcommand{\R}{\mathbb{R}}
\newcommand{\N}{\mathbb{N}}

\nc{\BN}{\mathbb{N}}
\newcommand{\Z}{\mathbb{Z}}
\nc{\BZ}{\mathbb{Z}}

\newcommand{\cV}{\mathcal{V}}
\newcommand{\cA}{\mathcal{A}}

\newcommand{\bone}{\mathbf{1}}

\newcommand{\bzero}{0}

\newcommand{\eps}{\varepsilon}
\nc{\esin}{\eps}

\newcommand{\cH}{\mathcal{H}}
\newcommand{\cL}{\mathcal{L}}
\newcommand{\cO}{\mathcal{O}}
\newcommand{\cB}{\mathcal{B}}
\newcommand{\cU}{\mathcal{U}}
\newcommand{\cM}{\mathcal{M}}
\newcommand{\cS}{\mathcal{S}}

\newcommand{\cC}{\mathcal{C}}
\newcommand{\cX}{\mathcal{X}}

\DeclareMathOperator{\cost}{cost}

\DeclareMathOperator{\argmin}{argmin}
\DeclareMathOperator{\opt}{OPT}
\DeclareMathOperator{\DLap}{DLap}

\DeclareMathOperator{\err}{err}
\DeclareMathOperator{\sgn}{sgn}
\renewcommand{\varepsilon}{\epsilon}

\newtheorem{theorem}{Theorem}

\newtheorem{observation}[theorem]{Observation}
\newtheorem{lemma}[theorem]{Lemma}

\newtheorem{definition}[theorem]{Definition}

\newtheorem{corollary}[theorem]{Corollary}

\newtheorem{fact}[theorem]{Fact}

\newcommand{\kmeans}{\textsf{%\small
$k$-means}\xspace}
\newcommand{\kmedian}{\textsf{%\small
$k$-median}\xspace}
\newcommand{\densestball}{\textsf{%\small
DensestBall}\xspace}
\newcommand{\minimumenclosingball}{\textsf{%\small
MinimumEnclosingBall}\xspace}
\newcommand{\minimumboundingsphere}{\textsf{%\small
MinimumBoundingSphere}\xspace}
\newcommand{\onecluster}{\textsf{%\small
$1$-Cluster}\xspace}
\newcommand{\onecenter}{\textsf{%\small
$1$-Center}\xspace}
\newcommand{\goodradius}{\textsc{%\small
GoodRadius}\xspace}
\newcommand{\cp}{\textsf{%\small
ClosestPair}\xspace}
\newcommand{\SparseSelection}{\textsf{%\small
SparseSelection}\xspace}
\newcommand{\Selection}{\textsf{%\small
Selection}\xspace}
\newcommand{\score}{\mathit{score}}

\title{Differentially Private Clustering: \\ Tight Approximation Ratios}
\author{
  Badih Ghazi\thanks{Email: {\tt badihghazi@gmail.com}}
  \hspace*{1cm}
  Ravi Kumar\thanks{Email: {\tt ravi.k53@gmail.com}}
  \hspace*{1cm}
  Pasin Manurangsi\thanks{Email: {\tt pasin@google.com}}  \\
  Google Research \\
  Mountain View, CA.
}

\begin{document}

\maketitle

\begin{abstract}
We study the task of differentially private clustering.  For several
basic clustering problems, including Euclidean \densestball,
\onecluster, \kmeans, and \kmedian, we give efficient differentially
private algorithms that achieve essentially the \emph{same} approximation
ratios as those that can be obtained by any non-private algorithm,
while incurring only small additive errors.  This improves upon
existing efficient algorithms that only achieve some large constant approximation factors. 

Our results also imply an improved algorithm for the Sample and Aggregate privacy framework. Furthermore, we show that one of the tools used in our \onecluster algorithm can be employed to get a faster quantum algorithm for \cp in a moderate number of dimensions.

\end{abstract}

\section{Introduction}\label{sec:intro}

%A rich body of work has aimed at designing efficient, accurate clustering algorithms, leading to popular practical procedures such as Lloyd’s algorithm
%\cite{lloyd1982least} and $k$-means$++$ \cite{arthur2006k} (see the textbooks \cite{xu2008clustering, charu2013data}).

With the significant increase in data collection, serious concerns about user privacy have emerged. This has stimulated research on formalizing and guaranteeing strong privacy protections for user-sensitive information. Differential Privacy (DP)~\cite{dwork2006calibrating,DworkKMMN06} is a rigorous mathematical concept for studying user privacy and has been widely adopted in practice~\cite{erlingsson2014rappor,CNET2014Google, greenberg2016apple, dp2017learning, ding2017collecting, abowd2018us}. 
%by tech companies such as Google~\cite{erlingsson2014rappor,CNET2014Google}, Apple~\cite{greenberg2016apple,dp2017learning}, and Microsoft~\cite{ding2017collecting}, and federal agencies such as the U.S.~Census Bureau \cite{abowd2018us}. 
Informally, the notion of privacy is that the algorithm's output (or output distribution) should be mostly unchanged when any one of its inputs is changed.  DP is quantified by two parameters $\epsilon$ and $\delta$; the resulting notion is referred to as pure-DP when $\delta = 0$, and approximate-DP when $\delta > 0$. See Section~\ref{sec:prelim} for formal definitions of DP and~\cite{DworkR14, vadhan2017complexity} for an overview.

Clustering is a central primitive in unsupervised machine learning~\cite{xu2008clustering, charu2013data}.  An algorithm for clustering in the DP model informally means that the cluster centers (or the distribution on cluster centers) output by the algorithm should be mostly unchanged when any one of the input points is changed. Many real-world applications involve clustering sensitive data.  Motivated by these, a long line of work has studied clustering algorithms in the DP model~\cite{blum2005practical, nissim2007smooth, feldman2009private, GuptaLMRT10, mohan2012gupt, wang2015differentially, NissimSV16, nock2016k, su2016differentially, feldman2017coresets, BalcanDLMZ17, NissimS18, huang2018optimal,nock2016k,NissimS18, StemmerK18,Stemmer20}.
%(we later elaborate on these in our related work Section~\ref{subsec:related_work}).
In this work we focus on several basic clustering problems in the DP model and obtain efficient algorithms with tight approximation ratios.

\paragraph{Clustering Formulations.}

The input to all our problems is a set $X$ of $n$ points, each contained in the $d$-dimensional unit ball.  There are many different formulations of clustering.  In the popular \kmeans problem~\cite{lloyd1982least}, the goal is to find $k$ centers minimizing the clustering cost, which is the sum of squared distances from each point to its closest center.  The \kmedian problem is similar to \kmeans except that the distances are not squared in the definition of the clustering cost.\footnote{For the formal definitions of \kmeans and \kmedian, see Definition~\ref{def:vanilla-clustering} and the paragraph following it.} Both problems are NP-hard, and there is a large body of work dedicated to determining the best possible approximation ratios achievable in polynomial time (e.g.~\cite{Bartal96,CharikarCGG98,CharikarGTS02,JainV01,JainMS02,AryaGKMMP04,KanungoMNPSW04,ArthurV07,LiS16,AwasthiCKS15,ByrkaPRST17,LeeSW17,AhmadianNSW17,Cohen-AddadS19}), although the answers remain elusive.  We consider approximation algorithms for both these problems in the DP model, where a $(w, t)$-approximation algorithm outputs a cluster whose cost is at most the sum of $t$ and $w$ times the optimum; we refer to $w$ as the \emph{approximation ratio} and $t$ as the \emph{additive error}. It is important that $t$ is small since without this constraint, the problem could become trivial.  (Note also that without privacy constraints, approximation algorithms typically work with $t = 0$.)

We also study two even more basic clustering primitives, \densestball and \onecluster, in the DP model. These underlie several of our results.
\begin{definition}[\densestball]\label{def:densest_ball_intro}
Given $r > 0$,  a $(w, t)$-approximation for the
\emph{\densestball} problem is a ball $B$ of radius $w \cdot r$ such that whenever there is a ball of radius $r$ that contains at least $T$ input points, $B$ contains at least $T - t$ input points.
\end{definition}
This problem is NP-hard for $w=1$~\cite{Ben-DavidS00, Ben-DavidES02,Shenmaier15}.  Moreover, approximating the largest number of points within any ball of radius of $r$ and up some constant factor is also NP-hard~\cite{Ben-DavidES02}. On the other hand, several polynomial-time approximation algorithms achieving $(1 + \alpha, 0)$-approximation for any $\alpha > 0$ are known~\cite{agarwal2005geometric,shenmaier2013problem,Ben-DavidES02}.

\densestball is a useful primitive since a DP algorithm for it allows one to ``peel off'' one important cluster at a time.
This approach has played a pivotal role in a recent fruitful line of research that obtains DP approximation algorithms for \kmeans and \kmedian~\cite{StemmerK18,Stemmer20}. 

The \onecluster problem studied, e.g., in~\cite{NissimSV16,NissimS18} is the ``inverse'' of \densestball, where instead of the radius $r$, the target number $T$ of points inside the ball is given. Without DP constraints, the computational complexities of these two problems are essentially the same (up to logarithmic factors in the number of points and the input universe size), as we may use binary search on $r$ to convert a \densestball algorithm into one for \onecluster, and vice versa.\footnote{To reduce from \onecluster to \densestball, one can binary-search on the target radius.
In this case, the number of iterations needed for the binary search depends logarithmically on the ratio between the maximum possible distance between two input points and the minimum possible distance between two (distinct) input points. In the other direction (i.e., reducing from \densestball to \onecluster), one can binary-search on the number of points inside the optimal ball, and here the number of iterations will be logarithmic in the number of input points.} These two problems are generalizations of the \minimumenclosingball (aka \minimumboundingsphere) problem, which is well-studied in statistics, operations research, and computational geometry.  

As we elaborate below, \densestball and \onecluster are also related to other well-studied problems, such as learning halfspaces with a margin and the Sample and Aggregate framework ~\cite{nissim2007smooth}.

\paragraph{Main Results.}
A common highlight of most of our results is that for the problems we study, our algorithms run in polynomial time (in $n$ and $d$) and obtain tight approximation ratios.  Previous work sacrificed one of these, i.e., either ran in polynomial time but produced sub-optimal approximation ratios or took time exponential in $d$ to guarantee tight approximation ratios.

\begin{table}[t]
    \centering
    % \bgroup
    \begin{tabular}{r|c|c|c}
        \hline
        Reference & 
        $w$ & 
        $t$ & 
        Running time \\
        \hline
        \cite{NissimSV16}, $\delta > 0$ & 
        $O(\sqrt{\log n})$ &
        $O(\frac{\sqrt{d}}{\eps} \cdot \poly\log\frac{1}{\delta})$ & 
        $\poly(n, d, \log \frac{1}{r})$\\
        \cite{NissimS18}, $\delta > 0$ & 
        $O(1)$ &
        $\tilde{O}_{\eps, \delta}(\frac{\sqrt{d}}{\eps} \cdot n^{0.1} \cdot \poly\log\frac{1}{\delta})$ & 
        $\poly(n, d, \log \frac{1}{r})$\\
        Exp. Mech.~\cite{McSherryT07}, $\delta = 0$ & $1 + \alpha$ & $O_{\alpha}(\frac{d}{\eps} \cdot \log \frac{1}{r})$ & $O\left(\left(\frac{1}{\alpha r}\right)^d\right)$\\
        % \hline
        % ??? &
        % ??? & ??? & ???\\ 
        % \hline
        % ??? &
        % ??? & ??? & ???\\     
        \hdashline
        Theorem~\ref{thm:1-cluster-given-radius-intro}, $\delta = 0$ & $1 + \alpha$ & $O_{\alpha}\left(\frac{d}{\epsilon}\cdot\log\left(\frac{d}{r}\right)\right)$ & $(nd)^{O_{\alpha}(1)} \poly\log \frac{1}{r}$\\
        Theorem~\ref{thm:1-cluster-given-radius-intro}, $\delta > 0$ & $1 + \alpha$ & $O_{\alpha}\left(\frac{\sqrt{d}}{\epsilon}\cdot\poly\log\left(\frac{nd}{\eps \delta}\right)\right)$ & $(nd)^{O_{\alpha}(1)} \poly\log \frac{1}{r}$\\        
        \hline
    \end{tabular}
    \caption{Comparison of $(\epsilon,\delta)$-DP algorithms for $(w, t)$-approximations for \densestball given $r$.}
    \label{fig:comparison_densest_ball}
\end{table}

(i) For \densestball, we obtain for any $\alpha > 0$,
a pure-DP $(1+\alpha, \tilde{O}_{\alpha}(\frac{d}{\epsilon}))$-approximation algorithm and 
an approximate-DP $(1+\alpha, \tilde{O}_{\alpha}(\frac{\sqrt{d}}{\epsilon}))$-approximation algorithm.%
\footnote{The notation $\tilde{O}_{x}(\cdot)$ ignores factors involving $x$ and factors polylogarithmic in $n, d, \epsilon, \delta$.}
The runtime of our algorithms is $\poly(nd)$.  Table~\ref{fig:comparison_densest_ball} shows our results compared to previous work. %The dependence on $d$ in the additive error $t$ is known to be optimal for $\delta=0$.
To solve \densestball with DP, we introduce and solve two problems: efficient list-decodable covers and private sparse selection. These could be of independent interest.

(ii) For \onecluster, informally, we obtain for any $\alpha > 0$, a pure-DP $(1+\alpha, \tilde{O}_{\alpha}(\frac{d}{\epsilon}))$-approximation algorithm running in time $(nd)^{O_{\alpha}(1)}$.  We also obtain an approximate-DP $(1+\alpha, \tilde{O}_{\alpha}(\frac{\sqrt{d}}{\epsilon}))$-approximation algorithm running in time $(nd)^{O_{\alpha}(1)}$.  The latter is an improvement over the previous work of~\cite{NissimS18} who obtain an $(\tilde{O}(1 + \frac{1}{\phi}), \tilde{O}_{\epsilon, \delta} (n^{\phi} \sqrt{d}))$-approximation.  In particular, they do not get an approximation ratio $w$ arbitrarily close to $1$.  Even worse, the exponent $\phi$ in the additive error $t$ can be made close to $0$ only at the expense of blowing up $w$.  Our algorithm for \onecluster follows by applying our DP algorithm for \densestball, along with ``DP binary search'' similarly to~\cite{NissimSV16}.

(iii) For \kmeans and \kmedian, we prove that we can take any (not necessarily private) approximation algorithm and convert it to a DP clustering algorithm with essentially the same approximation ratio, and with small additive error and small increase in runtime.  More precisely, given any $w^*$-approximation algorithm for \kmeans (resp., \kmedian), we obtain a pure-DP $(w^*(1+\alpha), \tilde{O}_{\alpha}(\frac{kd + k^{O_{\alpha}(1)}}{\epsilon}))$-approximation algorithm and an approximate-DP $(w^*(1+\alpha), \tilde{O}_{\alpha}(\frac{k\sqrt{d} + k^{O_{\alpha}(1)}}{\epsilon}))$-approximation algorithm for \kmeans (resp., \kmedian).
(The current best known non-private approximation algorithms achieve $w^*=6.358$ for \kmeans and  $w^* = 2.633$ for \kmedian~\cite{AhmadianNSW17}.) Our algorithms run in time polynomial in $n$, $d$ and $k$, and improve on those of~\cite{NissimS18} who only obtained some large constant factor approximation ratio independent of $w^*$.

It is known that $w^*$ can be made arbitrarily close to $1$ for (non-private) \kmeans and \kmedian if we allow fixed parameter tractable%
\footnote{
Recall that an algorithm is said to be \emph{fixed parameter tractable} in $k$ if its running time is of the form $f(k) \cdot \poly(n)$ for some function $f$, and where $n$ is the input size~\cite{DowneyF13}.
} algorithms~\cite{badoiu2002approximate,de2003approximation,kumar2004simple,kumar2005linear,chen2006k,feldman2007ptas,feldman2011unified}.
Using this, we get a pure-DP $(1+\alpha, \tilde{O}_{\alpha}(\frac{kd+k^2}{\epsilon}))$-approximation, and an approximate-DP $(1+\alpha, \tilde{O}_{\alpha}(\frac{k\sqrt{d}+k^2}{\epsilon}))$-approximation.  The algorithms run in time 
$2^{O_{\alpha}(k \log k)} \poly(nd)$.  

\paragraph{Overview of the Framework.}
All of our DP clustering algorithms follow this three-step recipe:

(i) Dimensionality reduction: we randomly project the input points to a low dimension.

(ii) Cluster(s) identification in low dimension: we devise a DP clustering algorithm in the low-dimensional space for the problem of interest, which results in cluster(s) of input points.

(iii) Cluster center finding in original dimension:  for each cluster found in step (ii), we privately compute a center in the original high-dimensional space minimizing the desired cost.

\paragraph{Applications.}

Our DP algorithms for \onecluster imply better algorithms for the Sample and Aggregate framework of~\cite{nissim2007smooth}.
Using a reduction from \onecluster due to~\cite{NissimSV16}, we get an algorithm that privately outputs a stable point with a radius not larger than the optimal radius than by a $1+\alpha$ factor, where $\alpha$ is an arbitrary positive constant. For more context, please see Section~\ref{sec:app_sample_and_aggregate}.

Moreover, by combining our DP algorithm for \densestball with a reduction of~\cite{Ben-DavidS00, Ben-DavidES02}, we obtain an efficient DP algorithm for agnostic learning of halfspaces with a constant margin.  Note that this result was already known from the work of Nguyen et al.~\cite{le2020efficient}; we simply give an alternative proof that employs our \densestball algorithm as a blackbox. For more on this and related work, please see Section~\ref{sec:app_agnostic_halfspaces}.

Finally, we provide an application of one of our observations outside of DP.  In particular, we give a faster (randomized) history-independent data structure for dynamically maintaining \cp in a moderate number of dimensions. This in turn implies a faster \emph{quantum} algorithm for \cp in a similar setting of parameters.

\paragraph{Organization.}

Section~\ref{sec:prelim} contains background on DP and clustering. Our algorithms for \densestball are presented in Section~\ref{sec:densest-ball-main-body}, and those for \kmeans and \kmedian are given in Section~\ref{sec:kmeans-kmedian-main-body}. Applications to \onecluster, Sample and Aggregate, agnostic learning of halfspaces with a margin, and  \cp are described in Section~\ref{sec:app_main_body}. We conclude with some open questions in Section~\ref{sec:conc_open_questions}. 
All missing proofs are deferred to the appendix.
\section{Preliminaries}
\label{sec:prelim}

\paragraph{Notation.}
For a finite universe $\cU$ and $\ell \in \N$, we let $\binom{\cU}{\leq \ell}$ be the set of all subsets of $\cU$ of size at most $\ell$.  Let $[n] = \{1, \ldots, n\}$. For $v \in \mathbb{R}^d$ and $r \in \R_{\geq 0}$, let $\cB(v, r)$ be the ball of radius $r$ centered at $v$. For $\kappa \in \R_{\geq 0}$, denote by $\mathbb{B}_{\kappa}^d$ the quantized $d$-dimensional unit ball with discretization step $\kappa$.\footnote{Whenever we assume that the inputs lie in $\mathbb{B}_{\kappa}^d$, our results will hold for any discretization as long as the minimum distance between two points as at least $\kappa$.} We throughout consider closed balls.
%(i.e. $\mathbb{B}_{\kappa}^d$ contains all points in $\cB(0, 1)$ such that each coordinates is an integer multiple of $\kappa$).

\paragraph{Differential Privacy (DP).}
We next recall the definition and basic properties of DP. Datasets $\bX$ and $\bX'$ are said to be neighbors if $\bX'$ results from removing or adding a single data point from $\bX$.%
\footnote{
This definition of DP is sometimes referred to as \emph{removal DP}.  Some works in the field consider the alternative notion of \emph{replacement DP} where two datasets are considered neighbors if one results from modifying (instead of removing) a single data point of the other. We remark that $(\eps, \delta)$-removal DP implies $(2\eps, 2\delta)$-replacement DP. Thus, our results also hold (with the same asymptotic bounds) for the replacement DP notion.}
\begin{definition}[Differential Privacy (DP)~\cite{dwork2006calibrating,DworkKMMN06}]
Let $\epsilon, \delta \in \R_{\geq 0}$ and $n \in \mathbb{N}$. A randomized algorithm $\cA$ taking as input a dataset is said to be \emph{$(\epsilon, \delta)$-differentially private} if for any two neighboring datasets $\bX$ and $\bX'$, and for any subset $S$ of outputs of $\cA$, it holds that $\Pr[\cA(\bX) \in S] \le e^{\epsilon} \cdot \Pr[\cA(\bX') \in S] + \delta$. If $\delta = 0$, then $\cA$ is said to be \emph{$\epsilon$-differentially private}.
\end{definition}
%
%Throughout the paper, 
We assume throughout that $0 < \epsilon \leq O(1)$, $0 < \alpha < 1$, and when used, $\delta > 0$.
%
\iffalse
\begin{theorem}[Composition] \label{thm:composition}
For any $\epsilon, \delta \geq 0, \delta' > 9$ and $k \in \N$, an algorithm that runs $k$ many $(\epsilon, \delta)$-DP algorithms (possibly adaptively) is 
(i) $(k \epsilon, k \delta)$-DP by basic composition~\cite{DworkKMMN06} and 
(ii) $( 2k\eps(e^\eps - 1) + \eps\sqrt{2k \ln \frac{1}{\delta'}}, k \delta + \delta')$-DP by advanced composition~\cite{DworkRV10}.
\end{theorem}
\fi

\paragraph{Clustering.}
Since many of the proof components are common to the analyses of \kmeans and \kmedian, we will use the following notion, which generalizes both problems.

\begin{definition}[$(k, p)$-Clustering] \label{def:vanilla-clustering}
Given $k \in \mathbb{N}$ and a multiset $\bX = \{x_1, \dots, x_n \}$ of points in the unit ball, we wish to find $k$ centers $c_1, \dots, c_k \in \R^d$ minimizing
% \mymath{
$\cost^p_{\bX}(c_1, \dots, c_k) := \sum_{i \in [n]} \left(\min_{j \in [k]} \|x_i - c_j\|\right)^p$.
% }
Let $\opt_{\bX}^{p, k}$ denote\footnote{
The cost is sometimes defined as the $(1/p)$th power.
} $\min_{c_1, \dots, c_k \in \R^d} \cost^p_{\bX}(c_1, \dots, c_k)$. 
A \emph{$(w, t)$-approximation algorithm} for \emph{$(k, p)$-Clustering} outputs $c_1, \dots, c_k$ such that $\cost^p_{\bX}(c_1, \dots, c_k) \leq w \cdot \opt_{\bX}^{p, k} + t$.
When $\bX$, $p$, and $k$ are unambiguous, we drop the subscripts and superscripts.
%Throughout our analysis, we will fix an optimal solution $c^*_1, \dots, c^*_k$ where ties are broken arbitrarily. For such a solution, let the map $\psi: [n] \to [k]$ be such that $c^*_{\psi^*(i)} \in \argmin_{j \in [k]} \|x_i - c_j\|$ (ties broken arbitrarily). For every $j \in [k]$, let\footnote{We assume throughout that $n^*_j > 0$. This is w.l.o.g. in the case $n \geq k$. When $n < k$, our differentially private algorithms can output anything, since additive errors allowed are more than $k$.} $n^*_j := |\psi^{-1}(j)|$ be the number of input points closest to center $c^*_j$ and let $r^*_j := \frac{1}{n^*_j} \sum_{i \in \psi^{-1}(j)} \|x_i - c_j\|$ denote the average distance of from input points mapped to center $c_j$ to the center. Finally, we let $\tr_j$ to denote $\min\left\{r^*_j, \left(\frac{\opt}{n^*_j k}\right)^{1/p}\right\}$.
\end{definition}

Note that $(k, 1)$-Clustering and $(k, 2)$-Clustering correspond to \kmedian and \kmeans  respectively.  It will also be useful to consider the \emph{Discrete $(k, p)$-Clustering} problem, which is the same as in Definition~\ref{def:vanilla-clustering}, except that we are given a set $\cC$ of ``candidate centers'' and we can only choose the centers from $\cC$.  We use $\opt_{\bX}^{p, k}(\cC)$ to denote $\min_{c_{i_1}, \dots, c_{i_k} \in \cC} \cost^p_{\bX}(c_{i_1}, \dots, c_{i_k})$.

%We remark that  Definition~\ref{def:vanilla-clustering}; we choose our definition because when $p = 2$ it coincides with $k$-means.

\iffalse 
\begin{definition}[Discrete $(k, p)$-Cluster] \label{def:vanilla-clustering-centers}
In the Discrete $(k, p)$-Clustering problem, we are given a multiset $\bX = (x_1, \dots, x_n)$ of $n$ points and a multiset $\bC = (c_1, \dots, c_m)$ of $m$ candidate centers all lying in the $d$-dimensional unit ball, along with an integer $k$, and the goal is to find $k$ centers $c_{i_1}, \dots, c_{i_k}$, among the given candidates, that minimize $\cost^p_{\bX}(c_{i_1}, \dots, c_{i_k})$.

We use $\opt_{\bX}^{p, k}(\cC)$ to denote $\min_{c_{i_1}, \dots, c_{i_k} \in \cC} \cost^p_{\bX}(c_{i_1}, \dots, c_{i_k})$. An algorithm is said to solve Discrete $(k, p)$-Clustering with approximation ratio $\rho$ and additive error $\tau$ if it outputs $c_{i_1}, \dots, c_{i_k} \in \cC$ such that $\cost^p_{\bX}(c_{i_1}, \dots, c_{i_k}) \leq \rho \cdot \opt^{p, k}_{\bX}(\cC) + \tau$.
\end{definition}
\fi

\paragraph{Centroid Sets and Coresets.}

A centroid set is a set of candidate centers such that the optimum does not increase by much even when we restrict the centers to belong to this set.

\begin{definition}[Centroid Set~\cite{Matousek00}] \label{def:centroid-set}
For $w, t > 0, p \geq 1$, $k, d \in \N$, a set $\cC \subseteq \R^d$ is a \emph{$(p, k, w, t)$-centroid set} of $\bX \subseteq \R^d$ if $\opt^{p, k}_{\bX}(\cC) \leq w \cdot \opt^{p, k}_{\bX} + t$.  When $k$ and $p$ are unambiguous, we simply say that $\cC$ is a \emph{$(w, t)$-centroid set} of $\bX$.
\end{definition}

A coreset is a (multi)set of points such that, for any possible $k$ centers, the cost of $(k, p)$-Clustering of the original set is roughly the same as that of the coreset (e.g.,~\cite{HarPeledM04}).

\begin{definition}[Coreset] \label{def:coreset}
For $\gamma, t > 0, p \geq 1, k \in \N$, a set $\bX'$ is a \emph{$(p, k, \gamma, t)$-coreset} of $\bX \subseteq \R^d$ if for every $\cC = \{c_1, \dots, c_k\} \subseteq \R^d$, we have $(1 - \gamma) \cdot \cost_{\bX}^p(\cC) - t \leq \cost_{\bX'}(\cC) \leq (1 + \gamma) \cdot \cost_{\bX}^p(\cC) + t$.
When $k$ and $p$ are unambiguous, we simply say that $\bX'$ is a \emph{$(\gamma, t)$-coreset} of $\bX$.
\end{definition}

%\begin{definition}[Clustering Coresets \cite{HarPeledM04}] \label{def:coreset}
%For a multiset $\bX$ of points in $\mathbb{R}^d$ and for $\epsilon \in [0,1]$, a multiset $\bS$ of points in $\mathbb{R}^d$ is said to be a $(k, p, \epsilon)$-coreset of $\bX$ if for any $k$ points $c_1, \dots, c_k \in \mathbb{R}^d$, it holds that
%\begin{equation*}
%    (1-\epsilon) \cost^p_{\bX}(c_1, \dots, c_k) \le \cost^p_{\bS}(c_1, \dots, c_k) \le (1+\epsilon) \cost^p_{\bX}(c_1, \dots, c_k).
%\end{equation*}
%\end{definition}

\section{Private \densestball}
\label{sec:densest-ball-main-body}

In this section, we obtain pure-DP and approximate-DP algorithms for \densestball.

\iffalse
\begin{theorem} \label{thm:pure-1-cluster-given-radius-intro}
For every $0 < \epsilon, \alpha < 1$, there is an $\eps$-DP algorithm that runs in time $(nd)^{O_{\alpha}(1)} \cdot \poly\log(1/r)$ and, with probability\footnote{Throughout the main body of the paper, we will state our accuracy guarantees that hold with probability 0.99. It is also simple to derive error guarantee that holds with probability $1 - \beta$ for any $\beta > 0$, with mild dependency on $\beta$ in the error. Such bounds are stated explicitly in the Supplementary Material.} $0.99$, returns a $\left(1 + \alpha, O_{\alpha}\left(\frac{d}{\epsilon} \cdot \poly\log\left(\frac{nd}{\eps r}\right)\right)\right)$-approximation for \densestball for a given $r$.
\end{theorem}

\begin{theorem} \label{thm:apx-1-cluster-given-radius-intro}
For every $0 < \eps, \delta, \alpha  < 1$, there is an $(\eps, \delta)$-DP algorithm that runs in time $(nd)^{O_{\alpha}(1)} \cdot \poly\log(1/r)$ and, with probability $0.99$, returns a $\left(1 + \alpha, O_{\alpha}\left(\frac{\sqrt{d}}{\epsilon} \cdot \poly\log\left(\frac{nd}{\eps \delta}\right)\right)\right)$-approximation for \densestball for a given $r$.
\end{theorem}
\fi

\begin{theorem} \label{thm:1-cluster-given-radius-intro}
There is an $\eps$-DP (resp., $(\eps, \delta)$-DP) algorithm that runs in time $(nd)^{O_{\alpha}(1)} \cdot \poly\log(1/r)$ and, w.p.\footnote{In the main body of the paper, we state error bounds that hold with probability $0.99$. In the appendix, we extend all our bounds to hold with probability $1 - \beta$ for any $\beta > 0$, with a mild dependency on $\beta$ in the error.} $0.99$, returns a $\left(1 + \alpha, O_{\alpha}\left(\frac{d}{\epsilon} \cdot \log\left(\frac{d}{r}\right)\right)\right)$-approximation 
(resp., $\left(1 + \alpha, O_{\alpha}\left(\frac{\sqrt{d}}{\epsilon} \cdot \poly\log\left(\frac{nd}{\eps \delta}\right)\right)\right)$-approximation)
for \densestball.
\end{theorem}

To prove this, we follow the three-step recipe from Section~\ref{sec:intro}.  Using the Johnson--Lindenstrauss (JL) lemma~\cite{JL} together with the Kirszbraun Theorem~\cite{kirszbraun1934} on extensions of Lipschitz functions, we project the input to $O((\log n) / \alpha^2)$ dimensions in step (i). It turns out that step (iii) is similar to (ii), as we can repeatedly apply a low-dimensional \densestball algorithm to find a center in the high-dimensional space. Therefore, the bulk of our technical work is in carrying out step (ii), i.e., finding an efficient, DP algorithm for \densestball in $O((\log n)/\alpha^2)$ dimensions.  We focus on this part in the rest of this section; the full proof with the rest of the arguments can be found in
Appendix~\ref{subsec:dim-red-densest-ball}.

\subsection{A Private Algorithm in Low Dimensions}
\label{sec:densest-ball-low-dim-main-body}

%For convenience, we reuse $d$ to denote $O((\log n) / \alpha^2)$ in this section. 
Having reduced the dimension to $d' = O((\log n) / \alpha^2)$ in step (i), we can afford an algorithm that runs in time $\exp(O_{\alpha}(d')) = n^{O_{\alpha}(1)}$. With this in mind, our algorithms in dimension $d'$ have the following guarantees:

\iffalse
\begin{theorem} \label{thm:pure-1-cluster-given-radius-main-body}
For every $0 < \eps, \alpha < 1$, there is an $\eps$-DP algorithm that runs in time $(1 + 1/\alpha)^{O(d)} \poly\log(1/r)$ and, with probability $0.99$, returns a $\left(1 + \alpha, O_{\alpha}\left(\frac{d}{\epsilon}\log\left(\frac{n}{\eps r}\right)\right)\right)$-approximation for \densestball.
\end{theorem}

\begin{theorem} \label{thm:apx-1-cluster-given-radius-main-body}
For every $0 < \eps, \delta, \alpha  < 1$, there is an $(\eps, \delta)$-DP algorithm that runs in time $(1 + 1/\alpha)^{O(d)} \poly\log(1/r)$ and, with probability $0.99$, returns a $\left(1 + \alpha, O_{\alpha}\left(\frac{d}{\epsilon}\log\left(\frac{n}{\eps \delta}\right)\right)\right)$-approximation for \densestball.
\end{theorem}
\fi

\begin{theorem} \label{thm:1-cluster-given-radius-main-body}
There is an $\eps$-DP (resp., $(\eps, \delta)$-DP) algorithm that runs in time $(1 + 1/\alpha)^{O(d')} \poly\log(1/r)$ and, w.p. $0.99$, returns a $\left(1 + \alpha, O_{\alpha}\left(\frac{d'}{\epsilon}\log\left(\frac{1}{r}\right)\right)\right)$-approximation
(resp., $\left(1 + \alpha, O_{\alpha}\left(\frac{d'}{\epsilon}\log\left(\frac{n}{\eps \delta}\right)\right)\right)$-approximation) 
for \densestball.
\end{theorem}

As the algorithms are allowed to run in time exponential in $d'$, Theorem~\ref{thm:1-cluster-given-radius-main-body} might seem easy to devise at first glance. Unfortunately, even the Exponential Mechanism~\cite{McSherryT07}, which is the only known algorithm achieving approximation ratio arbitrarily close to $1$, still takes $\Theta_{\alpha}(1/r)^{d'}$ time, which is $\exp(\omega(d'))$ for $r = o(1)$. (In fact, in applications to \kmeans and \kmedian, we set $r$ to be as small as $1/n$, which would result in a running time of $n^{\Omega(\log n)}$.) To understand, and eventually overcome this barrier, we recall the implementation of the Exponential Mechanism for \densestball:
\begin{itemize}
%[nosep]
    \item Consider any $(\alpha r)$-cover%
    \footnote{A \emph{$\zeta$-cover} $C$ of $\cB(0, 1)$ is a set of points such that for any $y \in \cB(0, 1)$, there is $c \in C$ with $\|c - y\| \leq \zeta$.}  $C$ of the unit ball $\cB(0, 1)$.
    \item For every $c \in C$, let $\score[c]$ be the number of input points lying inside $\cB(c, (1 + \alpha)r)$.
    \item Output a point $c^* \in C$ with probability $\frac{e^{(\eps/2) \cdot \score[c^*]}}{\sum_{c \in C} e^{(\eps/2) \cdot \score[c]}}$.
\end{itemize}
By the generic analysis of the Exponential Mechanism~\cite{McSherryT07}, this algorithm is $\eps$-DP and achieves a $\left(1 + \alpha, O_{\alpha}\left(\frac{d'}{\epsilon}\log\left(\frac{1}{r}\right)\right)\right)$-approximation as in Theorem~\ref{thm:1-cluster-given-radius-main-body}. The existence of an $(\alpha r)$-cover of size $\Theta\left(\frac{1}{\alpha r}\right)^{d'}$ is well-known and directly implies the $\Theta_{\alpha}(1/r)^{d'}$ running time stated above.

%Our main technical contribution for Theorem~\ref{thm:pure-1-cluster-given-radius} is to implement the Exponential Mechanism in running time $\Theta_{\alpha}(1)^d \poly\log\frac{1}{r}$ instead of $\Theta_{\alpha}(1/r)^d$ as above. For Theorem~\ref{thm:apx-1-cluster-given-radius}, as we will elaborate more below, we modify the Exponential Mechanism so that we get an error that is independent of $r$, at the cost of introducing approximate DP (i.e. $\delta > 0$).

Our main technical contribution is to implement the Exponential Mechanism in $\Theta_{\alpha}(1)^{d'} \poly\log\frac{1}{r}$ time instead of $\Theta_{\alpha}(1/r)^{d'}$.
To elaborate on our approach, for each input point $x_i$, we define $S_i$ to be $C \cap \cB(x_i, (1 + \alpha)r)$, i.e., the set of all points in the cover $C$ within distance $(1 + \alpha)r$ of $x_i$. Note that the score assigned by the Exponential Mechanism is $\score[c] = \{i \in [n] \mid c \in S_i\}$, and our goal is to privately select $c^* \in C$ with as large a score as possible. Two main questions remain: (1) How do we find the $S_i$'s efficiently? (2) Given the $S_i$'s, how do we sample $c^*$? We address these in the following two subsections, respectively. 

\subsubsection{Efficiently List-Decodable Covers}
\label{subsec:list-decodable-main-body}

In this section, we discuss how to find $S_i$ in time $(1 + 1/\alpha)^{O(d')}$. %A priori, the \emph{size} of $S_i$ is not even necessarily bounded by $(1 + 1/\alpha)^{O(d)}$. 
Motivated by works on error-correcting codes (see, e.g.,~\cite{Guruswami06}), we introduce the notion of \emph{list-decodability} for covers:

\begin{definition}[List-Decodable Cover]
A $\Delta$-cover is \emph{list-decodable} at distance $\Delta' \geq \Delta$ with list size $\ell$ if for any $x \in \cB(0, 1)$, we have that $|\{c \in C \mid \|c - x\| \leq \Delta'\}| \leq \ell$. Moreover, the cover is \emph{efficiently list-decodable} if there is an algorithm that returns such a list in time $\poly(\ell, d', \log(1/\Delta))$.
\end{definition}

We prove the existence of efficiently list-decodable covers with the following parameters:

\begin{lemma} \label{lem:cover-main-main-body}
For every $0 < \Delta < 1$, there exists a $\Delta$-cover $C_{\Delta}$ that is efficiently list-decodable at any distance $\Delta' \geq \Delta$ with list size $(1 + \Delta' / \Delta)^{O(d')}$.
\end{lemma}

In this terminology, $S_i$ is exactly the decoded list at distance $\Delta' = (1 + \alpha)r$, where $\Delta = \alpha r$ in our cover $C$. As a result, we obtain the $(1 + 1/\alpha)^{O(r)}$ bound on the time for computing $S_i$, as desired.

The proof of Lemma~\ref{lem:cover-main-main-body} has to include two tasks: (i) bounding the size of the list and (ii) coming up with an efficient decoding algorithm. It turns out that (i) is not too hard: if we ensure that our cover is also an $\Omega(\Delta)$-packing\footnote{A \emph{$\zeta$-packing} is a set of points such that each pairwise distance is at least $\zeta$.}, then a standard volume argument implies the bound in Lemma~\ref{lem:cover-main-main-body}.  However, carrying out (ii) is more challenging. To do so, we turn to lattice-based covers.  A \emph{lattice} is a set of points that can be written as an integer combination of some given basis vectors.
Rogers~\cite{rogers1959lattice} (see also~\cite{Micciancio04}) constructed a family of lattices that are both $\Delta$-covers and $\Omega(\Delta)$-packings. Furthermore, known lattice algorithms for the so-called \emph{Closest Vector Problem}~\cite{MicciancioV13} allow us to find a point $c \in C_{\Delta}$ that is closest to a given point $x$ in time $2^{O(d')}$. With some more work, we can ``expand'' from $c$ to get the entire list in time polynomial in $\ell$. This concludes the outline of our proof of Lemma~\ref{lem:cover-main-main-body}.

\subsubsection{\SparseSelection}
\label{subsec:sparse-selection-main-body}
We now move to (2): given $S_i$'s, how to privately select $c^*$ with large $\score[c^*] = \left|\{i \mid c^* \in S_i\}\right|$?

%\begin{wrapfigure}{R}{0.5\textwidth}
\begin{figure}[b]
\centering
\vspace*{-4mm}
\begin{minipage}{0.5\textwidth}
\begin{algorithm}[H]
\caption{}\label{alg:densest-ball-main-body}
\begin{algorithmic}[1]
\Procedure{\densestball$(x_1, \dots, x_n; r, \alpha)$}{}
\State $C_{\alpha r} \leftarrow$ $(\alpha r)$-cover from Lemma~\ref{lem:cover-main-main-body}
\For{$i \in [n]$}
\State $S_i \leftarrow$ decoded list of $x$ at distance $(1 + \alpha)r$ with respect to $C_{\alpha r}$
\EndFor
\Return $\SparseSelection(S_1, \dots, S_n)$
\EndProcedure
\end{algorithmic}
\end{algorithm}
\end{minipage}
%\end{wrapfigure}
\end{figure}

We formalize the problem as follows:
\begin{definition}[\SparseSelection] \label{def:sparse-selection-main-body}
For $\ell \in \mathbb{N}$, the input to the \emph{$\ell$-\SparseSelection} problem is a list $S_1, \dots, S_n$ of subsets, where $S_1, \dots, S_n \in \binom{C}{\leq \ell}$ for some finite universe $C$. An algorithm solves $\ell$-\SparseSelection with additive error $t$ if it outputs a universe element $\hat{c} \in C$ such that $\left|\{i \mid \hat{c} \in S_i\}\right| \geq \max_{c \in C} \left|\{i \mid c \in S_i\}\right|  - t$.
\end{definition}

The crux of our \SparseSelection algorithm is the following.  Since $\score[c^*] = 0$ for all $c^* \notin S_1 \cup \cdots \cup S_n$, to implement the Exponential Mechanism it suffices to first randomly select (with appropriate probability) whether we should sample from $S_1 \cup \cdots \cup S_n$ or uniformly from $C$. For the former, the sampling is efficient since $S_1 \cup \cdots \cup S_n$ is small. This gives the following for pure-DP:
\begin{lemma} \label{lem:pure-sparse-selection-main-body}
Suppose there is a $\poly \log |C|$-time algorithm $\cO$ that samples a random element of $C$ where each element of $C$ is output with probability at least $0.1/|C|$. Then, there is a $\poly(n, \ell, \log |C|)$-time $\epsilon$-DP algorithm that, with probability $0.99$, solves $\ell$-\SparseSelection with additive error $O\left(\frac{1}{\eps} \cdot \log |C|\right)$.
\end{lemma}

We remark that, in Lemma~\ref{lem:pure-sparse-selection-main-body}, we only require $\cO$ to sample \emph{approximately} uniformly from $C$. This is due to a technical reason that we only have such a sampler for the lattice covers we use. Nonetheless, the outline of the algorithm is still exactly the same as before.

For approximate-DP, it turns out that we can get rid of the dependency of $|C|$ in the additive error entirely, by adjusting the probability assigned to each of the two cases. In fact, for the second case, it even suffices to just output some symbol $\perp$ instead of sampling (approximately) uniformly from $C$. Hence, there is no need for a sampler for $C$ at all, and this gives us the following guarantees:
\begin{lemma} \label{lem:apx-sparse-selection-main-body}
There is a $\poly(n, \ell, \log |C|)$-time $(\epsilon, \delta)$-DP algorithm that, with probability $0.99$, solves $\ell$-\SparseSelection with additive error $O\left(\frac{1}{\eps}\log\left(\frac{n\ell}{\epsilon \delta}\right)\right)$.
\end{lemma}

\subsubsection{Putting Things Together}
\label{subsec:final-densest-ball-main-body}

With the ingredients ready, the \densestball algorithm is given in Algorithm~\ref{alg:densest-ball-main-body}. The pure- and approximate-DP algorithms for \SparseSelection in Lemmas~\ref{lem:pure-sparse-selection-main-body} and~\ref{lem:apx-sparse-selection-main-body} lead to Theorem~\ref{thm:1-cluster-given-radius-main-body}.
\section{Private \kmeans and \kmedian}
\label{sec:kmeans-kmedian-main-body}

We next describe how we use our \densestball algorithm along with additional ingredients adapted from previous studies of coresets to obtain DP approximation algorithms for \kmeans and \kmedian with nearly tight approximation ratios and small additive errors as stated next:

\begin{theorem} \label{thm:main-apx-non-private-to-private-pure-intro}
Assume there is a polynomial-time (not necessarily DP) algorithm for \kmeans (resp., \kmedian) in $\mathbb{R}^d$ with approximation ratio $w$. Then, there is an $\eps$-DP algorithm that runs in time $k^{O_{\alpha}(1)} \poly(nd)$ and, with probability $0.99$, produces a $\left(w(1 + \alpha), O_{w, \alpha}\left(\left(\frac{kd + k^{O_{\alpha}(1)}}{\eps}\right) \poly\log n\right)\right)$-approximation for  \kmeans (resp., \kmedian). Moreover, there is an $(\epsilon, \delta)$-DP algorithm with the same runtime and approximation ratio but with additive error $O_{w, \alpha}\left(\left(\frac{k\sqrt{d}}{\eps} \cdot \poly\log\left(\frac{k}{\delta}\right) \right) + \left(\frac{k^{O_{\alpha}(1)}}{\eps} \cdot \poly\log n\right)\right)$.
\end{theorem}

% Since the two share the same proof structure, we only focus on $k$-medians here for simplicity of exposition.
To prove Theorem~\ref{thm:main-apx-non-private-to-private-pure-intro}, as for \densestball, we first reduce the dimension of the clustering instance from $d$ to $d' = O_{\alpha}(\log k)$, which can be done using the recent result of Makarychev et al.~\cite{MakarychevMR19}. Our task thus boils down to proving the following low-dimensional analogue of Theorem~\ref{thm:main-apx-non-private-to-private-pure-intro}. 

\begin{theorem}\label{thm:main-apx-non-private-to-private-pure-intro-dim-red}
Under the same assumption as in Theorem~\ref{thm:main-apx-non-private-to-private-pure-intro}, there is an $\eps$-DP algorithm that runs in time $2^{O_{\alpha}(d')} \poly(n)$ and, with probability $0.99$, produces a $\left(w(1 + \alpha), O_{\alpha, w}\left(\frac{k^2 \cdot 2^{O_{\alpha}(d')}}{\eps} \poly\log n\right)\right)$-approximation for \kmeans (resp., \kmedian).
\end{theorem}

We point out that it is crucial for us that the reduced dimension $d'$ is $O_{\alpha}(\log k)$ as opposed to $O_{\alpha}(\log n)$ (which is the bound from a generic application of the JL lemma), as otherwise the additive error in Theorem~\ref{thm:main-apx-non-private-to-private-pure-intro-dim-red} would be $\poly(n)$, which is vacuous, instead of $\poly(k)$. We next proceed by (i) finding a ``coarse'' centroid set (satisfying Definition~\ref{def:centroid-set} with $w=O(1)$), (ii) turning the centroid set into a DP coreset (satisfying Definition~\ref{def:coreset} with $w = 1 + \alpha$), and (iii) running the non-private approximation algorithm as a black box. We describe these steps in more detail below. % Step (iii) can be done using any off-the-shelf clustering procedure. %We solve step (i) using our \textsf{DensestBall} algorithm as a subroutine. For step (ii), we use the result of~\cite{feldman2009private}, in turn based on~\cite{HarPeledM04}, which can, given any coarse centroid set, produce a DP coreset. We describe steps (ii) and (iii) in more detail below. %We note that a variant of step (ii) was shown in \cite{feldman2009private} though we cannot use their result as a black box as ???.

\subsection{Finding a Coarse Centroid Set via \densestball}
\label{subsec:coarse-centroid-main-body}

We consider geometrically increasing radii $r = 1/n, 2/n, 4/n, \dots$. For each such $r$, we iteratively run our \densestball algorithm $2k$ times, and for each %of the $2k$ obtained centers,
returned center, remove all points within a distance of $8r$ from it. This yields $2k \log n$ candidate centers. We prove that they form a centroid set with a constant approximation ratio and a small additive error: 
\begin{lemma} \label{lem:densest-ball-to-kmeans-kmedian-main-body}
There is a polynomial time $\eps$-DP algorithm that, with probability $0.99$, outputs an $\left(O(1), O\left(\frac{k^2 d'}{\eps}\poly\log n\right)\right)$-centroid set of size $2k\log n$ for \kmeans (resp., \kmedian).
\end{lemma}
We point out that the solution to this step is not unique. For example, it is possible to run the DP algorithm for \kmeans from~\cite{StemmerK18} instead of Lemma~\ref{lem:densest-ball-to-kmeans-kmedian-main-body}. However, we choose to use our algorithm since its analysis works almost verbatim for both \kmedian and \kmeans, and it is simple.

\subsection{Turning a Coarse Centroid Set into a Coreset}

Once we have a coarse centroid set from the previous step, we follow the approach of Feldman et al.~\cite{feldman2009private}, which can turn the coarse centroid and eventually produce a DP coreset:

\begin{lemma} \label{lem:coreset-main-body}
There is an $2^{O_{\alpha}(d')} \poly(n)$-time $\eps$-DP algorithm that, with probability 0.99, produces an $\left(\alpha, O_{\alpha}\left(\frac{k^2\cdot 2^{O_{\alpha}(d')}}{\eps} \poly\log n\right)\right)$-coreset for \kmeans (and \kmedian).
\end{lemma}

Roughly speaking, the idea is to first ``refine'' the coarse centroid by constructing an \emph{exponential cover} around each center $c$ from Lemma~\ref{lem:densest-ball-to-kmeans-kmedian-main-body}. Specifically, for each radius $r = 1/n, 2/n, 4/n, \dots$, we consider all points in the $(\alpha r)$-cover of the ball of radius $r$ around $c$. Notice that the number of points in such a cover can be bounded by $2^{O_{\alpha}(d')}$. Taking the union over all such $c, r$, this result in a new \emph{fine} centroid set of size $2^{O_{\alpha}(d')} \cdot \poly(k, \log n)$. Each input point is then snapped to the closet point in this set; these snapped points form a good coreset~\cite{HarPeledM04}.  To make this coreset private, we add an appropriately calibrated noise to the number of input points snapped to each point in the fine centroid set. The additive error resulting from this step scales linearly with the size of the fine centroid set, which is $2^{O_{\alpha}(d')} \cdot \poly(k, \log n)$ as desired.

We note that, although our approach in this step is essentially the same as Feldman et al.~\cite{feldman2009private}, they only fully analyzed the algorithm for \kmedian and $d \leq 2$. Thus, we cannot use their result as a black box and hence, we provide a full proof that also works for \kmeans and for any $d >0$ in Appendix~\ref{sec:cluster-low-dim}.

\subsection{Finishing Steps}

Finally, we can simply run the (not necessarily DP) approximation algorithm on the DP coreset from Lemma~\ref{lem:coreset-main-body}, which immediately yields Theorem~\ref{thm:main-apx-non-private-to-private-pure-intro-dim-red}.
\section{Applications}\label{sec:app_main_body}
Our \densestball algorithms imply new results for other well-studied tasks, which we now describe.

\subsection{\onecluster}

Recall the \onecluster problem from Section~\ref{sec:intro}. As shown by \cite{NissimSV16}, a discretization of the inputs is necessary to guarantee a finite error with DP, so we assume that they lie in $\mathbb{B}_{\kappa}^d$. For this problem, they obtained an $O(\sqrt{\log{n}})$ approximation ratio, which was subsequently improved to some large constant by \cite{NissimS18} albeit with an additive error that grows polynomially in $n$. Using our \densestball algorithms we get a $1+\alpha$ approximation ratio with additive error polylogarithmic in $n$:

\begin{theorem}\label{th:1_cluster_app}
For $0 < \kappa <1$, there is an $\epsilon$-DP algorithm that runs in $(nd)^{O_{\alpha}(1)} \poly\log(\frac{1}{\kappa})$ time and with probability $0.99$, outputs a $\left(1+\alpha, O_{\alpha}\left(\frac{d}{\epsilon}\poly\log\left(\frac{n}{\eps \kappa}\right)\right)\right)$-approximation for \onecluster. For any $\delta > 0$, there is an $(\epsilon, \delta)$-DP algorithm with the same runtime and approximation ratio but with additive error $O_{\alpha}\left(\frac{\sqrt{d}}{\epsilon} \cdot \poly\log\left(\frac{nd}{\eps \delta}\right)\right) + O\left(\frac{1}{\epsilon} \cdot \log(\frac{1}{\delta}) \cdot 9^{\log^*(d/\kappa)}\right)$.
\end{theorem}

\subsection{Sample and Aggregate}\label{sec:app_sample_and_aggregate}

Consider functions $f: {\cal U}^* \to \mathbb{B}_{\kappa}^d$ mapping databases to the discretized unit ball. A basic technique in DP is \emph{Sample and Aggregate} \cite{nissim2007smooth}, whose premise is that for large databases $S \in U^*$, evaluating $f$ on a random subsample of $S$ can give a good approximation to $f(S)$. This method enables bypassing worst-case sensitivity bounds in DP (see, e.g., \cite{DworkR14}) and it captures basic machine learning primitives such as bagging \cite{jordon2019differentially}. Concretely, a point $c \in \mathbb{B}_{\kappa}^d$ is an
\emph{$(m,r,\zeta)$-stable point} of $f$ on $S$ if $\Pr[\|f(S')-c\|_2 \le r] \geq \zeta$ for $S'$ a database of $m$ i.i.d.~samples from $S$. If such a point $c$ exists, $f$ is $(m,r,\zeta)$-stable on $S$, and $r$ is a radius of $c$. Via a reduction to \textsf{$1$-Cluster},~\cite{NissimSV16} find a stable point of radius within an $O(\sqrt{\log{n}})$ factor from the smallest possible while \cite{nissim2007smooth} got an $O(\sqrt{d})$ approximation, and a constant factor is subsequently implied by \cite{NissimS18}. Our \textsf{$1$-Cluster} algorithm yields a $1+\alpha$ approximation:

\begin{theorem}\label{thm:sample_and_agg_ptas_app}
Let $d, m, n \in \mathbb{N}$ and $0 < \epsilon, \zeta, \alpha, \delta, \kappa < 1$ with $m \le n$, $\epsilon \le \frac{\zeta}{72}$ and $\delta \le \frac{\epsilon}{300}$. There is an $(\epsilon, \delta)$-DP algorithm that takes $f:U^n \to \mathbb{B}_{\kappa}^d$ and parameters $m$, $\zeta$, $\epsilon$, $\delta$, runs in time $(\frac{nd}{m})^{O_{\alpha}(1)} \poly\log(\frac{1}{\kappa})$ plus the time for $O(\frac{n}{m})$ evaluations of $f$ on a dataset of size $m$, and whenever $f$ is $(m, r, \zeta)$-stable on $S$, with probability $0.99$, the algorithm outputs an $(m, (1+\alpha) r, \frac{\zeta}{8})$-stable point of $f$ on $S$, provided that $n \geq m \cdot O_{\alpha}\left(\frac{\sqrt{d}}{\epsilon} \cdot \poly\log\left(\frac{nd}{\eps \delta}\right) + \frac{1}{\epsilon} \cdot \log(\frac{1}{\delta}) \cdot 9^{\log^*(d/\kappa)}\right)$.
\end{theorem}

% Theorem~\ref{thm:sample_and_agg_ptas_app} gets a $1+\alpha$ approximation to the radius whereas \cite{nissim2007smooth} had obtained an $O(\sqrt{d})$ approximation, \cite{NissimSV16} improved it to $O(\sqrt{\log{n}})$, and some constant is implied from \cite{NissimS18}.

\subsection{Agnostic Learning of Halfspaces with a Margin}\label{sec:app_agnostic_halfspaces}

We next apply our algorithms to the well-studied problem of agnostic learning of halfspaces with a margin (see, e.g., \cite{Ben-DavidS00,bartlett2002rademacher,mcallester2003simplified,SSS09,BirnbaumS12,diakonikolas2019nearly,DKM20}).
Denote the error rate of a hypothesis $h$ on a distribution $D$ on labeled samples by $\err^D(h)$, and the $\mu$-margin error rate of halfspace $h_u(x) = \sgn(u \cdot x)$ on $D$ by $\err^D_{\mu}(u)$. (See Appendix~\ref{sec:agnostic_halfspaces_margin} for precise definitions.)
Furthermore, let $\opt^D_{\mu} := \min_{u \in \mathbb{R}^d} \err^D_{\mu}(u)$.
The problem of learning halfspaces with a margin in the agnostic PAC model~\cite{haussler1992decision,kearns1994toward} can be defined as follows.

\begin{definition}
Let $d \in \mathbb{N}$ and $\mu, t \in \mathbb{R}^{+}$. An algorithm \emph{properly agnostically PAC learns halfspaces} with margin $\mu$, error $t$ and sample complexity $m$, if given as input a training set $S = \{(x^{(i)}, y^{(i)})\}_{i=1}^m$ of i.i.d. samples drawn from an unknown distribution $D$ on $\cB(0, 1) \times \{\pm 1\}$, it outputs a halfspace $h_u :\mathbb{R}^d \to \{\pm 1\}$ satisfying $\err^D(h_u) \le \opt^D_{\mu} + t$ with probability $0.99$.
\end{definition}

Via a reduction of \cite{Ben-DavidS00, Ben-DavidES02} from agnostic learning of halfspaces with a margin to \densestball, we can use our \densestball algorithm to derive the following:

\begin{theorem}\label{thm:DP_agn_learn_halfspaces_app}
For $0 < \mu, t < 1$, there is an $\eps$-DP algorithm that runs in time $(\frac{1}{\epsilon t})^{O_{\mu}(1)} + \poly\left(O_{\mu}\left(\frac{d}{\epsilon t}\right)\right)$, and with probability $0.99$, properly agnostically learns halfspaces with margin $\mu$, error $t$, and sample complexity $O_{\mu} \left(\frac{1}{\eps t^2}\cdot\poly\log\left(\frac{1}{\eps t}\right)\right)$.
\end{theorem}
%The sample complexity dependence on $1 / t^2$ is needed even for non-DP algorithms \cite{bartlett2002rademacher,mcallester2003simplified}. 

We reiterate that this result can also be derived by an algorithm of Nguyen et al.~\cite{le2020efficient}\footnote{\cite{le2020efficient} analyzed their algorithm only in the realizable case where $\opt^D_\mu = 0$ but the guarantee of their algorithm can also be extended to the agnostic case.}; we prove Theorem~\ref{thm:DP_agn_learn_halfspaces_app} here as it is a simple blackbox application of the \densestball algorithm.

% \begin{theorem}
% For every $0< \eps < 1$ and $0 < \mu < 1$, there exists an $\eps$-DP algorithm that runs in time $(nd)^{O_{\mu}(1)} \poly\log(\frac{1}{1-\mu})$, and with probability $0.99$, agnostically learns a halfspace with margin $\mu$, and with additive error $O_{\mu}\left(\frac{d}{\epsilon}\cdot\poly\log\left(\frac{nd}{\eps (1-\mu)}\right)\right)$. Moreover, for every $\delta > 0$, there is an $(\epsilon, \delta)$-DP algorithm with the same runtime but with additive error $O_{\mu} \left(\frac{\sqrt{d}}{\eps}\cdot\poly\log\left(\frac{n d}{\eps \delta}\right)\right)$.
% \end{theorem}

\subsection{\cp}

Finally, we depart from the notion of DP and instead give an application of efficiently list-decodable covers to the \cp problem:

\begin{definition}[\cp]
Given points $x_1, \dots, x_n \in \Z^d$, where each coordinate of $x_i$ is represented as an $L$-bit integer, and an integer $\xi \in \Z$, determine whether there exists $1 \leq i < j \leq n$ such that $\|x_i - x_j\|_2^2 \leq \xi$.
\end{definition}

In the dynamic setting of \cp, we start with an empty set $S$ of points.  At each step, a point maybe added to and removed\footnote{Throughout, we assume without loss of generality that $x$ must belong to $S$ before ``remove $x$'' can be invoked. To make the algorithm work when this assumption does not hold, we simply keep a history-independent data structure that can quickly answer whether $x$ belongs to $S$~\cite{Ambainis07,BernsteinJLM13}.} from $S$, and we have to answer whether there are two distinct points in $S$ whose squared Euclidean distance is at most $\xi$. 

Our main contribution is a faster history-independent data structure for dynamic \cp. Recall that a deterministic data structure is said to be \emph{history-independent} if, for any two sequences of updates that result in the same set of points, the states of the data structure must be the same in both cases. For a randomized data structure, we say that it is history-independent if, for any two sequences of updates that result in the same set of points, the distribution of the state of the data structure must be the same.

\begin{theorem} 
\label{thm:dynamic-cp}
There is a history-independent randomized data structure for dynamic \cp that supports up to $n$ updates, with each update takes $2^{O(d)} \poly(\log n, L)$ time, and uses $O(nd \cdot \poly(\log n, L))$ memory. 
\end{theorem}

We remark that the data structure is only randomized in terms of the layout of the memory (i.e., state), and that the correctness always holds. Our data structure improves that of Aaronson et al.~\cite{Aaronson-cp}, in which the running time per update operation is $d^{O(d)} \poly(\log n, L)$.

Aaronson et al.~\cite{Aaronson-cp} show how to use their data structure together with quantum random walks from~\cite{MagniezNRS11} (see also \cite{Ambainis07,Szegedy04}) to provide a fast quantum algorithm for \cp in low dimensions which runs in time $d^{O(d)} n^{2/3} \poly(\log n, L)$. 
With our improvement above, we immediately obtain a speed up in terms of the dependency on $d$ under the same model\footnote{The model assumes the presence of gates for random access to an $m$-qubit quantum memory that takes time only $\poly(\log m)$. As discussed in~\cite{Ambainis07}, such an assumption is necessary even for element distinctness, which is an easier problem than \cp.}:

\begin{corollary}
There exists a quantum algorithm that solves (offline) \cp with probability 0.99 in time $2^{O(d)} n^{2/3} \poly(\log n, L)$.
\end{corollary}

\paragraph{Proof Overview.} We will now briefly give an outline of the proof of Theorem~\ref{thm:dynamic-cp}. Our proof in fact closely follows that of Aaronson et al.~\cite{Aaronson-cp}. As such, we will start with the common outline before pointing out the differences. At a high-level, both algorithms partition the space $\R^d$ into small cells $C_1, C_2, \dots$, each cell having a diameter at most $\sqrt{\xi}$. Two cells $C, C'$ are said to be \emph{adjacent} if there are $x \in C, x' \in C'$ for which $\|x - x'\|_2^2 \leq \xi$. The main observations here are that (i) if there are two points from the same cell, then clearly the answer to \cp is YES and (ii) if no two points are from the same cell, it suffices to check points from adjacent cells. Thus, the algorithm maintains a map from each present cell to the set of points in the cell, and the counter $p_{\leq \xi}$ of the number of points from different cells that are within $\sqrt{\xi}$ in Euclidean distance. A data structure to maintain such a map is known~\cite{Ambainis07,BernsteinJLM13} (see Theorem~\ref{thm:dynamic-map}). As for $p_{\leq \xi}$, adding/removing a point only requires one to check the cell to which the point belongs, together with the adjacent cells. Thus, the update will be fast, as long as the number of adjacent cells (to each cell) is small.

The first and most important difference between the two algorithms is the choice of the cells. \cite{Aaronson-cp} lets each cell be a $d$-dimensional box of length $\sqrt{\xi / d}$, which results in the number of adjacent cells being $d^{O(d)}$. On the other hand, we use a $(0.5\sqrt{\xi})$-cover from Lemma~\ref{lem:cover-main} and let the cells be the Voronoi cells of the cover. It follows from the list size bound at distance $(1.5\sqrt{\xi})$ that the number of adjacent cells is at most $2^{O(d)}$. This indeed corresponds to the speedup seen in our data structure.

A second modification is that, instead of keeping all points in each cell, we just keep their (bit-wise) XOR. The reason behind this is the observation (i) above, which implies that, when there are more than one point in a cell, it does not matter anymore what exactly these points are. This helps simplify our proof; in particular,~\cite{Aaronson-cp} needs a different data structure to handle the case where there is more than one solution; however, our data structure works naturally for this case.

There are several details that we have glossed over; the full proof will be given in Section~\ref{subsec:online-cp}.

\section{Conclusion and Open Questions}\label{sec:conc_open_questions}
In this work, we obtained tight approximation ratios for several fundamental DP clustering tasks.  An interesting research direction is to study the smallest possible additive error for DP clustering while preserving the tight non-private approximation ratios that we achieve.  Another important direction is to obtain practical implementations of DP clustering algorithms that could scale to large datasets with many clusters.
We focused in this work on the Euclidean metric; it would also be interesting to extend our results to other metric spaces.

\section*{Acknowledgments}

We are grateful to Noah Golowich for providing helpful comments on a previous draft. We also thank Nai-Hui Chia for useful discussions on the quantum \cp problem.

\bibliographystyle{alpha}
\bibliography{refs}

\appendix

\section*{Appendix}

We give some further preliminaries in Section~\ref{sec:add_prelims}. Our algorithms for \densestball in low dimensions are given and analyzed in Section~\ref{sec:densest-ball-low-dim}, and those for \kmeans and \kmedian are presented in Section~\ref{sec:cluster-low-dim}. The resulting algorithms in high dimensions are obtained in Section~\ref{sec:dim_red_cluster_db}. Our results for \onecluster, Sample and Aggregate, agnostic learning of halfspaces with a margin, and \cp are presented in Sections~\ref{sec:one_cluster_from_densest_ball},~\ref{sec:sample_and_aggregate},~\ref{sec:agnostic_halfspaces_margin}, and~\ref{sec:closest-pair} respectively.

\section{Additional Preliminaries}\label{sec:add_prelims}
For any vector $v \in \mathbb{R}^d$, we denote by $\|v\|_2$ its $\ell_2$-norm, which is defined by $\|v\|_2 := \sqrt{\sum_{i=1}^d v_i^2}$; most of the times we simply use $\|v\|$ as a shorthand for $\|v\|_2$.  For any positive real number $\lambda$, the Discrete Laplace distribution $\DLap(\lambda)$ is defined as 
$\DLap(k; \lambda) := \frac{1}{C(\lambda)} \cdot e^{- \frac{|k|}{\lambda}}$ for any $k \in \mathbb{Z}$, where $C(\lambda) := \sum_{k = -\infty}^{\infty} e^{- \frac{|k|}{\lambda}}$ is the normalization constant.

\subsection{Composition Theorems}\label{sec:add_prelim}

We recall the ``composition theorems'' that allow us to easily keep track of privacy losses when running multiple algorithms on the same dataset. 

\begin{theorem}[Basic Composition~\cite{DworkKMMN06}] \label{thm:basic-composition}
For any $\epsilon, \delta \geq 0$ and $k \in \N$, an algorithm that runs $k$ many $(\epsilon, \delta)$-DP algorithms (possibly adaptively) is $(k \epsilon, k \delta)$-DP.
\end{theorem}

It is possible to get better bounds using the following theorem (albeit at the cost of adding a \emph{positive} $\delta'$ parameter).
\begin{theorem}[Advanced Composition~\cite{DworkRV10}] \label{thm:advance-composition}
For any $\epsilon, \delta \geq 0, \delta' > 0$ and $k \in \N$, an algorithm that runs $k$ many $(\epsilon, \delta)$-DP algorithms (possibly adaptively) is
$( 2k\eps(e^\eps - 1) + \eps\sqrt{2k \ln \frac{1}{\delta'}}, k \delta + \delta')$-DP.
%$(\epsilon', k \delta + \delta')$-DP where $\epsilon' = 2k\eps^2 + \eps\sqrt{2k \ln(1/\delta')}$.
\end{theorem}

For an extensive overview of DP, we refer the reader to \cite{DworkR14, vadhan2017complexity}.

\section{\densestball in Low Dimensions}
\label{sec:densest-ball-low-dim}
In this section, we provide our algorithms for \densestball in low dimensions, stated formally below. We start by stating our pure-DP algorithm.

\begin{theorem} \label{thm:pure-1-cluster-given-radius}
For every $\eps > 0$ and $0 < \alpha \leq 1$, there is an $\eps$-DP algorithm that runs in time $(1 + 1/\alpha)^{O(d)} \poly\log(1/r)$ and, with probability $1 - \beta$, returns a $\left(1 + \alpha, O_{\alpha}\left(\frac{d}{\epsilon}\log\left(\frac{1}{\beta r}\right)\right)\right)$-approximation for \densestball, for every $\beta > 0$.
\end{theorem}

We next state our approximate-DP algorithm.

\begin{theorem} \label{thm:apx-1-cluster-given-radius}
For every $\eps > 0$ and $0 < \delta, \alpha  \leq 1$, there is an $(\eps, \delta)$-DP algorithm that runs in time $(1 + 1/\alpha)^{O(d)} \poly\log(1/r)$ and, with probability at least $1 - \beta$, returns a $\left(1 + \alpha, O_{\alpha}\left(\frac{d}{\epsilon}\log\left(\frac{n}{\min\{\eps, 1\} \cdot \beta \delta}\right)\right)\right)$-approximation for \densestball, for every $\beta > 0$.
\end{theorem}

Notice that Theorems~\ref{thm:pure-1-cluster-given-radius} and~\ref{thm:apx-1-cluster-given-radius} imply Theorem~\ref{thm:1-cluster-given-radius-main-body} in Section~\ref{sec:densest-ball-low-dim-main-body}. As discussed there, the main components of our algorithm are efficiently list-decodable covers and algorithms for the \SparseSelection problem, which will be dealt with in the upcoming two subsections. Finally, in Section~\ref{subsec:densest-ball-low-dim-final}, we put the ingredients together to obtain the \densestball algorithms as stated in Theorems~\ref{thm:pure-1-cluster-given-radius} and~\ref{thm:apx-1-cluster-given-radius}.

\subsection{List-Decodable Covers of the Unit Ball} \label{sec:cover}
\newcommand{\VR}{\mathit{VR}}

We start by defining the notion of a $\Delta$-cover and its ``list-decodable'' variant.

\begin{definition}
A \emph{$\Delta$-cover} of the $d$-dimensional unit ball is a set $C \subseteq \R^d$ such that for every point $x$ in the unit ball, there exists $c \in C$ such that $\|c - x\| \leq \Delta$.

Furthermore, we say that a $\Delta$-cover is \emph{list-decodable} at distance $\Delta' \geq \Delta$ with list size $\ell$ if, for any $x$ in the unit ball, we have that $|\{c \in C \mid \|c - x\| \leq \Delta'\}| \leq \ell$. Finally, if there is an algorithm that returns such a list in time $\poly(\ell, d, \log(1/\Delta))$, then we say that the cover is \emph{efficiently} list-decodable.
\end{definition}

We will derive the existence of a certain family of efficiently list-decodable covers, which, as we argue next, can be done by combining tools from the literature on packings, coverings, and lattice algorithms. The properties of the family are stated below.

\begin{lemma} \label{lem:cover-main}
For every $0 < \Delta < 1$, there exists a $\Delta$-cover $C_{\Delta}$ that is efficiently list-decodable at any distance $\Delta' \geq \Delta$ with list size $O(1 + \Delta' / \Delta)^{O(d)}$.
\end{lemma}

Furthermore,
% as our algorithm in Lemma~\ref{lem:pure-sparse-selection} requires an oracle to sample from the candidate set $\cU$,
we will need to be able to quickly sample points from the cover, as stated next:

\begin{lemma} \label{lem:cover-sampling}
For every $0 < \Delta < 1$, there exists a $\poly(1/\Delta, 2^d)$-time algorithm $\cO_{\Delta}$ that samples a random element from the cover $C_{\Delta}$ (given in Lemma~\ref{lem:cover-main}) such that the probability that each element is output is at least $\frac{0.99}{|C_{\Delta}|}$.
\end{lemma}

We prove Lemmas~\ref{lem:cover-main} and~\ref{lem:cover-sampling} in Subsections~\ref{subsec:cover-main} and~\ref{subsec:cover-sampling} respectively. Before doing so, we provide some additional preliminaries in Subsection~\ref{subsec:cover-prelim}.

\subsubsection{Additional Preliminaries on Lattices} \label{subsec:cover-prelim}

We start by defining lattices and related quantities that will be useful in our proofs. Interested readers may refer to surveys and books on the topic such as~\cite{micciancio2012complexity} for more background.

A \emph{basis} is a set of linearly independent vectors. A \emph{lattice} generated by a basis $\bB = \{\bb_1, \dots, \bb_m\}$, denoted by $\cL(\cB)$, is defined as the set $\{\sum_{i=1}^m a_i \bb_i \mid a_1, \dots, a_m \in \Z\}$. The length of the shortest non-zero vector of a lattice $\cL$ is denoted by $\lambda(\cL)$, i.e., $$\lambda(\cL) := \min_{\bv \in \cL, \bv \ne \bzero} \|\bv\|.$$ The \emph{covering radius} of the lattice $\cL(\cB)$ is defined as the smallest $r \in \R^+$ such that every point in $\R^d$ is within a distance of $r$ from some lattice point; more formally, the covering radius is $$\mu(\cL) := \inf\left\{r \in \R^+ \mid \bigcup_{\bv \in \cL} \cB(\bv, r) = \R^d\right\}.$$

The \emph{Voronoi cell} of a lattice $\cL$ is denoted by $\cV(\cL)$ and is defined as the set of points closer to $\bzero$ than to other points of the lattice, i.e., $$\cV(\cL) = \{\by \in \R^d \mid \|\by\| \leq \min_{\bv \in \cL, \bv \ne \bzero} \|\bv - \by\|\}.$$
It is known (see, e.g.,~\cite{MicciancioV13}) that the Voronoi cell can also be defined as the intersection of at most $2(2^d - 1)$ halfspaces of the form $\{\by \in \R^d \mid \|\by\| \leq \|\bv - \by\|\}$ for $\bv \in \cL$. These vectors $\bv$ are said to be the \emph{Voronoi relevant} vectors; we denote the set of Voronoi relevant vectors by $\VR(\cL)$.

We will also use the following simple property of Voronoi relevant vectors. This fact is well-known but we include its proof for completeness.

\begin{observation} \label{obs:voronoi-norm-reduction}
Let $\bv \in \cL$ be a non-zero vector in the lattice. There exists a Voronoi relevant vector $\bv^* \in \VR(\cL)$ such that $\|\bv - \bv^*\| < \|\bv\|$.
\end{observation}

\begin{proof}
Let $\eta > 0$ be the largest real number such that $\eta\bv \in \cV(\cL)$. Notice that $\eta \leq 1/2$, as otherwise $\eta\bv$ is closer to $\bv$ than to $\bzero$. Moreover, $\eta\bv$ must lie on a facet of $\cV(\cL)$; let $\bv^*$ be the Voronoi relevant vector corresponding to this facet. It is obvious that if $\bv^*$ is a multiple of $\bv$, then the claimed statement holds. Otherwise, we have
\begin{align*}
\|\bv - \bv^*\| &= \|(1 - \eta)\bv - (\bv^* - \eta \bv)\| \\
&\leq (1 - \eta)\|\bv\| + \|\bv^* - \eta \bv\| \tag{triangle inequality}\\
&\leq (1 - \eta)\|\bv\| + \|\bzero - \eta\bv\| \tag{from definition of $\bv^*$}\\
&= \|\bv\|.
\end{align*}
Moreover, since we assume that $\bv^*$ is not a multiple of $\bv$, the triangle inequality above must be a strict inequality. As a result, we must have $\|\bv - \bv^*\| < \|\bv\|$ as desired.
\end{proof}

When $\cL$ is clear from the context, we may drop it from the notations and simply write $\lambda, \mu, \cV, \VR$ instead of $\lambda(\cL), \mu(\cL), \cV(\cL), \VR(\cL)$ respectively.

In the \emph{Closest Vector Problem (CVP)}, we are given a target vector $v'$, and the goal is to find a vector $\bv \in \cL(\cB)$ that is closest to $v'$ in the Euclidean metric (i.e., minimizes $\|\bv - v'\|$). It is known that this problem can be solved in time $2^{O(d)}$, as stated more precisely next.

\begin{theorem}[\cite{MicciancioV13}] \label{thm:cvp}
There is a deterministic algorithm that takes a basis $\bB = \{\bb_1, \dots, \bb_m\} \subseteq \R^d$ and a target vector $v' \in \R^d$ where each coordinate of these vectors has bit complexity $M$, and finds the closest vector to $v'$ in $\cL(\bB)$ in time $\poly(M, 2^d)$. Furthermore, the set of Voronoi relevant vectors can be computed in the same time complexity.
\end{theorem}

Note that there are faster \emph{randomized} CVP algorithms~\cite{AggarwalDS15,AggarwalS18} that run in $2^{d + o(d)}\poly(M)$ time; we chose to employ the above algorithm, which is \emph{deterministic}, for simplicity.

\subsubsection{Almost Perfect Lattices and Proof of Lemma~\ref{lem:cover-main}} \label{subsec:cover-main}

For completeness, we will prove Lemma~\ref{lem:cover-main} in this subsection. Many of the proof components are from~\cite{Micciancio04,rogers1959lattice}; in addition, we observe the efficient list-decodability.
First, we have to define the notion of almost perfect lattices~\cite{Micciancio04}, which are the lattices that are simultaneously good packings and coverings:

\begin{definition}
Let $\tau \geq 1$. A lattice $\cL$ is said to be \emph{$\tau$-perfect} if $\mu(\cL) / \lambda(\cL) \leq \tau / 2$.
\end{definition}

It is known that $O(1)$-perfect lattices can be computed in $2^{O(d)}$-time\footnote{The claim in~\cite{Micciancio04} states the running time as $d^{O(d)}$. However, this was just because, at the time of publication of~\cite{Micciancio04}, only $d^{O(d)}$-time algorithms were known for CVP. By plugging the $2^{O(d)}$-time algorithm for CVP of~\cite{MicciancioV13} into the first step of the construction in~\cite{Micciancio04}, the running time of the construction immediately becomes $2^{O(d)}$.}.

\begin{theorem}[\cite{rogers1959lattice,Micciancio04}] \label{thm:almost-perfect-lattice}
There is an algorithm that, given $d \in \N$, runs in $2^{O(d)}$ time and outputs a basis $\cB = \{\bb_1, \dots, \bb_d\}$ such that $\cL(\cB)$ is 3-perfect.
\end{theorem}

With all the previous results stated, we can now easily prove Lemma~\ref{lem:cover-main}.

\begin{proof}[Proof of Lemma~\ref{lem:cover-main}]
We use the algorithm from Theorem~\ref{thm:almost-perfect-lattice} to construct a basis $\cB = \{\bb_1, \dots, \bb_d\}$ that is 3-perfect. By scaling, we may assume that $\mu(\cL(\cB)) \leq \Delta$ and $\lambda(\cL(\cB)) \geq 2\Delta/3$. Our $\Delta$-cover is defined as $C_{\Delta} := \{\bv \in \cL(\cB) \mid \|\bv\| \leq 1 + \Delta\}$.

To list-decode at distance $\Delta'$, we first compute the set $R := \{\bv \in \cL(\cB) \mid \|\bv\| \leq \Delta' + \Delta\}$, as follows. We start from $R = \{\bzero\}$. At each iteration, we go through all vectors $\bw$ in the current set $S$ and all Voronoi relevant vectors $\bv$; if $\|\bw + \bv\| \leq \Delta' + \Delta$, we add $\bw + \bv$ to $S$. We repeat this until no additional vectors are added to $S$. The correctness of the algorithm to construct $S$ follows from Observation~\ref{obs:voronoi-norm-reduction}. Furthermore, since the list of Voronoi relevant vectors can be computed in time $2^{O(d)}$ (Theorem~\ref{thm:cvp}), it is obvious that the algorithm runs in $\poly(|S|, 2^d)$. Now, from $\lambda(\cL(\cB)) \geq 2\Delta/3$, $S$ is a $\Delta/3$-packing. As a result, by a standard volume argument, we have $|S| \leq O(1 + \Delta'/\Delta)^{O(d)}$. In other words, the running time of constructing $S$ is at most $O(1 + \Delta'/\Delta)^{O(d)}$ as desired.

Once we have constructed $S$, we can list-decode $\bx$ at distance $\Delta$ as follows. First, we use the CVP algorithm from Theorem~\ref{thm:cvp} to find the closest vector $\bv \in \cL(\cB)$ to it. Then, we consider $\bv + \bw$ for each $\bw \in S$; if $\|\bv + \bw - \bx\| \leq \Delta'$, we add $\bv + \bw$ into the list. Clearly, this step of the algorithm runs in time $2^{O(d)} + \poly(|S|) = O(1 + \Delta'/\Delta)^{O(d)}$, and this also constitutes the list size bound. Finally, the correctness of this step is also straightforward: for any vector $\bz \in \cL(\cB)$ such that $\|\bz - \bx\| \leq \Delta'$, we must have $\|\bz - \bv\| \leq \|\bz - \bx\| + \|\bv - \bx\| \leq \Delta' + \Delta$, which means that it must be added to the list by our algorithm.
\end{proof}

\subsubsection{Near-Uniform Sampler: Proof of Lemma~\ref{lem:cover-sampling}} \label{subsec:cover-sampling}

Finally, we give a proof of Lemma~\ref{lem:cover-sampling}

\begin{proof}[Proof of Lemma~\ref{lem:cover-sampling}]
The algorithm repeats the following for $W = 100\left(1 + 2\Delta\right)^d$ times: it samples a point $\bx$ uniformly at random from $\cB(0, 1 + 2\Delta)$, uses the CVP algorithm from Theorem~\ref{thm:cvp} to find the closest lattice vector $\bv \in \cL(\cB)$ to $\bx$ and, if $\|\bv\| \leq 1 + \Delta$, it returns $\bv$ and terminates. Otherwise, it returns $\bzero$.

First of all, notice that when the algorithm terminates within $W$ steps, it returns a point uniformly at random from the cover $C_{\Delta}$. Hence, we only have to show that the probability that it does not terminate within the first $W$ steps is at most $0.01$. To see that this is the case, note that the algorithm always terminates if $\|\bx\| \leq 1$; in each iteration, this happens with probability $100/W$. Hence, the probability that this does not happen in the $W$ iterations is only $(1 - 100/W)^W \leq 0.01$.
\end{proof}

We remark that, if we never stop after $W$ iterations, then we would get an algorithm that has an expected running time of $O(W)$ and for which the output distribution is exactly uniform over $C_{\Delta}$. While the exact uniformity seems neat, it turns out that we do not need it anyway in the next section, which leads us to cut off after $W$ iterations so as to get a fixed upper bound on the running time.

\subsection{\SparseSelection}
\label{sec:sparse-selection}
In the \Selection problem, each user $i$ receives a subset $S_i$ of some universe $\cU$. The goal is to output an element $u \in \cU$ that appears in a maximum number of the $S_i$'s. This problem is very well-studied in the DP literature, and tight bounds are known in a large regime of parameters both in the central~\cite{SteinkeU17,BafnaU17} and the local~\cite{Ullman18} models.

However, known algorithms~\cite{McSherryT07,DworkNRRV09}\footnote{See also Section 3.6 of~\cite{DworkR14} for a concise description of how~\cite{DworkNRRV09} can be applied to \Selection.} for \Selection run in time $\Omega(|\cU|)$ which can be large; specifically, this will be insufficient for our application to private clustering where $|\cU|$ is super-polynomial. Instead, we will consider a restriction of the problem where we have an upper bound $\ell$ on the sizes of the $S_i$'s, and show that, under certain assumptions, we can solve \Selection in this case with running time polynomial in $\ell$ and $\log |\cU|$.

\begin{definition}[\SparseSelection] \label{def:sparse-selection}
For a positive integer $\ell$, the input to the $\ell$-\SparseSelection problem is a list $\bS = (S_1, \dots, S_n)$ of subsets, where $S_1, \dots, S_n \in \binom{\cU}{\leq \ell}$ for some finite universe $\cU$. We say that an algorithm solves the $\ell$-\SparseSelection problem with additive error $t$ if it outputs a universe element $\hat{u} \in U$ such that
\begin{align*}
\left|\{i \mid \hat{u} \in S_i\}\right| \geq \max_{u \in \cU} \left|\{i \mid u \in S_i\}\right|  - t.
\end{align*}
\end{definition}

Throughout this section, we assume that each universe element of $\cU$ can be represented by a $\poly\log|\cU|$-bit string, but that $\cU$ itself is not explicitly known. (This is the case for lattice covers from the previous subsection, where each element of the covers can be represented by the coefficients.) We will give two simple $\poly(n, \ell, \log |\cU|)$-time algorithms for the problem, both of which are variants of the Exponential Mechanism of McSherry and Talwar~\cite{McSherryT07}.

Our first algorithm is an approximate-DP algorithm with an additive error independent of the universe size $|\cU|$; furthermore, this algorithm does not require any additional assumption.

\begin{lemma}[Approximate-DP Algorithm for \SparseSelection] \label{lem:apx-sparse-selection}
For every $\eps > 0$ and $0 < \delta \leq 1$, there is a $\poly(n, \ell, \log |\cU|)$-time $(\epsilon, \delta)$-DP algorithm that, with probability at least $1 - \beta$, outputs a universe element that solves the $\ell$-\SparseSelection problem with additive error $O\left(\frac{1}{\eps}\log\left(\frac{n\ell}{\min\{\eps, 1\} \cdot \delta \beta}\right)\right)$, for every $\beta \in (0, 1)$.
\end{lemma}

Next, we give a pure-DP algorithm for the problem. This algorithm is nearly identical to the original Exponential Mechanism of McSherry and Talwar~\cite{McSherryT07} except that, instead of going over all elements of $\cU$ in the algorithm itself, we assume that there is an oracle $\cO$ that can sample an approximately uniformly random element from $\cU$. 

\begin{lemma}[Pure-DP Algorithm for \SparseSelection] \label{lem:pure-sparse-selection}
Suppose there is an oracle $\cO$ that runs in time $\poly \log |\cU|$ and outputs a sample from $\cU$ such that the probability of outputting each element $u \in \cU$ is at least $p > 0$. Then, for every $\eps > 0$, there is a $\poly(n, \ell, \log |\cU|)$-time $\epsilon$-DP algorithm that, with probability at least $1 - \beta$, outputs a universe element that solves the $\ell$-\SparseSelection problem with additive error $O\left(\frac{1}{\eps} \ln\left(\frac{1}{\beta p}\right)\right)$, for every $\beta \in (0, 1)$.
\end{lemma}

We remark that the approximate-DP algorithm in Lemma~\ref{lem:apx-sparse-selection} has an additive error that does not grow with $|\cU|$, whereas the pure-DP algorithm in Lemma~\ref{lem:pure-sparse-selection} incurs an additive error that depends (at least) logarithmically on $|\cU|$ because $p$ can be at most $\frac{1}{|\cU|}$. It is simple to see that this $\log(|\cU|)$ dependency of the pure-DP algorithm is necessary even when $\ell = 1$. Finally, note that Lemmas~\ref{lem:pure-sparse-selection} and~\ref{lem:apx-sparse-selection} imply Lemmas~\ref{lem:pure-sparse-selection-main-body} and~\ref{lem:apx-sparse-selection-main-body} in Section~\ref{subsec:sparse-selection-main-body}, respectively.

We next prove Lemma~\ref{lem:apx-sparse-selection} in Section~\ref{subsec:approx_DP_sparse_sel} and Lemma~\ref{lem:pure-sparse-selection} in Section~\ref{subsec:pure_DP_sparse_sel}.

\subsubsection{Approximate-DP Algorithm}\label{subsec:approx_DP_sparse_sel}

This section is devoted to the proof of Lemma~\ref{lem:apx-sparse-selection}. On a high level, the algorithm runs the Exponential Mechanism on the union $S_1 \cup \cdots \cup S_n$, with a small modification: we have an additional candidate $\perp$ whose score is fixed. We prove below that, when the score of $\perp$ is set to be sufficiently large (i.e., $O\left(\frac{1}{\eps}\log\left(\frac{\ell}{\eps \delta}\right)\right)$), the resulting algorithm is $(\eps, \delta)$-DP.
% The accuracy guarantee of this algorithm is essentially the same as that of the standard Exponential Mechanism.

\begin{algorithm}[h!]
\caption{Approximate-DP Algorithm for \SparseSelection.}\label{alg:apx-sparse-selection}
\begin{algorithmic}[1]
\Procedure{ApxSparseSelection$(\bS = (S_1, \dots, S_n))$}{}
\State $\cU(\bS) \leftarrow S_1 \cup \cdots \cup S_n$.
\For{$u \in \cU(\bS)$}
\State $score_{\bS}[u] \leftarrow |\{i \mid u \in S_i\}|$
\EndFor
\State $score_{\bS}[\perp] \leftarrow \frac{2}{\eps}\left(1 + \ln\left(\frac{\ell}{\delta(1 - e^{-\eps/2})}\right)\right)$
\State \Return a value drawn from $\cU(\bS) \cup \{\perp\}$ where $u$ has probability $\frac{e^{(\eps/2) \cdot score_{\bS}[u]}}{\sum_{u \in \cU(\bS) \cup \{\perp\}} e^{(\eps/2) \cdot score_{\bS}[u]}}$
\EndProcedure
\end{algorithmic}
\end{algorithm}

\begin{proof}[Proof of Lemma~\ref{lem:apx-sparse-selection}]
We now prove that Algorithm~\ref{alg:apx-sparse-selection} satisfies the desired privacy and accuracy guarantees. For brevity, we use $\cM$ as a shorthand for the mechanism \textsc{ApxSparseSelection}. It is immediate that the algorithm runs in time $\poly(n, \ell, \log|\cU|)$, as desired.

\paragraph{Privacy.}
Consider any pair of neighboring input datasets $\bS$ and $\bS'$.
Recall that to show that the algorithm is $(\eps, \delta)$-DP, it suffices to show that
\begin{equation}\label{eq:DP_def_ineq}
\Pr_{o \sim \cM(\bS)}\left[\frac{\Pr[o = \cM(\bS)]}{\Pr[o = \cM(\bS')]} > e^{\eps}\right] \leq \delta.
\end{equation}

To prove the inequality in~(\ref{eq:DP_def_ineq}), let $score_{\perp}$ be the (fixed) score of $\perp$. Additionally, we denote
$$Z_{\bS} := \sum_{u \in \cU(\bS) \cup \{\perp\}} e^{(\eps/2) \cdot score_{\bS}[u]},$$
$$Z_{\bS' } := \sum_{u \in \cU(\bS') \cup \{\perp\}} e^{(\eps/2) \cdot score_{\bS'}[u]}.$$
First, we will argue that $Z_{\bS} \geq e^{-\eps/2} \cdot Z_{\bS'}$.
This holds because
\begin{align}
Z_{\bS} &= \sum_{u \in \cU(\bS) \cup \{\perp\}} e^{(\eps/2) \cdot score_{\bS}[u]}\nonumber\\
&\geq \left(\sum_{u \in \cU(\bS) \cap \cU(\bS')} e^{(\eps/2) \cdot score_{\bS}[u]}\right) + e^{(\eps / 2) \cdot score_{\perp}}\nonumber \\
&\geq \left(\sum_{u \in \cU(\bS) \cap \cU(\bS')} e^{(\eps/2) \cdot (score_{\bS'}[u] - 1)}\right) + e^{(\eps / 2) \cdot score_{\perp}}\nonumber \\
&= e^{-\eps/2} \cdot Z_{\bS'} - \left(\sum_{u \in \cU(\bS') \setminus \cU(\bS)} e^{(\eps/2) \cdot (score_{\bS'}[u] - 1)}\right) + e^{(\eps / 2) \cdot score_{\perp}} \cdot \left(1 - e^{-\eps/2}\right).\label{eq:Z_S_lb}
\end{align}
Now observe that if $u$ belongs to $\cU(\bS') \setminus \cU(\bS)$, it must belong to a single set in $\bS'$ or equivalently $score_{\bS'}[u] = 1$. Furthermore, since each set has size at most $\ell$, we have $|\cU(\bS') \setminus \cU(\bS)| \leq \ell$. Plugging this into~(\ref{eq:Z_S_lb}), we get
\begin{align}
Z_{\bS}
&\geq e^{-\eps/2} \cdot Z_{\bS'} - \ell + e^{(\eps / 2) \cdot score_{\perp}} \cdot \left(1 - e^{-\eps/2}\right) \nonumber \\
&\geq e^{-\eps / 2} \cdot Z_{\bS'}, \label{eq:normalization-ratio}
\end{align}
where the last inequality holds from our setting of $score_{\perp}$ in Algorithm~\ref{alg:apx-sparse-selection}.

For every $u \in (\cU(S) \cap \cU(S')) \cup \{\perp\}$, we thus get that
\begin{align*}
\frac{\Pr[u = \cM(\bS)]}{\Pr[u = \cM(\bS')]} 
&= \frac{e^{\eps/2 \cdot score_{\bS}[u]} / Z_{\bS}}{e^{\eps/2 \cdot score_{\bS'}[u]} / Z_{\bS'}} \\
&\leq \frac{e^{\eps/2 \cdot (score_{\bS'}[u] + 1)} / Z_{\bS}}{e^{\eps/2 \cdot score_{\bS'}[u]} / Z_{\bS'}} \\
&\leq e^{\eps},
\end{align*}
where the last inequality follows from~(\ref{eq:normalization-ratio}) above. As a result, we obtain
\begin{align*}
\Pr_{u \sim \cM(\bS)}\left[\frac{\Pr[u = \cM(\bS)]}{\Pr[u = \cM(\bS')]} > e^{\eps}\right] 
&\leq \Pr_{u \sim \cM(\bS)}[u \in \cU(\bS) \setminus \cU(\bS')] \\
&= \sum_{u \in \cU(\bS) \setminus \cU(\bS')} \frac{e^{\eps/2 \cdot score_{\bS}[u]}}{Z_{\bS}} \\
&= \sum_{u \in \cU(\bS) \setminus \cU(\bS')} \frac{e^{\eps/2}}{Z_{\bS}} \\
&\leq \frac{\ell \cdot e^{\eps/2}}{Z_{\bS}} \\
&\leq \frac{\ell \cdot e^{\eps/2}}{e^{\eps/2 \cdot score_{\perp}}} \\
&\leq \delta,
\end{align*}
where the second equality uses the fact that $score_{\bS}[u] = 1$ whenever $u \in \cU(\bS) \setminus \cU(\bS')$, and the last inequality follows from our setting of $score_{\perp}$ in Algorithm~\ref{alg:apx-sparse-selection}. Thus, Algorithm~\ref{alg:apx-sparse-selection} is $(\eps, \delta)$-DP as claimed.
% (Note that above we once again use the fact that $u \in \cU(\bS) \setminus \cU(\bS')$ implies that $u \in S_n$ and $score_{\bS}[u] = 1$.)

\paragraph{Accuracy.} We will now show that, with probability at least $1-\beta$, Algorithm~\ref{alg:apx-sparse-selection} outputs a universe element that solves the \SparseSelection problem with additive error\footnote{Note that $1 - e^{-\eps/2} \geq 0.5\min\{1, \eps\}$, which implies that $t = O\left(\frac{1}{\eps}\log\left(\frac{n\ell}{\min\{\epsilon, 1\} \cdot \delta \beta}\right)\right)$.} $t = score_{\perp} + \frac{2}{\eps}\ln\left(\frac{2n\ell}{\beta}\right)$. To do so, we let $\opt := \max_{u \in \cU} \left|\{i \mid u \in S_i\}\right|$. If $\opt \leq t$, the statement trivially holds. If $\opt > t$, we let $\cU_{good} := \{u \in \cU \mid \left|\{i \mid u \in S_i\}\right| \geq \opt - t\}$. Let $Z_{good} := \sum_{u \in \cU_{good}} e^{(\eps/2) \cdot score_{\bS}[u]}$. Note that $Z_{good} \geq e^{(\eps/2) \cdot \opt}$. We therefore have that
\begin{align*}
\Pr_{u \sim \cM(\bS)}[u \notin \cU_{good}] &= 1 - \frac{Z_{good}}{Z_{\bS}} \\
&= \frac{e^{(\eps/2) \cdot score_{\perp}} + \sum_{u \in \cU(S) \setminus \cU_{good}} e^{(\eps/2) \cdot score_{\bS}[u]}}{Z_{good} + e^{(\eps/2) \cdot score_{\perp}} + \sum_{u \in \cU(S) \setminus \cU_{good}} e^{(\eps/2) \cdot score_{\bS}[u]}} \\
&\leq \frac{e^{(\eps/2) \cdot score_{\perp}} + n\ell \cdot e^{(\eps/2) \cdot (\opt - t)}}{Z_{good}} \\
&\leq e^{(\eps/2) \cdot (score_{\perp} - \opt)} + n\ell \cdot e^{(\eps/2) \cdot (- t)} \\
&\leq \beta,
\end{align*}
where the first inequality follows from the fact that $|\cU(\bS)| \leq |S_1| + \cdots + |S_n| \leq n\ell$, and the last inequality follows from our setting of $t$ and from the assumption that $\opt > t$. We thus conclude that the output of Algorithm~\ref{alg:apx-sparse-selection}, with probability at least $1 - \beta$, solves \SparseSelection with additive error $t$ as desired.
\end{proof}

\subsubsection{Pure-DP Algorithm}\label{subsec:pure_DP_sparse_sel}

We next prove Lemma~\ref{lem:pure-sparse-selection}. It relies on Algorithm~\ref{alg:pure-sparse-selection}, which is very similar to Algorithm~\ref{alg:apx-sparse-selection} for approximate-DP, except that (i) instead of returning $\perp$, we draw from the oracle $\cO$ and return its output, and (2) for each $u \in \cU(S)$, we adjust the probability of sampling it directly to offset the probability that it is returned by $\cO$. (Below the ``adjusted multiplier'' is $q_{\bS}[u]$, which serves similar purpose to $e^{score_{\bS}[u]}$ in the vanilla Exponential Mechanism.)

\begin{algorithm}[h!]
\caption{Pure-DP Algorithm for \SparseSelection.}\label{alg:pure-sparse-selection}
\begin{algorithmic}[1]
\Procedure{PureSparseSelection$_{\cO}(\bS = (S_1, \dots, S_n))$}{}
\State $\cU(\bS) \leftarrow S_1 \cup \cdots \cup S_n$.
\For{$u \in \cU(\bS)$}
\State $score_{\bS}[u] \leftarrow |\{i \mid u \in S_i\}|$
\State $q_{\bS}[u] \leftarrow e^{(\eps/2) \cdot score_{\bS}[u]} - 1$
\EndFor
\State $score_{\bS}[\perp] \leftarrow \frac{2}{\eps}\ln\left(\frac{1}{p}\right)$
\State $q_{\bS}[\perp] \leftarrow e^{(\eps/2) \cdot score_{\bS}[\perp]}$
\State $\hat{u} \leftarrow$ a value drawn from $\cU(\bS) \cup \{\perp\}$ where $u$ has probability $\frac{q_{\bS}[u]}{\sum_{u \in \cU(\bS) \cup \{\perp\}} q_{\bS}[u]}$
\If{$\hat{u} = \perp$}
\State \Return an output from a call to $\cO$
\Else
\State \Return $\hat{u}$
\EndIf
\EndProcedure
\end{algorithmic}
\end{algorithm}

\begin{proof}[Proof of Lemma~\ref{lem:pure-sparse-selection}]
We now prove that Algorithm~\ref{alg:pure-sparse-selection} yields the desired privacy and accuracy guarantees. For brevity, we use $\cM$ as a shorthand for the mechanism \textsc{PureSparseSelection}. It is immediate that Algorithm~\ref{alg:pure-sparse-selection} runs in time $\poly(n, \ell, \log|\cU|)$, as desired.

\paragraph{Privacy.} For every $u \in \cU$, we let $p_{\cO}(u) \geq p$ denote the probability that the oracle $\cO$ outputs $u$. For convenience, when $u \notin \cU(\bS)$, we set $score_{\bS}[u]$ to $0$. We define 
\begin{align*}
\widetilde{score}_{\bS}[u] 
&:= \frac{2}{\eps} \cdot \ln\left(e^{(\eps/2) \cdot score_{\perp}} \cdot p_{\cO}(u) + \bone[u \in \cU(\cS)] \cdot (e^{(\eps/2) \cdot score_{\bS}[u]} - 1) \right) \\
&= \frac{2}{\eps} \cdot \ln\left(e^{(\eps/2) \cdot score_{\perp}} \cdot p_{\cO}(u) + (e^{(\eps/2) \cdot score_{\bS}[u]} - 1) \right).
\end{align*}
We observe that for an input $\bS = (S_1, \cdots, S_n)$, the probability that each $u^* \in \cU$ is selected is exactly $\frac{e^{(\eps/2) \cdot \widetilde{score}_{\bS}(u^*)}}{\sum_{u \in \cU} e^{(\eps/2) \cdot \widetilde{score}_{\bS}(u)}}$. Thus, Algorithm~\ref{alg:pure-sparse-selection} is equivalent to running the exponential mechanism of~\cite{McSherryT07} with the scoring function $\widetilde{score}_{\bS}$. Hence, to prove that Algorithm~\ref{alg:pure-sparse-selection} is $\eps$-DP, it suffices to show that the sensitivity of $\widetilde{score}_{\bS}[u]$ is at most $1$. Consider any two neighboring datasets $\bS$ and $\bS'$. Due to symmetry, it suffices to show that 
\begin{align*}
\widetilde{score}_{\bS}[u] - \widetilde{score}_{\bS'}[u] \leq 1,
\end{align*}
which is equivalent to
\begin{align} \label{eq:main-ratio-exp}
\frac{e^{(\eps/2) \cdot score_{\perp}} \cdot p_{\cO}(u) + (e^{(\eps/2) \cdot score_{\bS}[u]} - 1)}{e^{(\eps/2) \cdot score_{\perp}} \cdot p_{\cO}(u) + (e^{(\eps/2) \cdot score_{\bS'}[u]} - 1)} \leq e^{\eps/2}.
\end{align}

To prove~\eqref{eq:main-ratio-exp}, notice that $e^{(\eps/2) \cdot score_{\perp}} = 1 / p$. As a result, we have
\begin{align*}
\frac{e^{(\eps/2) \cdot score_{\perp}} \cdot p + (e^{(\eps/2) \cdot score_{\bS}[u]} - 1)}{e^{(\eps/2) \cdot score_{\perp}} \cdot p + (e^{(\eps/2) \cdot score_{\bS'}[u]} - 1)} &= \frac{e^{(\eps/2) \cdot score_{\bS}[u]} }{e^{(\eps/2) \cdot score_{\bS'}[u]}}
\leq e^{\eps/2}.
\end{align*}
This, together with $p_{\cO}(u) \geq p$, implies that~\eqref{eq:main-ratio-exp} holds, and hence our algorithm is $\eps$-DP as desired.

\paragraph{Accuracy.} The accuracy analysis is very similar to the proof of Lemma~\ref{lem:apx-sparse-selection}.
Specifically, we will now show that, with probability at least $1-\beta$, Algorithm~\ref{alg:pure-sparse-selection} outputs a universe element that solves the \SparseSelection problem with additive error\footnote{Notice that $t = \frac{2}{\eps} \ln\left(\frac{2|\cU|}{\beta p}\right) = O\left(\frac{1}{\eps} \ln\left(\frac{1}{\beta p}\right)\right)$, where the inequality holds because $p \leq 1 / |\cU|$.} $t = score_{\perp} + \frac{2}{\eps}\ln\left(\frac{2|\cU|}{\beta}\right)$. To do so, we let $\opt := \max_{u \in \cU} \left|\{i \mid u \in S_i\}\right|$. If $\opt \leq t$, the statement trivially holds. If $\opt > t$, we let $\cU_{good} := \{u \in \cU \mid \left|\{i \mid u \in S_i\}\right| \geq \opt - t\}$. Let $Z_{good} := \sum_{u \in \cU_{good}} e^{(\eps/2) \cdot \widetilde{score}_{\bS}[u]}$. Note that $Z_{good} \geq e^{(\eps/2) \cdot \opt}$. Also, let $Z_{\bS} := \sum_{u \in \cU} e^{(\eps/2) \cdot \widetilde{score}_{\bS}[u]}$. We therefore have that
\begin{align*}
\Pr_{u \sim \cM(\bS)}[u \notin \cU_{good}] &= 1 - \frac{Z_{good}}{Z_{\bS}} \\
&\leq \frac{e^{(\eps/2) \cdot score_{\perp}} + \sum_{u \in \cU \setminus \cU_{good}} (e^{(\eps/2) \cdot score_{\bS}[u]} - 1)}{Z_{\bS}} \\
&\leq \frac{e^{(\eps/2) \cdot score_{\perp}} + |\cU| \cdot e^{(\eps/2) \cdot (\opt - t)}}{Z_{good}} \\
&\leq e^{(\eps/2) \cdot (score_{\perp} - \opt)} + |\cU| \cdot e^{(\eps/2) \cdot (- t)} \\
&\leq \beta,
\end{align*}
where the last inequality follows from our setting of $t$ and from the assumption that $\opt > t$. We thus conclude that the output of Algorithm~\ref{alg:pure-sparse-selection}, with probability at least $1 - \beta$, solves \SparseSelection with additive error $t$ as desired.
\end{proof}

\subsection{Putting Things Together}
\label{subsec:densest-ball-low-dim-final}

Having set up all the ingredients in Sections~\ref{sec:sparse-selection} and~\ref{sec:cover}, we now put them together to derive our DP algorithm for \densestball in low dimensions. The idea is to run Algorithm~\ref{alg:1-cluster}, where the algorithm for \SparseSelection is either from Lemma~\ref{lem:apx-sparse-selection} or Lemma~\ref{lem:pure-sparse-selection}.

\begin{algorithm}[h!]
\caption{\densestball Algorithm.}\label{alg:1-cluster}
\begin{algorithmic}[1]
\Procedure{DensestBallLowDimension$(x_1, \dots, x_n; r, \alpha)$}{}
\State $C_{\alpha r} \leftarrow$ $\alpha r$-cover from Lemma~\ref{lem:cover-main}
\For{$i \in \{1, \dots, n\}$}
\State $S_i \leftarrow$ decoded list of $x$ at distance $(1 + \alpha)r$ with respect to $C_{\alpha r}$
\EndFor
\Return \SparseSelection$(S_1, \dots, S_n)$ \label{step:sparse-selection-call}
\EndProcedure
\end{algorithmic}
\end{algorithm}

When we set \SparseSelection on Line~\ref{step:sparse-selection-call} to be the pure-DP algorithm for \SparseSelection from Lemma~\ref{lem:pure-sparse-selection}, we obtain the pure-DP algorithm for \densestball in low dimensions (Theorem~\ref{thm:pure-1-cluster-given-radius}).

\begin{proof}[Proof of Theorem~\ref{thm:pure-1-cluster-given-radius}]
We run Algorithm~\ref{alg:1-cluster} with \SparseSelection being the $\eps$-DP algorithm from Lemma~\ref{lem:pure-sparse-selection} using the oracle $\cO$ from Lemma~\ref{lem:cover-sampling} for $C_{\alpha r}$. Recall that the list size $\ell$ guarantee from Lemma~\ref{lem:cover-main} is $((1 + \alpha) / \alpha)^{O(d)} = (1 + 1/\alpha)^{O(d)}$. Hence, the running time of the algorithm is $\poly(\ell, d, \log(1/r)) = (1 + 1/\alpha)^{O(d)} \poly\log(1/r)$ as desired.

The privacy of the algorithm follows immediately from the $\eps$-DP of the \SparseSelection algorithm. Finally, to argue about its accuracy, assume that there exists a ball $\cB(c^*, r)$ that contains at least $T$ of the input points. Since $C_{\alpha r}$ is an $\alpha r$-cover of the unit ball, there exists $c \in C_{\alpha r}$ such that $\|c - c^*\| \leq \alpha r$. As a result, $\cB(c, (1 + \alpha)r)$ contains at least $T$ of the input points, which means that $c$ belongs to the decoded list $S_i$ of these points. By Lemma~\ref{lem:pure-sparse-selection}, the algorithm \SparseSelection outputs, with probability at least $1 - \beta$, a center $c'$ that belongs to at least $T - O\left(\frac{1}{\eps}\log\left(\frac{1}{\beta p}\right)\right) = T - O\left(\frac{1}{\eps}\log\left(\frac{|C_{\alpha r}|}{\beta}\right)\right) = T - O_\alpha\left(\frac{d}{\epsilon}\log\left(\frac{1}{\beta \alpha r}\right)\right)$ decoded lists $S_i$'s. This indeed means that $c'$ is a $\left(1 + \alpha, O_{\alpha}\left(\frac{d}{\epsilon}\log\left(\frac{1}{\beta r}\right)\right)\right)$-approximate solution, as desired.
\end{proof}

We similarly obtain an approximate-DP algorithm for \densestball with possibly smaller additive error than in Theorem~\ref{thm:pure-1-cluster-given-radius} by setting \SparseSelection to be the approximate-DP algorithm for \SparseSelection from Lemma~\ref{lem:apx-sparse-selection}:

\begin{proof}[Proof of Theorem~\ref{thm:apx-1-cluster-given-radius}]
The proof of this theorem is exactly the same as that of Theorem~\ref{thm:pure-1-cluster-given-radius}, except that \SparseSelection is chosen as the $(\eps, \delta)$-DP algorithm from Lemma~\ref{lem:apx-sparse-selection}. 
\end{proof}

\section{\kmeans and \kmedian in Low Dimensions} \label{sec:cluster-low-dim}

In this section, we use our algorithm for \densestball in low dimensions from Section~\ref{sec:densest-ball-low-dim} to obtain DP approximation algorithms for \kmeans and \kmedian, culminating in the proofs of the following theorems, which essentially matches the approximation ratios in the non-private case:

\begin{theorem} \label{thm:main-apx-non-private-to-private}
For any $p \geq 1$, suppose that there is a polynomial-time (not necessarily private) $w$-approximation algorithm %$\cA_{non\text{-}private}$
 for $(k, p)$-Clustering. Then, for every $\eps > 0$ and $0 < \alpha \leq 1$, there is an $\eps$-DP algorithm that runs in time $2^{O_{p, \alpha}(d)} \cdot \poly(n)$ and, with probability $1 - \beta$, outputs a $\left(w(1 + \alpha), O_{p, \alpha, w}\left(\frac{k^2 \log^2 n \cdot 2^{O_{p, \alpha}(d)}}{\eps} \log\left(\frac{n}{\beta}\right) + 1\right)\right)$-approximation for $(k, p)$-Clustering, for every $\beta \in (0, 1)$.
\end{theorem}

\begin{theorem} \label{thm:fpt-apx}
For every $\eps > 0$, $0 < \alpha \leq 1$ and $p \geq 1$, there is an $\eps$-DP algorithm that runs in time $2^{O_{\alpha, p}(d k + k \log k)} \cdot \poly(n)$ and, with probability $1 - \beta$, outputs an $\left(1 + \alpha, O_{\alpha, p}\left(\frac{d k^2 \log n}{\eps} \log\left(\frac{n}{\beta}\right) + 1\right)\right)$-approximation for $(k, p)$-Clustering, for every $\beta \in (0, 1)$.
\end{theorem}

Note here that Theorem~\ref{thm:main-apx-non-private-to-private} implies Theorem~\ref{thm:main-apx-non-private-to-private-pure-intro-dim-red} in Section~\ref{sec:kmeans-kmedian-main-body}.

The structure of the proof of Theorem~\ref{thm:main-apx-non-private-to-private} closely follows the outline in Section~\ref{sec:kmeans-kmedian-main-body}. First, in Section~\ref{sec:coarse-centroid-set}, we construct a centroid set with $w = O(1)$ by repeated applications of \densestball. From that point on, we roughly follow the approach of~\cite{feldman2009private,HarPeledM04}. Specifically, in Section~\ref{sec:centroid-refinement}, we refine our centroid set to get $w = 1 + \alpha$ using exponential covers. Then, in Section~\ref{sec:private-coreset}, we argue that the noisy snapped points form a private coreset with $\gamma$ arbitrarily close to zero. Finally, in Section~\ref{sec:apx-from-coreset}, we put things together and obtain a proof of Theorem~\ref{thm:main-apx-non-private-to-private}.

While this approach also yields an FPT algorithm with approximation ratio $1 + \alpha$, the additive errors will depend exponentially on $d$ (as in Theorem~\ref{thm:main-apx-non-private-to-private}). In this case, the error can be reduced to $\poly(d, k, \log n, 1/\eps)$ as stated in Theorem~\ref{thm:fpt-apx}. Roughly speaking, we can directly run the Exponential Mechanism on the refined coreset. This is formalized in Section~\ref{sec:discrete-algo}.

%\subsection{Constructing a Centroid Set}
%\label{sec:centroid-set}

%In this section, we give an algorithm that constructs good centroid sets. The algorithms presented here will be used in both Sections~\ref{sec:non-private-ratio} and~\ref{sec:discrete-algo}.

%\subsubsection{Coarse Centroid Set via Repeated Invocations of \densestball}
\subsection{Coarse Centroid Set via Repeated Invocations of \densestball}
\label{sec:coarse-centroid-set}

The first step in our approximation algorithm is to construct a ``coarse'' centroid set (with $w = O(1)$) by repeatedly applying our \densestball algorithm\footnote{Here we only require the approximation ratio to be some constant for \densestball, which is fixed to $2$ in the algorithm itself.}, while geometrically increasing the radius $r$ with each call. Each time a center is found, we also remove points that are close to it. The procedure is described more precisely below as Algorithm~\ref{alg:important-cluster}. (Here we use $\bzero$ to denote the origin in $\R^d$.)

% The first step towards constructing a good centroid set is to construct a coarse centroid set by repeatedly applying our \densestball algorithm\footnote{Note that here we only require the approximation ratio to be some constant, which is fixed to $2$to  in the algorithm itself.}, while geometrically increasing the radius $r$ with each call. Each time a center is found, we also remove points that are close to it. The procedure is described more precisely below as Algorithm~\ref{alg:important-cluster}. 

\begin{algorithm}[h!]
\caption{Finding Coarse Centroid Set.}\label{alg:important-cluster}
\begin{algorithmic}[1]
\Procedure{CoarseCentroidSet$^\eps(x_1, \dots, x_n)$}{}
\State $\bX_{uncovered} \leftarrow (x_1, \dots, x_n)$
\State $\cC \leftarrow \{\bzero\}$
\For{$i \in \{1, \dots, \lceil\log n\rceil\}$}
\State $r \leftarrow 2^i / n$
%\State $\bX_{uncovered} \leftarrow \bX_{uncovered} \setminus \cB(\cC, a r)$
\For{$j=1 ,\dots, 2k$}
\State $c_{i, j} \leftarrow\textsc{DensestBallLowDimension}(\bX_{uncovered}; r, 1)$ \label{step:1-cluster}
\State $\cC \leftarrow \cC \cup \{c\}$
\State $\bX_{uncovered} \leftarrow \bX_{uncovered} \setminus \cB(c, 8 r)$ \label{eq:ball-removal}
%\State Remove from $\bX_{uncovered}$ any points within distance $8r$ from $c$ \label{eq:ball-removal}
\EndFor
\EndFor
\Return $\cC$
\EndProcedure
\end{algorithmic}
\end{algorithm}

We can show that the produced set $\cC$ is a centroid set with approximation ratio $w = O(1)$. In fact, below we state an even stronger property that for every $c$ and $r$ where the ball $\cB(c, r)$ contains many points, at least one of the point in $\cC$ is close to $c$. Throughout this section, we write $\opt$ as a shorthand for $\opt^{p, k}_{\bX}$. %(It is simple to verify that this property implies that $\cC$ is an $(O(1), \poly(k, \log n, 1/\eps))$-centroid set; we do not fully state this corollary because this exact fact is not subsequently used.)

\begin{lemma} \label{lem:good-center}
For any $d \in \N, \eps > 0$, and $0 < r, \alpha, \beta \leq 1$, let $T_{d, \epsilon, \beta, r,\alpha} = O_\alpha\left(\frac{d}{\eps} \log\left(\frac{1}{\beta r}\right)\right)$ be the additive error guarantee from Theorem~\ref{thm:pure-1-cluster-given-radius}. Furthermore, let $T^*$ be a shorthand for $T_{d, \frac{\eps}{2 k\lceil \log n\rceil},\frac{\beta}{2 k \lceil \log n \rceil}, \frac{1}{n}, 1} = O\left(\frac{d k \log n}{\eps} \log\left(\frac{n}{\beta}\right)\right)$. 

For every $\eps > 0$, there is a $2^{O(d)} \poly(n)$-time $\eps$-DP algorithm that outputs a set $\cC \subseteq \R^d$ of size $O(k \log n)$ which, for every $\beta \in (0, 1)$, satisfies the following with probability at least $1 - \beta$: for all $c \in \R^d$ and $r \in \left[\frac{1}{n}, 1\right]$ such that $n_{c, r} := |\bX \cap \cB(c, r)|$ is at least $2T^*$, there exists $c' \in \cC$ such that $\|c - c'\| \leq 18 \cdot \max\left\{r, \left(\frac{2\opt}{n_{c, r} k}\right)^{1/p}\right\}$.
\end{lemma}

Before we prove Lemma~\ref{lem:good-center}, let us note that it immediately implies that the output set is an $\left(O_p(1), O_p\left(\frac{d k^2 \log n}{\eps} \log\left(\frac{n}{\beta}\right) + 1\right)\right)$-centroid set, as stated below. Nonetheless, we will not use this fact directly in subsequent steps since the properties in Lemma~\ref{lem:good-center} are stronger and more convenient to use.

\begin{corollary} \label{cor:coarse-centroid-set}
For every $\eps > 0$ and $p \geq 1$, there is an $2^{O(d)}\poly(n)$-time $\eps$-DP algorithm that, with probability $1 - \beta$, outputs an $\left(O_p(1), O_p\left(\frac{d k^2 \log n}{\eps} \log\left(\frac{n}{\beta}\right) + 1\right)\right)$-centroid set for $(k, p)$-Clustering of size $O(k \log n)$, for every $\beta \in (0, 1)$.
\end{corollary}

Note that Corollary~\ref{cor:coarse-centroid-set} implies Lemma~\ref{lem:densest-ball-to-kmeans-kmedian-main-body} in Section~\ref{sec:kmeans-kmedian-main-body}.

\begin{proof}[Proof of Corollary~\ref{cor:coarse-centroid-set}]
We claim that the set of points $\cC$ guaranteed by Lemma~\ref{lem:good-center} forms the desired centroid set. To prove this, let us fix an optimal solution $c^*_1, \dots, c^*_k$ of $(k, p)$-Clustering on the input $\bX$. where ties are broken arbitrarily. For such a solution, let the map $\psi: [n] \to [k]$ be such that $c^*_{\psi(i)} \in \argmin_{j \in [k]} \|x_i - c^*_j\|$ (with ties broken arbitrarily). For every $j \in [k]$, let\footnote{We assume throughout that $n^*_j > 0$. This is without loss of generality in the case where $n \geq k$. When $n < k$, our DP algorithms can output anything, since the allowed additive errors are larger than $k$.} $n^*_j := |\psi^{-1}(j)|$ be the number of input points closest to center $c^*_j$ and let $r^*_j := \left(\frac{1}{n^*_j} \sum_{i \in \psi^{-1}(j)} \|x_i - c^*_j\|^p\right)^{1/p}$. Finally, we use $\tr_j$ to denote $\max\left\{2r^*_j, \frac{1}{n}, 2\left(\frac{4\opt}{n^*_j k}\right)^{1/p}\right\}$.

Let $T^*$ be as in Lemma~\ref{lem:good-center}. Let $J \subseteq [k]$ be the set $\{j \in [k] \mid n^*_j \geq 4T^*\}$. Due to Markov's inequality and $p \geq 1$, we have that $|\bX \cap \cB(c_j, 2 r^*_j)| \geq 0.5 n^*_j$, which is at least $2T^*$ for all $j \in J$. 

Thus, Lemma~\ref{lem:good-center} ensures that, with probability $1 - \beta$, the following holds for all $j \in J$: there exists $c'_j \in \cC$ such that $\|c'_j - c^*_j\| \leq 18 \tr_j$. Henceforth, we will assume that this event holds and show that $\cC$ must be an $\left(O_p(1), O_p\left(\frac{d k^2 \log n}{\eps} \log\left(\frac{n}{\beta}\right) + 1\right)\right)$-centroid set of $\bX$.

For convenience, we let $c'_j = \bzero$ for all $j \notin J$. From the discussion in the previous paragraph, we can derive
\begin{align}
\cost_{\bX}^p(c'_1, \dots, c'_k) &\leq \sum_{i \in [n]} \|x'_i - c'_{\psi(i)}\|^p \nonumber \\
&= \sum_{j \in [k]} \sum_{i \in \psi^{-1}(j)} \|x'_i - c'_j\|^p \nonumber \\
&= \sum_{j \in J} \sum_{i \in \psi^{-1}(j)} \|x'_i - c'_j\|^p + \sum_{j \in J \setminus [k]} \sum_{i \in \psi^{-1}(j)} \|x'_i - c'_j\|^p \nonumber \\
&\leq \sum_{j \in J} \sum_{i \in \psi^{-1}(j)} (\|x'_i - c^*_j\| + \|c^*_j - c'_j\|)^p + \sum_{j \in J \setminus [k]} \sum_{i \in \psi^{-1}(j)} 1 \nonumber \\
&\leq \sum_{j \in J} \sum_{i \in \psi^{-1}(j)} \left(2^p\|x'_i - c^*_j\|^p + 2^p\|c^*_j - c'_j\|^p\right) + \sum_{j \in J \setminus [k]} 4T^* \nonumber \\
&\leq \sum_{j \in J} \sum_{i \in \psi^{-1}(j)} \left(2^p\|x'_i - c^*_j\|^p + 2^p\|c^*_j - c'_j\|^p\right) + 4kT^* \nonumber \\
&\leq 2^p \cdot \opt + 2^p \left(\sum_{j \in J} n^*_j \|c^*_j - c'_j\|^p\right) + O\left(\frac{d k^2 \log n}{\eps} \log\left(\frac{n}{\beta}\right)\right). \label{eq:coarse-centroid-set-intermediate}
\end{align}
Now, since $\|c'_j - c^*_j\| \leq 18\tr_j$, we have
\begin{align*}
\sum_{j \in J} n^*_j \|c^*_j - c'_j\|^p &\leq \sum_{j \in J} n^*_j \left(18\tr_j\right)^p \\
&= 18^p \sum_{j \in J} n^*_j \left(\max\left\{2r^*_j, \frac{1}{n}, 2\left(\frac{4\opt}{n^*_j k}\right)^{1/p}\right\}\right)^p \\
&\leq 18^p \sum_{j \in J} n^*_j \left(\left(2r^*_j\right)^p + \left(\frac{1}{n}\right)^p + \frac{4\opt}{n^*_j k} \right) \\
&\leq 36^p \opt + 18^p + 4 \cdot 18^p \cdot \opt.
\end{align*}
Plugging this back into~\ref{eq:coarse-centroid-set-intermediate}, we have
\begin{align*}
\cost_{\bX}^p(c'_1, \dots, c'_k) \leq O_p(1) \cdot \opt + O_p\left(\frac{d k^2 \log n}{\eps} \log\left(\frac{n}{\beta}\right) + 1\right),
\end{align*}
which concludes our proof.
\end{proof}

We will now turn our attention back to the proof of Lemma~\ref{lem:good-center}.

\begin{proof}[Proof of Lemma~\ref{lem:good-center}]
We claim that Algorithm~\ref{alg:important-cluster}, where \densestball on Line~\ref{step:1-cluster} is the $\left(\frac{\eps}{2k\lceil \log n \rceil}\right)$-DP algorithm from Theorem~\ref{thm:pure-1-cluster-given-radius} (with $\alpha = 1$), satisfies the properties. It is clear that the runtime of the algorithm is as claimed. We will next argue the privacy and security guarantees of our algorithm.

\paragraph{Privacy.} We will now argue that the algorithm is $\eps$-DP. To do so, consider any pair of datasets $\bX, \bX'$ and any possible output $\btc = (\tc_{i, j})_{i \in [\lceil \log n \rceil], j \in [2k]}$. Furthermore, let $\cM$ be the shorthand for our algorithm $\textsc{CoarseCandidates}$, and for every $(i, j) \in [\lceil \log n \rceil] \times [2k]$, let $R_{<(i, j)} = \{(i', j') \in [\lceil \log n \rceil] \times [2k] \mid i' < i \text{ or } i' = i, j' < j\}$. We have
\begin{align} \label{eq:good-center-dp-intermediate}
&\frac{\Pr[\cM(\bX) = \bc]}{\Pr[\cM(\bX') = \bc]} \\
&= \Pi_{(i, j) \in [\lceil \log n \rceil] \times [2k]} \frac{{\Pr\left[\cM(\bX)_{(i, j)} = \tc_{i, j} \mid \forall (i', j') \in R_{<(i, j)} \cM(\bX)_{(i', j')} = \tc_{i', j'} \right]}}{\Pr\left[\cM(\bX')_{(i, j)} = \tc_{i, j} \mid \forall (i', j') \in R_{<(i, j)} \cM(\bX')_{(i', j')} = \tc_{i', j'} \right]}.
\end{align}
Now note that when $\cM(\bX)_{(i', j')} = \cM(\bX')_{(i', j')}$ for all $(i', j') < R_{< (i, j)}$, the sets $\bX_{uncovered}$ at step $(i, j)$ of the two runs are neighboring datasets. Thus, the $\left(\frac{\eps}{2 k \lceil\log n\rceil}\right)$-DP guarantee of the call to \densestball on line~\ref{step:1-cluster} implies that $$\frac{{\Pr\left[\cM(\bX)_{(i, j)} = \tc_{i, j} \mid \forall (i', j') \in R_{<(i, j)} \cM(\bX)_{(i', j')} = \tc_{i', j'} \right]}}{\Pr\left[\cM(\bX')_{(i, j)} = \tc_{i, j} \mid \forall (i', j') \in R_{<(i, j)} \cM(\bX')_{(i', j')} = \tc_{i', j'} \right]} \leq e^{\frac{\eps}{2 k \lceil\log n\rceil}}.$$
Plugging this back into~\eqref{eq:good-center-dp-intermediate}, we get
\begin{align*}
\frac{\Pr[\cM(\bX) = \bc]}{\Pr[\cM(\bX') = \bc]} \leq \left(e^{\frac{\eps}{2 k \lceil\log n\rceil}}\right)^{2 k \lceil\log n\rceil} = e^{\eps},
\end{align*}
which means that our algorithm is $\eps$-DP as desired.

\paragraph{Accuracy.} The rest of this proof is devoted to proving the accuracy guarantee of Algorithm~\ref{alg:important-cluster}. To do so, we first note that the accuracy guarantee in Theorem~\ref{thm:pure-1-cluster-given-radius} implies that each call to the \densestball algorithm in line~\ref{step:1-cluster} solves the \densestball problem with approximation ratio $2$ and additive error $T^*$, with probability at least $1 - \frac{\beta}{2 k \lceil \log n \rceil}$. By a union bound, this holds for \emph{all} calls to \densestball with probability at least $1 - \beta$. Henceforth, we assume that this event, which we denote by $E_{\densestball}$ for brevity, occurs.

 Now, let us fix $c \in \R^d$ and $r \in [1/n, 1]$ such that $n_{c, r} := |\bX \cap \cB(c, r)|$ is at least $2T^*$. We will next argue that, with probability at least $1 - \beta$, there exists $c' \in \cC$ such that $\|c - c'\| \leq 18 \cdot \max\left\{r, \left(\frac{2\opt}{n_{c, r} k}\right)^{1/p}\right\}$. We will prove this by contradiction. 

Suppose for the sake of contradiction that for all $c' \in \cC$, we have $\|c - c'\| > 18 \cdot \max\left\{r, \left(\frac{2\opt}{n_{c, r} k}\right)^{1/p}\right\}$. Let $\ti = \left\lceil \log\left(n \cdot \max\left\{r, \left(\frac{2\opt}{n_{c, r} k}\right)^{1/p}\right\}\right) \right\rceil$ and $\tr = 2^{\ti} / n$. Our assumption implies that 
 \begin{align} \label{eq:selected-center-far}
 \|c - c'\| \geq 9\tilde{r}
 \end{align}
 for all $c' \in \cC$.

 Now, let us consider the centers selected on line~\ref{step:1-cluster} when $i = \ti$; let these centers be $c'_1, \dots, c'_{2k}$. Using~\eqref{eq:selected-center-far} and the fact that $\tr \geq r$, we get that all the $n_{c, r}$ points in $\bX \cap \cB(c, r)$ still remain in $\bX_{uncovered}$. As a result, from our assumption that $E_{\densestball}$ occurs, when $c'_j$ is selected (in line~\ref{step:1-cluster}) we must have that
 \begin{align} \label{eq:selected-ball-intersection-size}
 |\cB(c'_j, 2\tr) \cap \bX_{uncovered}| \geq n_{c, r} - T^* \geq 0.5n_{c, r},
 \end{align}
 for all $j \in [2k]$. Note that this also implies that
 \begin{align} \label{eq:found-centered-far}
 \|c'_j - c'_{j'}\| > 6 \tr,
 \end{align}
 for $j < j'$; otherwise, $\cB(c'_{j'}, 2\tr)$ would have been completely contained in $\cB(c'_j, 8\tr)$ and line~\ref{eq:ball-removal} would have already removed all elements of $\cB(c'_{j'}, 2\tr)$ from $\bX_{uncovered}$.

Now, consider any optimal solution $C^* = \{c^*_1, \dots, c^*_k\}$ to the $(k, p)$-Clustering problem with cost $\opt$. Notice that~\eqref{eq:found-centered-far} implies that the balls $\cB\left(c'_1, 3 \tr \right), \dots, \cB\left(c'_{2k}, 3 \tr \right)$ are disjoint. As a result, there must be (at least) $k$ selected centers $c'_{j_1}, \dots, c'_{j_k}$ such that $\cB\left(c'_{j_1}, 3\tr \right), \dots, \cB\left(c'_{j_k}, 3\tr \right)$ do not contain any optimal centers from $C^*$. This implies that every point in $\cB\left(c'_{j_1}, 2\tr \right), \dots, \cB\left(c'_{j_k}, 2\tr \right)$ is at distance more than $\tr$ from any centers in $C^*$. Furthermore, from~\eqref{eq:found-centered-far} and~\eqref{eq:selected-ball-intersection-size}, the balls $\cB\left(c'_{j_1}, 2\tr \right), \dots, \cB\left(c'_{j_k}, 2\tr \right)$ are all pairwise disjoint and each contains at least $0.5n_{c, r}$ points. This means that
\begin{align*}
\cost^p_{\bX}(c^*_1, \dots, c^*_k) &> k \cdot 0.5 \cdot n_{c, r} \cdot \tr^p \\
&\geq k \cdot 0.5 \cdot n_{c, r} \left(\left(\frac{2\opt}{n_{c, r} k}\right)^{1/p}\right)^p \tag{from our choice of $\tr$} \\
&= \opt.
\end{align*}
This contradicts our assumption that $\cost^p_{\bX}(c^*_1, \dots, c^*_k) = \opt$.

 As a result, the accuracy guarantee holds conditioned on $E_{\densestball}$. Since we argued earlier that $\Pr[E_{\densestball}] \geq 1 - \beta$, we have completed our proof.
\end{proof}

\subsection{Centroid Set Refinement via Exponential Covers}
\label{sec:centroid-refinement}

As stated earlier, we will now follow the approach of~\cite{feldman2009private}, which is in turn based on a (non-private) coreset construction of~\cite{HarPeledM04}. Specifically, we refine our centroid set by placing exponential covers over each of the point in the coarse centroid set from Section~\ref{sec:coarse-centroid-set}. This is described formally in Algorithm~\ref{alg:refined-candidate} below. We note that~\cite{HarPeledM04} orginally uses \emph{exponential grids}, where covers are replaced by grids; this does not work for us because grids will lead to an additive error bound of $O(d)^d$ (instead of $O(1)^d$ for covers) which is super-polynomial for our regime of parameter $d = O(\log k)$. We also remark that exponential covers are implicitly taken in~\cite{feldman2009private} where the authors take equally space lines through each center and place points at exponentially increasing distance on each such line.

%As discussed in Section~\ref{subsec:overview_proofs}, this is crucial for us to achieve an additive error that is polynomial in $k$.

\begin{algorithm}[h!]
\caption{Centroid Set Refinement.}\label{alg:refined-candidate}
\begin{algorithmic}[1]
\Procedure{RefinedCentroidSet$^{\eps}(x_1, \dots, x_n; \zeta)$}{}
\State $\cC \leftarrow$ \textsc{CoarseCentroidSet}$^{\eps}(x_1, \dots, x_n)$
\State $\cC' \leftarrow \{\bzero\}$
\For{$c \in \cC$}
\For{$i \in \{1, \dots, \lceil\log n\rceil\}$}
\State $r \leftarrow 2^i / n$
\State $C_{r, j} \leftarrow$ $(\zeta r)$-cover of the ball $\cB(c, 40r)$ \label{step:cover-center}
\State $\cC' \leftarrow \cC' \cup C_{r, j}$
\EndFor
\EndFor
\Return $\cC'$
\EndProcedure
\end{algorithmic}
\end{algorithm}

At this point, we take two separate paths. First, in Section~\ref{sec:non-private-ratio}, we will continue following the approach of~\cite{feldman2009private} and eventually prove Theorem~\ref{thm:main-apx-non-private-to-private}. In the second path, we use a different approach to prove Theorem~\ref{thm:fpt-apx} in Section~\ref{sec:discrete-algo}.

While the \textsc{RefinedCandidates} algorithm will be used in both paths%Sections~\ref{sec:non-private-ratio} and~\ref{sec:discrete-algo}
, the needed guarantees are different, and thus we will state them separately in each subsequent section. %Nonetheless, a common bound needed for both sections is that on the size of $\cC'$, which can be stated as follows.

\subsection{Approximation Algorithm I: Achieving Non-Private Approximation Ratio via Private Coresets}
\label{sec:non-private-ratio}

This section is devoted to the proof of Theorem~\ref{thm:main-apx-non-private-to-private}. The bulk of the proof is in providing a good private coreset for the problem, which is done in Section~\ref{sec:private-coreset}. As stated earlier, this part closely follows Feldman et al.~\cite{feldman2009private}, except that our proof is more general in that it works for every $p \geq 1$ and that we give a full analysis for all dimension $d$. Once the private coreset is constructed, we may simply run the non-private approximation algorithm on the coreset to get the desired result; this is formalized in Section~\ref{sec:apx-from-coreset}.

\subsubsection{Private Coreset Construction}
\label{sec:private-coreset}

We first show that we can construct a private coreset efficiently when the dimension $d$ is small:

\begin{lemma} \label{lem:coreset-main}
For every $\eps > 0$, $p \geq 1$ and $0 < \alpha < 1$, there is an $2^{O_{\alpha, p}(d)}\poly(n)$-time $\eps$-DP algorithm %that takes $\bX = (x_1, \dots, x_n) \subseteq (\R^d)^n$ and $k \in \N$ as inputs, and outputs a multiset $\bX'$ such 
that, with probability $1 - \beta$, outputs an $\left(\alpha, O_{p, \alpha}\left(\frac{k^2 \log^2 n \cdot 2^{O_{p, \alpha}(d)}}{\eps} \log\left(\frac{n}{\beta}\right) + 1\right)\right)$-coreset for $(k, p)$-Clustering, for every $\beta \in (0, 1)$. 
\end{lemma}

Notice that Lemma~\ref{lem:coreset-main} implies Lemma~\ref{lem:coreset-main-body} in Section~\ref{sec:kmeans-kmedian-main-body}. The algorithm is presented below in Algorithm~\ref{alg:coreset}; here $\zeta$ is a parameter to be specified in the proof of Lemma~\ref{lem:coreset-main}.

\begin{algorithm}[h!]
\caption{Private Coreset Construction.}\label{alg:coreset}
\begin{algorithmic}[1]
\Procedure{PrivateCoreset$^{\eps}(x_1, \dots, x_n; \zeta)$}{}
\State $\cC' \leftarrow$ \textsc{RefinedCentroidSet}$^{\eps/2}(x_1, \dots, x_n; \zeta)$.
\For{$c \in \cC'$}
\State $count[c] = 0$
\EndFor
\For{$i \in [n]$}
\State $x'_i \leftarrow$ closest point in $\cC'$ to $x_i$
\State $count[x'_i] \leftarrow count[x'_i] + 1$
\EndFor
\State $\bX' \leftarrow \emptyset$
\For{$c \in \cC'$}
\State $\widetilde{count}[c] \leftarrow  count[c] + \DLap(2/\eps)$
\State Add $\max\{\widetilde{count}[c], 0\}$ copies of $c$ to $\bX'$
\EndFor
\State \Return $\bX'$
\EndProcedure
\end{algorithmic}
\end{algorithm}

To prove Lemma~\ref{lem:coreset-main}, we will use the following simple fact:

\begin{fact} \label{fact:power-sum-ineq}
For any $p \geq 1$ and $\gamma > 0$, define $\lambda_{p, \gamma} := \left(\frac{1 + \gamma}{((1 + \gamma)^{1/p} - 1)^p}\right)$. Then, for all $a, b \geq 0$, we have
\begin{align*}
\left(a + b\right)^p \leq (1 + \gamma) a^p + \lambda_{p, \gamma} \cdot b^p.
\end{align*}
\end{fact}
\begin{proof}
It is obvious to see that the inequality holds when $a = 0$ or $b = 0$. Hence, we may assume that $a, b > 0$. Now, consider two cases, based on whether $b \leq \left((1 + \gamma)^{1/p} - 1\right)a$.

If $b \leq \left((1 + \gamma)^{1/p} - 1\right)a$, we have $(a + b)^p \leq ((1 + \gamma)^{1/p} a)^p = (1 + \gamma)a^p$.

On the other hand, if $b > \left((1 + \gamma)^{1/p} - 1\right)a$, we have $a < \frac{b}{(1 + \gamma)^{1/p} - 1}$. This implies that
\begin{align*}
(a + b)^p \leq \left(\frac{(1 + \gamma)^{1/p}}{(1 + \gamma)^{1/p} - 1} \cdot b\right)^p = \lambda_{p, \gamma} \cdot b^p. &\qedhere
\end{align*}
\end{proof}

We run Algorithm~\ref{alg:refined-candidate} with $\zeta = 0.01 \cdot \left(\frac{\alpha}{10 \lambda_{p, \alpha/2}}\right)^{1/p}$. It is obvious that the algorithm is $\eps$-DP. Furthermore, the running time of the algorithm is polynomial in $n, k$ and the size of the cover used in Line~\ref{step:cover-center} of Algorithm~\ref{alg:refined-candidate}. We can pick such a cover so that the size\footnote{This holds for any $(\zeta r)$-cover that is also a $\Omega(\zeta r)$-packing. For example, covers described in Section~\ref{sec:cover} satisfy this property.} is $O(1/\zeta)^d = 2^{O_{\alpha, p}(d)}$ as desired. Thus, we are only left to prove that $\bX'$ is (with high probability) a good coreset of $\bX$.

To prove this, let $\bX_{snapped}$ denote the multiset of points that contain $count[c]$ copies of every $c \in \cC$. (In other words, for every input point $x_i \in \bX$, we add its closest point $c_i$ from $\cC$ to $\bX_{snapped}$.) The correctness proof of Lemma~\ref{lem:coreset-main} is then divided into two parts. First, we will show that $\bX_{snapped}$ is a good coreset of $\bX$:

\begin{lemma} \label{lem:coreset-snapped}
For every $\beta > 0$, with probability $1 - \frac{\beta}{2}$, $\bX_{snapped}$ is an $\left(\alpha, O_{p, \alpha}\left(\frac{dk^2 \log n}{\eps} \cdot \log\left(\frac{n}{\beta}\right) + 1\right)\right)$-coreset of $\bX$. 
\end{lemma}

Then, we show that the final set $\bX'$ is a good coreset of $\bX$.

\begin{lemma} \label{lem:coreset-noisy}
For every $\beta > 0$, with probability $1 - \frac{\beta}{2}$, $\bX'$ is a $\left(0, O\left(\frac{(k\log^2 n) \cdot 2^{O_{p, \alpha}(d)}}{\eps} \cdot \log\left(\frac{n}{\beta}\right)\right)\right)$-coreset of $\bX_{snapped}$.
\end{lemma}

It is simple to see that Lemma~\ref{lem:coreset-main} is an immediate consequence of Lemmas~\ref{lem:coreset-snapped} and~\ref{lem:coreset-noisy}. Hence, we are left to prove these two lemmas.

\paragraph{Snapped Points are a Coreset: Proof of Lemma~\ref{lem:coreset-snapped}.}

The proof of Lemma~\ref{lem:coreset-snapped} share some similar components as that in Corollary~\ref{cor:coarse-centroid-set}, but the $(\zeta r)$-covers employed in Algorithm~\ref{alg:refined-candidate} allow one to get a sharped bound, leading to the better ratio.

\begin{proof}[Proof of Lemma~\ref{lem:coreset-snapped}]
Let us fix an optimal solution $c^*_1, \dots, c^*_k$ of $(k, p)$-Clustering on the input $\bX$. where ties are broken arbitrarily. For such a solution, let the map $\psi: [n] \to [k]$ be such that $c^*_{\psi(i)} \in \argmin_{j \in [k]} \|x_i - c^*_j\|$ (with ties broken arbitrarily). For every $j \in [k]$, let $n^*_j := |\psi^{-1}(j)|$ be the number of input points closest to center $c^*_j$ and let $r^*_j := \left(\frac{1}{n^*_j} \sum_{i \in \psi^{-1}(j)} \|x_i - c^*_j\|^p\right)^{1/p}$. Finally, we let $\tr_j$ to denote $\max\left\{2r^*_j, \frac{1}{n}, 2\left(\frac{4\opt}{n^*_j k}\right)^{1/p}\right\}$.

Let $T^*$ be as in Lemma~\ref{lem:good-center}, but with failure probability $\beta/2$ instead of $\beta$. Let $J \subseteq [k]$ be the set $\{j \in [k] \mid n^*_j \geq 4T^*\}$. Due to Markov's inequality and $p \geq 1$, we have that $|\bX \cap \cB(c_j, 2 r^*_j)| \geq 0.5 n^*_j$, which is at least $2T^*$ for all $j \in J$. 

Thus, Lemma~\ref{lem:good-center} ensures that, with probability $1 - \beta/2$, the following holds for all $j \in J$: there exists $c'_j \in \cC$ such that $\|c'_j - c^*_j\| \leq 18 \tr_j$. Henceforth, we will assume that this event holds and show that $\bX_{snapped}$ must be an $\left(\alpha, O_{p, \alpha}\left(\frac{dk^2 \log n}{\eps} \log\left(\frac{n}{\beta}\right)\right)\right)$-coreset of $\bX$.

Consider any input point $i \in \psi^{-1}(J)$. Let $\hr_i = \|x_i - c^*_{\psi(i)}\| + 18\tr_{\psi(i)}$. From the previous paragraph, we have $\|x_i - c'_{\psi(i)}\| \leq \hr_i$. Hence, from Line~\ref{step:cover-center} of Algorithm~\ref{alg:refined-candidate},
\begin{align} \label{eq:close-snapped-point}
\|x_i - x'_i\| \leq 2\zeta \hr_i.
\end{align}

Now, consider any $c_1, \dots, c_k \in \R^d$. We have
\begin{align}
\cost^p_{\bX_{snapped}}(c_1, \dots, c_k) 
&= \sum_{i \in [n]} \left(\min_{j' \in [k]} \|x'_i - c_{j'}\|\right)^p \nonumber \\
&\leq \sum_{i \in [n]} \left(\left(\min_{j' \in [k]} \|x_i - c_{j'}\|\right) + \|x_i - x'_i\|\right)^p \nonumber \\
&\leq \sum_{i \in [n]} \left((1 + \alpha/2) \cdot \left(\min_{j' \in [k]} \|x_i - c_{j'}\|\right)^p + \lambda_{p, \alpha/2} \cdot \|x_i - x'_i\|^p\right) \nonumber \tag{by Fact~\ref{fact:power-sum-ineq}}\\
&= (1 + \alpha/2) \cdot \cost^p_{\bX}(c_1, \dots, c_k) + \lambda_{p, \alpha/2} \cdot \sum_{i \in [n]} \|x_i - x'_i\|^p. \label{eq:separate-cost-snapped}
\end{align}

Now, we can separate the term $\sum_{i \in [n]} \|x_i - x'_i\|^p$ as follows.
\begin{align}
\sum_{i \in [n]} \|x_i - x'_i\|^p &= \sum_{j \in k} \sum_{i \in \psi^{-1}(j)} \|x_i - x'_i\|^p \nonumber \\
&= \sum_{j \in J} \sum_{i \in \psi^{-1}(j)} \|x_i - x'_i\|^p + \sum_{j \notin J} \sum_{i \in \psi^{-1}(j)} \|x_i - x'_i\|^p \nonumber \\
&\overset{\eqref{eq:close-snapped-point}}{\leq} \sum_{j \in J} \sum_{i \in \psi^{-1}(j)} \left(2\zeta \hr_i\right)^p + \sum_{j \in [k] \setminus J} \sum_{i \in \psi^{-1}(j)} 1 \nonumber \\
&\leq (2\zeta)^p \cdot \left(\sum_{j \in J} \sum_{i \in \psi^{-1}(j)} \hr_i^p\right) + k \cdot 4T^* \nonumber \\
&= (2\zeta)^p \cdot \left(\sum_{j \in J} \sum_{i \in \psi^{-1}(j)} \hr_i^p\right) + O\left(\frac{dk^2 \log n}{\eps} \log\left(\frac{n}{\beta}\right)\right),
\label{eq:separate-large-and-small-clusters-snapped}
\end{align}
where in the last inequality we recall from the definition that $|\psi^{-1}(j)| \leq 4T^*$ for all $j \notin J$.

From the definition of $\hr_i$, we can now bound the term $\sum_{j \in J} \sum_{i \in \psi^{-1}(j)} \hr_i^p$ by
\begin{align}
\sum_{j \in J} \sum_{i \in \psi^{-1}(j)} \hr_i^p &= \sum_{j \in J} \sum_{i \in \psi^{-1}(j)} \left(\|x_i - c^*_j\| + 18\tr_j\right)^p \nonumber \\
&\leq  19^p \cdot \sum_{j \in J} \sum_{i \in \psi^{-1}(j)} \max\{\|x_i - c^*_{\psi(i)}\|, \tr_j\}^p \nonumber \\
&= 19^p \cdot \sum_{j \in J} \sum_{i \in \psi^{-1}(j)} \left(\|x_i - c^*_{\psi(i)}\|^p + \tr_j^p\right) \nonumber \\
&\leq 19^p \left(\opt + \sum_{j \in J} n^*_j \tr_j^p\right).
\label{eq:separate-sum-snapped}
\end{align}
where the first inequality follows from the fact that $(a + b)^p \leq (2a)^p + (2b)^p$.

From the definition of $\tr_j$, we may now bound the term $\sum_{j \in J} n^*_j \tr_j^p$ by
\begin{align}
\sum_{j \in J} n^*_j \tr_j^p
&= \sum_{j \in J} n^*_j \cdot \max\left\{2r^*_j, \frac{1}{n}, 2\left(\frac{4\opt}{n^*_j k}\right)^{1/p}\right\}^p \nonumber \\
&= 2^p \sum_{j \in J} n^*_j \cdot \left((r^*_j)^p + \frac{1}{n} + \frac{4\opt}{n^*_j k}\right) \nonumber \\
&\leq 2^p \left(\opt + 1 + 4\opt\right) \nonumber \\
&= 5 \cdot 2^p \cdot \opt + O_p(1). \label{eq:large-cluster-bound-additive}
\end{align}

Plugging~\eqref{eq:separate-large-and-small-clusters-snapped},~\eqref{eq:separate-sum-snapped}, and~\eqref{eq:large-cluster-bound-additive} back into~\eqref{eq:separate-cost-snapped}, we get
\begin{align*}
&\cost^p_{\bX_{snapped}}(c_1, \dots, c_k) \\
&\leq (1 + \alpha/2) \cdot \cost^p_{\bX}(c_1, \dots, c_k) + \lambda_{p, \alpha/2} \cdot (100 \zeta)^p \opt + O_{p, \alpha}\left(\frac{dk^2 \log n}{\eps} \log\left(\frac{n}{\beta}\right) + 1\right) \\
&\leq (1 + \alpha/2) \cdot \cost^p_{\bX}(c_1, \dots, c_k) + (\alpha/2) \cdot \opt + O_{p, \alpha}\left(\frac{dk^2 \log n}{\eps} \log\left(\frac{n}{\beta}\right) + 1\right) \\
&\leq (1 + \alpha) \cdot \cost^p_{\bX}(c_1, \dots, c_k) + O_{p, \alpha}\left(\frac{dk^2 \log n}{\eps} \log\left(\frac{n}{\beta}\right) + 1\right),
\end{align*}
where the second inequality follows from our choice of $\zeta$.

Using an analogous argument, we get that
\begin{align*}
\cost_{\bX}(c_1, \dots, c_k) \leq (1 + \alpha) \cdot \cost^p_{\bX_{snapped}}(c_1, \dots, c_k) + O_{p, \alpha}\left(\frac{dk^2 \log n}{\eps} \log\left(\frac{n}{\beta}\right) + 1\right).
\end{align*}
Dividing both sides by $1 + \alpha$ yields
\begin{align*}
(1 - \alpha) \cdot \cost^p_{\bX}(c_1, \dots, c_k) \leq \cost^p_{\bX_{snapped}}(c_1, \dots, c_k) + O_{p, \alpha}\left(\frac{dk^2 \log n}{\eps} \log\left(\frac{n}{\beta}\right) + 1\right).
\end{align*}
Thus, $\bX_{snapped}$ is a $\left(1 + \alpha, O_{p, \alpha}\left(\frac{dk^2 \log n}{\eps} \log\left(\frac{n}{\beta}\right)\right) + 1\right)$-coreset of $\bX$ as desired.
\end{proof}

\paragraph{Handling Noisy Counts: Proof of Lemma~\ref{lem:coreset-noisy}.} We next give a straightforward proof of Lemma~\ref{lem:coreset-noisy}. Similar statements were shown before in~\cite{feldman2009private,Stemmer20}; we include the proof here for completeness.

\begin{proof}[Proof of Lemma~\ref{lem:coreset-noisy}]
For each $c \in \cC'$, recall that $|\widetilde{count}[c] - count[c]|$ is just distributed as the absolute value of the discrete Laplace distribution with parameter $2/\eps$. It is simple to see that, with probability $0.5 \beta / |\cC'|$, we have $|\widetilde{count}[c] - count[c]| \leq \frac{\log(2|\cC'|/\beta)}{\eps}$. As a result, by a union bound, we get that $\sum_{c \in \cC'} |\widetilde{count}[c] - count[c]| \leq |\cC'| \cdot \frac{\log(|\cC'|/\beta)}{\eps}$ with probability at least $1 - \beta/2$.

Finally, we observe that for any centers $c_1, \dots, c_k \in \R^d$, it holds that
\begin{align*}
&|\cost^p_{\bX_{snapped}}(c_1, \dots, c_k) - \cost^p_{\bX'}(c_1, \dots, c_k)| \\
&\leq \sum_{c \in \cC'} \left|\max\{\widetilde{count}[c], 0\} - count[c]\right| \cdot \left(\min_{i \in [k]} \|c_i - c\|\right). \\
&\leq \sum_{c \in \cC'} \left|\widetilde{count}[c] - count[c]\right| \\
&\leq |\cC'| \cdot \frac{\log(|\cC'|/\beta)}{\eps}.
\end{align*} 

Finally, recall that $|\cC'| \leq |\cC| \cdot \lceil \log n \rceil \cdot O(1 / \zeta)^d = O\left(k\log^2 n \cdot 2^{O_{p, \alpha}(d)}\right)$. Plugging this to the above yields the desired bound.
\end{proof}

\subsubsection{From Coreset to Approximation Algorithm}
\label{sec:apx-from-coreset}

Finally, we give our DP approximation algorithm. This is extremely simple: first find a private coreset using Algorithm~\ref{alg:coreset} and then run a (possibly non-private) approximation algorithm on this coreset.

\begin{algorithm}[h!]
\caption{Algorithm for $(k, p)$-Clustering in Low Dimension.}\label{alg:apx-main}
\begin{algorithmic}[1]
\Procedure{ClusteringLowDimension$^{\eps}(x_1, \dots, x_n, k; \zeta)$}{}
\State $\bX' \leftarrow$  \textsc{PrivateCoreset}$^{\eps}(x_1, \dots, x_n; \zeta)$
\State \Return $\textsc{NonPrivateApproximation}(\bX', k)$
\EndProcedure
\end{algorithmic}
\end{algorithm}

As alluded to earlier, the above algorithm can give us an approximation ratio that is arbritrarily close to that of the non-private approximation algorithm, while the error remains small (when the dimension is small). This is formalized below.

\begin{proof}[Proof of Theorem~\ref{thm:main-apx-non-private-to-private}]
We run Algorithm~\ref{alg:apx-main} with $\zeta$ being the same as in the proof of Lemma~\ref{lem:coreset-main}, except that with approximation guarantee $0.1\alpha$ instead of $\alpha$, and \textsc{NonPrivateApproximation} being the (not necessarily DP) $w$-approximation algorithm. The privacy and running time of the algorithm follow from Lemma~\ref{lem:coreset-main}. We will now argue its approximation guarantee.

By Lemma~\ref{lem:coreset-main}, with probability at least $1 - \beta$, $\bX'$ is a $\left(0.1\alpha, t\right)$-coreset of $\bX$, where $t = O_{p, \alpha}\left(\frac{k^2 \log^2 n \cdot 2^{O_{p, \alpha}(d)}}{\eps} \log\left(\frac{n}{\beta}\right) + 1\right)$. Let $c_1^*, \dots, c_k^*$ be the optimal solution of $\bX$. Since \textsc{NonPrivateApproximation} is a $w$-approximation algorithm, it must return a set $c_1, \dots, c_k$ of centers such that
\begin{align}
\cost_{\bX'}(c_1, \dots, c_k) &\leq w \cdot \opt^{p, k}_{\bX'} \nonumber \\
&\leq w \cdot \cost^p_{\bX'}(c^*_1, \dots, c^*_k) \nonumber \\
&\leq w(1 + 0.1\alpha) \cdot \cost^p_{\bX}(c_1^*, \dots, c_k^*) + wt \nonumber \tag{since $\bX'$ is a $(0.1\alpha, t)$-coreset of $\bX$} \\
&= w(1 + 0.1\alpha) \cdot \opt^{p, k}_{\bX} + wt. \label{eq:apx-non-private-tmp}
\end{align}
Using once again the fact that $\bX'$ is a $(0.1\alpha, t)$-coreset of $\bX$, we get
\begin{align*}
\cost^p_{\bX}(c_1, \dots, c_k) &\leq \frac{1}{1 - 0.1\alpha} \cdot \left(\cost^p_{\bX'}(c_1, \dots, c_k) + t\right) \\
&\overset{\eqref{eq:apx-non-private-tmp}}{\leq} \frac{1}{1 - 0.1\alpha} \cdot \left(w(1 + 0.1\alpha) \cdot \opt^{p, k}_{\bX} + wt + t\right) \\
&\leq w(1 + \alpha) \opt_{\bX}^{p, k} + O_{w}(t),
\end{align*}
which completes our proof.
\end{proof}

\subsection{Approximation Algorithms II:  Private Discrete $(k, p)$-Clustering Algorithm}
\label{sec:discrete-algo}

In this section, we show how to reduce the additive error in some cases, by using a DP  algorithm for \emph{Discrete} $(k, p)$-Clustering. Recall the definition of discrete $(k, p)$-Clustering from Section~\ref{sec:prelim}: in addition to $\bX = (x_1, \dots, x_n) \in (\R^d)^n$ and $k \in \N$, we are also given a set $\cC \subseteq \R^d$ and the goal is to find $c_1, \dots, c_k \in \cC$ that minimizes $\cost^p_{\bX}(c_1, \dots, c_k)$.

The overview is very simple: we will first show (in Section~\ref{sec:centroid-set-guarantee}) that \textsc{RefinedCentroidSet} can produce a centroid set with an approximation ratio arbitrarily close to one. Then, we explain in Section~\ref{sec:generic-discrete} that by running the natural Exponential Mechanism for Discrete $(k, p)$-Clustering with the candidate set being the output from \textsc{RefinedCentroidSet}, we arrive at a solution for $(k, p)$-Clustering with an approximation ratio arbitrarily close to one, thereby proving Theorem~\ref{thm:fpt-apx}.

We remark that previous works~\cite{BalcanDLMZ17,StemmerK18,Stemmer20} also take the approach of producing a centroid set and then run DP approximation for Discrete $(k, p)$-Clustering from~\cite{GuptaLMRT10}. However, the centroid sets produced in previous works do not achieve ratio arbitrarily close to one and thus cannot be used to derive such a result as our Theorem~\ref{thm:fpt-apx}.

\subsubsection{Centroid Set Guarantee of \textsc{RefinedCentroidSet}}
\label{sec:centroid-set-guarantee}

The centroid set guarantee for the candidates output by \textsc{RefinedCentroidSet} is stated below. The crucial point is that the approximation ratio can be $1 + \alpha$ for any $\alpha > 0$.

\begin{lemma} \label{lem:refined-cluster-approximation}
For every $\eps > 0, p \geq 1$ and $0 < \alpha \leq 1$, there is an $2^{O_{\alpha, p}(d)}\poly(n)$-time $\eps$-DP algorithm that, with probability $1 - \beta$, outputs an $\left(1 + \alpha, O_{\alpha, p}\left(\frac{d k^2 \log n}{\eps} \log\left(\frac{n}{\beta}\right) + 1\right)\right)$-centroid set for $(k, p)$-Clustering of size $O\left(k \log^2 n \cdot2^{O_{\alpha, p}(d)}\right)$, for every $\beta \in (0, 1)$.
\end{lemma}

The proof of Lemma~\ref{lem:refined-cluster-approximation} below follows similar blueprint as that of Lemma~\ref{lem:coreset-snapped}.

\begin{proof}[Proof of Lemma~\ref{lem:refined-cluster-approximation}]
We simply run Algorithm~\ref{alg:refined-candidate} with $\zeta = 0.01 \cdot \left(\frac{\alpha}{10 \lambda_{p, \alpha/2}}\right)^{1/p}$ (where $\lambda_{\cdot, \cdot}$ is as defined in Fact~\ref{fact:power-sum-ineq}). It follows immediately from Lemma~\ref{lem:good-center} that the algorithm is $\eps$-DP. To bound the size of $\cC$, note that we may pick the cover on Line~\ref{step:cover-center} so that its size is $O(1/\zeta)^d = 2^{O_{\alpha, p}(d)}$. Hence, the size of the output set $\cC'$ is at most $O\left(k \log^2 n \cdot 2^{O_{\alpha, p}(d)}\right)$ as desired.

We let $c^*_1, \dots, c^*_k, \psi, n^*_1, \dots, n^*_k, r^*_1, \dots, r^*_k, \tr_1, \dots, \tr_k, T^*, J$ be defined similarly as in the proof of Lemma~\ref{lem:coreset-snapped}.

%We will next move on to prove the clustering approximation guarantee. To do so, let us fix an optimal solution $c^*_1, \dots, c^*_k$ where ties are broken arbitrarily. For such a solution, let the map $\psi: [n] \to [k]$ be such that $c^*_{\psi^*(i)} \in \argmin_{j \in [k]} \|x_i - c_j\|$ (ties broken arbitrarily). For every $j \in [k]$, let\footnote{We assume throughout that $n^*_j > 0$. This is w.l.o.g. in the case $n \geq k$. When $n < k$, our differentially private algorithms can output anything, since additive errors allowed are more than $k$.} $n^*_j := |\psi^{-1}(j)|$ be the number of input points closest to center $c^*_j$ and let $r^*_j := \left(\frac{1}{n^*_j} \sum_{i \in \psi^{-1}(j)} \|x_i - c_j\|^p\right)^{1/p}$. Finally, we let $\tr_j$ to denote $\max\left\{2r^*_j, \frac{1}{n}, 2\left(\frac{4\opt}{n^*_j k}\right)^{1/p}\right\}$.

%Due to Markov's inequality and $p \geq 1$, we have that $|\bX \cap \cB(c_j, 2 r^*_j)| \geq 0.5 n^*_j$. Let $J \subseteq [k]$ be the set $\{j \in [k] \mid n^*_j \geq 4T^*\}$. 
Recall from the proof of Lemma~\ref{lem:coreset-snapped} that, with probability at least $1 - \beta$, the following holds for all $j \in J$: there exists $c'_j \in \cC$ such that $\|c'_j - c_j\| \leq 18 \tr_j$. We henceforth assume that this event occurs. From line~\ref{step:cover-center}, this implies that for all $j \in J$ there exists $c_j \in \cC'$ such that 
\begin{align} \label{eq:close-refined-center}
\|c_j - c^*_j\| \leq 2\zeta \tr_j.
\end{align}
For all $j \notin J$, let $c_j = \bzero$ for notational convenience.

We will now bound $\opt^{p, k}_{\bX}(\cC')$ as follows.
\begin{align}
\opt^{p, k}_{\bX}(\cC')
&\leq \cost^p_{\bX}(c_1, \dots, c_k) \nonumber \\
&= \sum_{i \in [n]} \left(\min_{j' \in [k]} \|x_i - c_{j'}\|\right)^p \nonumber \\
&= \sum_{j \in [k]} \sum_{i \in \psi^{-1}(j)} \left(\min_{j' \in [k]} \|x_i - c_{j'}\|\right)^p \nonumber \\
&\leq \sum_{j \in [k]} \sum_{i \in \psi^{-1}(j)} \|x_i - c_j\|^p \nonumber \\
&= \left(\sum_{j \in J} \sum_{i \in \psi^{-1}(j)} \|x_i - c_j\|^p\right) + \left(\sum_{j \in [k] \setminus J} \sum_{i \in \psi^{-1}(j)} \|x_i - c_j\|^p\right). \label{eq:separate-large-and-small-clusters}
\end{align}
We will bound the two terms in~\eqref{eq:separate-large-and-small-clusters} separately. First, we bound the second term. Recall that since $j \notin J$, we have that $|\psi^{-1}(j)| \leq n^*_j \leq 4T^* = O\left(\frac{d k \log n}{\eps} \log\left(\frac{n}{\eps \beta}\right)\right)$. Hence, we get
\begin{align}
\left(\sum_{j \in [k] \setminus J} \sum_{i \in \psi^{-1}(j)} \|x_i - c_j\|^p\right) 
&= \left(\sum_{j \in [k] \setminus J} \sum_{i \in \psi^{-1}(j)} \|x_i\|^p\right) \nonumber \\
&\leq k \cdot 4T^* \nonumber \\
&= O\left(\frac{d k^2 \log n}{\eps} \log\left(\frac{n}{\beta}\right)\right). \label{eq:small-cluster-bound}
\end{align}

Next, we can bound the first term in~\eqref{eq:separate-large-and-small-clusters} as follows. 
\begin{align}
\left(\sum_{j \in J} \sum_{i \in \psi^{-1}(j)} \|x_i - c_j\|^p\right)
&\leq \left(\sum_{j \in J} \sum_{i \in \psi^{-1}(j)} (\|x_i - c^*_j\| + \|c_j - c^*_j\|)^p\right) \nonumber \\
&\leq \left(\sum_{j \in J} \sum_{i \in \psi^{-1}(j)} (1 + \alpha/2) \cdot \|x_i - c^*_j\|^p + \lambda_{p, \alpha/2} \cdot \|c_j - c^*_j\|^p\right) \nonumber \tag{Fact~\ref{fact:power-sum-ineq}} \\
&\leq (1 + \alpha/2) \cdot \opt + \left(\sum_{j \in J} n^*_j \cdot \lambda_{p, \alpha/2} \cdot \|c_j - c^*_j\|^p\right) \nonumber \\
&\overset{\eqref{eq:close-refined-center}}{\leq} (1 + \alpha/2) \cdot \opt + \left(\sum_{j \in J} n^*_j \cdot \lambda_{p, \alpha/2} \cdot (2 \zeta \tr_j)^p\right) \nonumber \\ 
&= (1 + \alpha/2) \cdot \opt + \lambda_{p, \alpha/2} \cdot (2\zeta)^p \cdot \left(\sum_{j \in J} n^*_j \tr_j\right) \nonumber \\
&\leq (1 + \alpha/2) \cdot \opt + \lambda_{p, \alpha/2} \cdot (2\zeta)^p \cdot \left(5 \cdot 2^p \cdot \opt + O_p(1)\right) \nonumber \\
&\leq (1 + \alpha) \cdot \opt + O_{\alpha, p}(1), \tag{from our choice of $\zeta$} \label{eq:large-cluster-bound-expansion}
\end{align}
where the second-to-last inequality holds via a similar argument to~\eqref{eq:large-cluster-bound-additive}.
%From definition of $\tr_j$, we may now bound the term $\left(\sum_{j \in J} n^*_j \cdot \lambda_{p, \alpha/2} \cdot (2 \zeta \tr_j)^p\right)$ by
%\begin{align}
%\sum_{j \in J} n^*_j \cdot \lambda_{p, \alpha/2} \cdot (2 \zeta \tr_j)^p
%&= \sum_{j \in J} n^*_j \cdot \lambda_{p, \alpha/2} \cdot \left(2 \zeta \cdot \max\left\{2r^*_j, \frac{1}{n}, 2\left(\frac{4\opt}{n^*_j k}\right)^{1/p}\right\}\right)^p \nonumber \\
%&= \sum_{j \in J} n^*_j \cdot \lambda_{p, \alpha/2} \cdot (4 \zeta)^p \cdot \left((r^*_j)^p + \frac{1}{n} + \frac{4\opt}{n^*_j k}\right) \nonumber \\
%(\text{From our choice of } \zeta) &\leq \frac{\alpha}{10} \sum_{j \in J} n^*_j \cdot \left((r^*_j)^p + \frac{1}{n} + \frac{4\opt}{n^*_j k}\right) \nonumber \\
%&\leq \frac{\alpha}{10} \left(\opt + 1 + 4\opt\right) \nonumber \\
%&= \frac{\alpha}{2} \cdot \opt + O(1). \label{eq:large-cluster-bound-additive}
%\end{align}
Plugging~\eqref{eq:small-cluster-bound} and \eqref{eq:large-cluster-bound-expansion} back into~\eqref{eq:separate-large-and-small-clusters}, we conclude that $\cC'$ is a $\left(1 + \alpha, O_{\alpha, p}\left(\frac{d k^2 \log n}{\eps} \log\left(\frac{n}{\beta}\right) + 1\right)\right)$-centroid set of $\bX$ as desired.
\end{proof}

\subsubsection{Approximation Algorithm from Private Discrete $(k, p)$-Cluster}
\label{sec:generic-discrete}

It was observed by Gupta et al.~\cite{GuptaLMRT10}\footnote{Note that the precise theorem statement in~\cite{GuptaLMRT10} is only for \kmedian. However, the same argument applies for $(k, p)$-Clustering for any $p \geq 1$.} that the straightforward application of the Exponential Mechanism~\cite{McSherryT07} gives an algorithm with approximation ratio $1$ and additive error $O\left(\frac{k \log |\cC|}{\eps}\right)$, albeit with running time $|\cC|^k \cdot \poly(n)$:

\begin{theorem}[{\cite[Theorem 4.1]{GuptaLMRT10}}] \label{thm:exp-discrete-clustering}
For any $\eps > 0$ and $p \geq 1$, there is an $|\cC|^k \cdot \poly(n)$-time $\eps$-DP algorithm that, with probability $1 - \beta$, outputs an $\left(1, O\left(\frac{k}{\eps} \log\left(\frac{|\cC|}{\beta}\right)\right)\right)$-approximation for $(k, p)$-Clustering, for every $\beta \in (0, 1)$.
\end{theorem}

Our algorithm is simply to run the above algorithm on $(\bX, \textsc{RefinedCentroidSet}(\bX))$:

\begin{algorithm}[h!]
\caption{Approximation Algorithm for $(k, p)$-Clustering.}\label{alg:apx-generic-candidate-centers}
\begin{algorithmic}[1]
\Procedure{ApxClustering$^{\eps}(x_1, \dots, x_n; \zeta)$}{}
\State $\cC \leftarrow$ \textsc{RefinedCentroidSet}$^{\eps/2}(x_1, \dots, x_n; \zeta)$. \label{step:refine-center-fpt-apx}
\State \Return \textsc{DiscreteClusteringApprox}$^{\eps/2}(x_1, \dots, x_n, \cC, k)$ \label{step:discrete-clustering-apx}
\EndProcedure
\end{algorithmic}
\end{algorithm}

\begin{proof}[Proof of Theorem~\ref{thm:fpt-apx}]
We run Algorithm~\ref{alg:apx-generic-candidate-centers}, where $\zeta$ is as in the proof of Lemma~\ref{lem:refined-cluster-approximation} and the algorithm on Line~\ref{step:discrete-clustering-apx} is an $(\eps/2)$-DP algorithm from Theorem~\ref{thm:exp-discrete-clustering}. To see that the algorithm is $\eps$-DP, recall from Lemma~\ref{lem:refined-cluster-approximation} that the algorithm on Line~\ref{step:refine-center-fpt-apx} is $(\eps/2)$-DP. Since $\textsc{DiscreteClusteringApprox}$ is $(\eps/2)$-DP, Basic Composition (Theorem~\ref{thm:basic-composition}) implies that the entire algorithm is $\eps$-DP as desired. The bottleneck in terms of running time comes from \textsc{DiscreteClusteringApprox}. From Theorem~\ref{thm:exp-discrete-clustering}, the running time bound is
\begin{align*}
|\cC|^k \cdot \poly(n) &\leq O(k \log^2 n \cdot 2^{O_{\alpha, p}(d)})^k \cdot \poly(n) = 2^{O_{\alpha, p}(k d + k \log k)} \cdot \poly(n)
\end{align*}
where the bound on $|\cC|$ comes from Lemma~\ref{lem:refined-cluster-approximation}, and the second inequality comes from the fact that\footnote{Specifically, if $k \leq \frac{\log n}{\log \log n}$, it holds that $(\log n)^{O(k)} \leq \poly(n)$; on the other hand, if $k > \frac{\log n}{\log \log n}$, then $(\log n)^{O(k)} \leq k^{O(k)} = 2^{O(k \log k)}$.} $(k \log n)^k \leq 2^{O(k \log k)} \cdot \poly(n)$.

Finally, we argue the approximation guarantee of the algorithm. Recall from Lemma~\ref{lem:refined-cluster-approximation} that, with probability $1 - \beta/2$, $\cC$ is a $\left(1 + \alpha, O_{\alpha, p}\left(\frac{d k^2 \log n}{\eps} \log\left(\frac{n}{\beta}\right) + 1\right)\right)$-centroid set of $\bX$. Furthermore, from the approximation guarantee of Theorem~\ref{thm:exp-discrete-clustering}, \textsc{DiscreteClusteringApprox} outputs $c_1, \dots, c_k$ such that $\cost^p_{\bX}(c_1, \dots, c_k) \leq \opt_{\bX}^{p, k}(\cC) + O\left(\frac{k}{\eps} \log\left(\frac{|\cC|}{\beta}\right)\right)$. Combining these two, the following holds with probability $1 - \beta$:
\begin{align*}
&\cost^p_{\bX}(c_1, \dots, c_k) \\
&\leq \opt_{\bX}^{p, k}(\cC) + O_p\left(\frac{k}{\eps} \log\left(\frac{|\cC|}{\beta}\right)\right) \\
&\leq \left((1 + \alpha) \cdot \opt + O_{\alpha, p}\left(\frac{d k^2 \log n}{\eps} \log\left(\frac{n}{\beta}\right) + 1\right)\right) + O\left(\frac{k}{\eps} \log\left(\frac{k \log^2 n \cdot 2^{O_{\alpha, p}(d)}}{\beta}\right)\right) \\
&\leq (1 + \alpha) \cdot \opt + O_{\alpha, p}\left(\frac{d k^2 \log n}{\eps} \log\left(\frac{n}{\beta}\right) + 1\right),
\end{align*}
which completes our proof.
\end{proof}

\section{Dimensional Reduction: There and Back Again}\label{sec:dim_red_cluster_db}
In this section, we will extend our algorithm to work in high dimension. The overall idea is quite simple: we will use well-known random dimensionality reduction techniques, and use our formerly described algorithms to solve the problem in this low-dimensional space. While the \emph{centers} found in low-dimensional space may not immediately give us the information about the \emph{centers} in the high-dimensional space, it does give us an important information: the \emph{clusters}. For $(k, p)$-Clustering, these clusters mean the partition of the points into $k$ parts (each consisting of the points closest to each center). For \densestball, the cluster is simply the set of points in the desired ball. As we will elaborate below, known techniques imply that it suffices to only consider these clusters in high dimension without too much additional error. Given these clusters, we only have to find the center in high-dimension. It turns out that this is an easier task, compared to determining the partitions themselves. In fact, without privacy constraints, finding the optimal center of a given cluster is a simple convex program. Indeed, for $(k, p)$-Clustering, finding a center privately can be done using known tools in private convex optimization~\cite{ChaudhuriMS11,KiferST12,JainKT12,DuchiJW13,BassilyST14,WangYX17}. On the other hand, the case of \densestball is slightly more complicated, as applying these exisiting tools directly result in a large error; as we will see below, it turns out that we will apply another dimensional reduction one more time to overcome this issue.

We will now formalize the intuition outlined above. It will be convenient to use the following notation throughout this section: For any $\theta \geq 0$, we write $a \approx_{1 + \theta} b$ to denote $\frac{1}{1 + \theta} \leq \frac{a}{b} \leq 1 + \theta$. 

\subsection{$(k, p)$-Clustering}

We will start with $(k, p)$-Clustering. The formal statements of our results are stated below:
\begin{theorem} \label{thm:main-cluster-pure-high-dim}
For any $p \geq 1$, suppose that there exists a polynomial time (not necessarily private) $w$-approximation algorithm for $(k, p)$-Clustering.
Then, for every $0 < \eps \leq O(1)$ and $0 < \alpha, \beta \leq 1$, there exists an $\eps$-DP algorithm that runs in $(k / \beta)^{O_{p, \alpha}(1)} \poly(nd)$ time and, with probability $1 - \beta$, outputs an  $\left(w(1 + \alpha), O_{p, \alpha, w}\left(\left(\frac{kd + (k/\beta)^{O_{p, \alpha}(1)}}{\eps}\right) \cdot \poly\log\left(\frac{n}{\beta}\right)\right)\right)$-approximation $(k, p)$-Clustering.
\end{theorem}

\begin{theorem} \label{thm:main-cluster-apx-high-dim}
For any $p \geq 1$, suppose that there exists a polynomial time (not necessarily private) $w$-approximation algorithm for $(k, p)$-Clustering.
Then, for every $0 < \eps \leq O(1)$ and $0 < \delta, \alpha, \beta \leq 1$, there exists an $\eps$-DP algorithm that runs in $(k / \beta)^{O_{p, \alpha}(1)} \poly(nd)$ time and, with probability $1 - \beta$, outputs an $\left(w(1 + \alpha), O_{p, \alpha, w}\left(\left(\frac{k\sqrt{d}}{\eps} \cdot \poly\log\left(\frac{k}{\delta \beta}\right) \right) + \left(\frac{(k/\beta)^{O_{p, \alpha}(1)}}{\eps} \cdot \poly\log\left(\frac{n}{\beta}\right)\right)\right)\right)$-approximation for $(k, p)$-Clustering.
\end{theorem}

We remark that, throughout this section, we will state our results under the assumption that $\eps \leq O(1)$. In all cases, our algorithms extend to the case $\eps = \omega(1)$, but with more complicated additve error expressions; thus, we choose not state them here.

To do so, we will need the following definition of the cost of a $k$-partition, as stated below. Roughly speaking, this means that we already fix the points assigned to each of the $k$ clusters, and we can only select the center of each cluster.
\begin{definition}[Partition Cost]
Given a partition $\cX = (\bX_1, \dots, \bX_k)$ of $\bX$, its cost is defined as
\begin{align*}
\cost^p(\cX) := \sum_{i=1}^k \min_{c_i \in \R^d} \|x_i - c_j\|^p.
\end{align*}
\end{definition}

For $(k, p)$-Clustering, we need the following recent breakthrough result due to Makarychev et al.~\cite{MakarychevMR19}, which roughly stating that reducing to $O(\log k)$ dimension suffices to preserve the cost of $(k, p)$-Clustering for all paritions.

\begin{theorem}[Dimensionality Reduction for $(k, p)$-Cluster~\cite{MakarychevMR19}] \label{thm:dim-red-cluster}
For every $0 < \beta, \talpha < 1, p \geq 1$ and $k \in \N$, there exists $d' = O_{\talpha}\left(p^4 \log(k/\beta)\right)$. Let $S$ be a random $d$-dimensional subspace of $\R^d$ and $\Pi_S$ denote the projection from $\R^d$ to $S$. Then, with probability $1 - \beta$, the following holds for every partition $\cX = (\bX_1, \dots, \bX_k)$ of $\bX$:
%\begin{align} \label{eq:cluster-cost-preserving}
%\Pr_S\left[\forall \text{partition } \cX = (\bX_1, \dots, \bX_k) \text{ of } \bX, \cost^p(\cX) \approx_{1 + \talpha} \left(\frac{d}{d'}\right)^{p/2} \cdot \cost^p(\Pi_S(\cX))\right] \geq 1 - \beta,
%\end{align}
\begin{align*}
\cost^p(\cX) \approx_{1 + \talpha} \left(d/d'\right)^{p/2} \cdot \cost^p(\Pi_S(\cX)),
\end{align*}
where $\Pi_S(\cX)$ denote the partition $(\Pi_S(\bX_1), \dots, \Pi_S(\bX_k))$.
\end{theorem}

Another ingredient we need is the algorithms for private empirical risk minimization (ERM). Recall that, in ERM, there is a convex loss function $\ell$ and we are given data points $x_1, \dots, x_n$. The goal to find $\theta$ in the unit ball in $p$ dimension that minimizes $\sum_{i=1}^n \ell(\theta; x_i)$.
When $\ell$ is $L$-Lipschitz, Bassily et al.~\cite{BassilyST14} give an algorithm with small errors, both for pure- and approximate-DP. These are stated formally below.

\begin{theorem}[\cite{BassilyST14}] \label{thm:erm-pure}
Suppose that $\ell(\cdot; x)$ is convex and $L$-Lipschitz for some constant $L$. For every $\eps > 0$, there exists an $\eps$-DP polynomial time algorithm for ERM with loss function $\ell$ such that, with probability $1 - \beta$, the additive error is at most $O_L\left(\frac{d}{\eps} \cdot \poly\log\left(\frac{1}{\beta}\right)\right)$, for every $\beta \in (0, 1)$.
\end{theorem}

\begin{theorem}[\cite{BassilyST14}] \label{thm:erm-apx}
Suppose that $\ell(\cdot; x)$ is convex and $L$-Lipschitz for some constant $L$. For every $0 < \eps < O(1)$ and $0 < \delta < 1$, there exists an $\eps$-DP polynomial time algorithm for ERM with loss function $\ell$ such that, with probability $1 - \beta$, the additive error is at most $O_L\left(\frac{\sqrt{d}}{\eps} \cdot \poly\log\left(\frac{n}{\delta\beta}\right)\right)$, for every $\beta \in (0, 1)$.
\end{theorem}

We remark here that the ``high probability'' versions we use above are not described in the main body of~\cite{BassilyST14}, but they are included in Appendix D of the arXiv version of~\cite{BassilyST14}.

Notice that the $(1, p)$-Clustering is exactly the ERM problem, but with $\ell(\theta, x) = \|\theta - x\|^p$ where $\theta$ is the center. Note that since both $\theta, x \in \cB(0, 1)$, $\ell(\cdot ; x)$ is $O_p(1)$-Lipschitz for $p \geq 1$. It is also simple to see that $\ell(\cdot ; x)$ is convex. Thus, results of~\cite{BassilyST14} immediately yield the following corollaries.

\begin{corollary} \label{cor:1-cluster-pure}
For every $\eps > 0$ and $p \geq 1$, there exists an $\eps$-DP polynomial time algorithm for $(1, p)$-Clustering such that, with probability $1 - \beta$, the additive error is at most $O_p\left(\frac{d}{\eps} \cdot \poly\log\left(\frac{1}{\beta}\right)\right)$, for every $\beta \in (0, 1)$.
\end{corollary}

\begin{corollary} \label{cor:1-cluster-apx}
For every $0 < \eps < O(1), 0 < \delta < 1$ and $p \geq 1$, there exists an $(\eps, \delta)$-DP polynomial time algorithm for $(1, p)$-Clustering such that, with probability $1 - \beta$, the additive error is at most $O_p\left(\frac{\sqrt{d}}{\eps} \cdot \poly\log\left(\frac{n}{\delta\beta}\right)\right)$, for every $\beta \in (0, 1)$.
\end{corollary}

We are now ready to state the algorithm. As outlined before, we start by projecting to a random low-dimensional space and use our low-dimensional algorithm (Theorem~\ref{thm:main-apx-non-private-to-private}) to determine the clusters (i.e., partition). Then, for each of the cluster, we use the algorithms above (Corollaries~\ref{cor:1-cluster-pure} and~\ref{cor:1-cluster-apx}) to find the center.  The full pseudo-code of the algorithm is given in Algorithm~\ref{alg:cluster-high-dim}. There is actually one deviation from our rough outline here: we scale the points after projection by a factor of $\Lambda$ (and zero them out if the norm is larger than one). The reason is: if we do not implement this step, the additive error from our low dimensional algorithm will get multiplied by a factor of $(d/d')^{p/2} = \tilde{\Omega}(d^{p/2})$, which is too large for our purpose. By picking an appropriate scaling factor $\Lambda$, we only incur a polylogarithmic multiplicative factor in the additive error.

\begin{algorithm}[h!]
\caption{Algorithm for $(k, p)$-Clustering.}\label{alg:cluster-high-dim}
\begin{algorithmic}[1]
\Procedure{ClusteringHighDimension$^\eps(x_1, \dots, x_n; r, \alpha; d', \Lambda)$}{}
\State $S \leftarrow$ Random $d'$-dimension subspace of $\R^d$
\For{$i \in \{1, \dots, n\}$}
\State $\tx_i \leftarrow \Pi_S(x_i)$
\If{$\|\tx_i\| \leq 1 / \Lambda$}
\State $x'_i = \Lambda \tx_i$
\Else
\State $x'_i = 0$
\EndIf
\EndFor
\State $(c'_1, \dots, c'_k) \leftarrow \textsc{ClusteringLowDimension}^{\eps/2}(x'_1, \dots, x'_n)$
\State $(\bX_1, \dots, \bX_k) \leftarrow$ the partition induced by $(c'_1, \dots, c'_k)$ on $(x'_1, \dots, x'_n)$
\For{$j \in \{1, \dots, k\}$}
\State $c_j \leftarrow \textsc{FindCenter}^{\eps/2}(\bX_j)$
\EndFor
\State \Return $(c_1, \dots, c_k)$
\EndProcedure
\end{algorithmic}
\end{algorithm}

We will now prove the guarantee of the algorithm, starting with the pure-DP case:

\begin{proof}[Proof of Theorem~\ref{thm:main-cluster-pure-high-dim}]
We simply run Algorithm~\ref{alg:1-center} where $d'$ be as in Theorem~\ref{thm:dim-red-cluster} with failure probability $\beta/4$ and $\talpha = 0.1\alpha$, $\Lambda = \sqrt{\frac{0.01}{\log (n / \beta)} \cdot \frac{d'}{d}}$, \textsc{ClusteringLowDimension} is the algorithm from Theorem~\ref{thm:main-apx-non-private-to-private} that is $(\eps/2)$-DP, has with $\alpha = 0.1\alpha$ and the failure probability $\frac{\beta}{4k}$, and \textsc{FindCenter}
 is the algorithm from Corollary~\ref{cor:1-cluster-pure} that is $(\eps/2)$-DP and the failure probability $\beta/4$. Since algorithm \textsc{ClusteringLowDimension} is $(\eps/2)$-DP and each parition $\bX_j$ is applied \textsc{FindCenter} only once, the trivial composition implies that the entire algorithm is $\eps$-DP. Furthermore, it is obvious that every step except the application of \textsc{ClusteringLowDimension} runs in polynomial time. From Theorem~\ref{thm:main-apx-non-private-to-private}, the application of \textsc{ClusteringLowDimension} takes 
\begin{align*}
(1 + 10/\alpha)^{O_{p, \alpha}(d')}\poly(n) = (1 + 10/\alpha)^{O_{p, \alpha}(\log(k/\beta))}\poly(n) = (k/\beta)^{O_{p, \alpha}(1)} \poly(n)
\end{align*}
time. As a result, the entire algorithm runs in $(k/\beta)^{O_{\alpha}(1)} \poly(nd)$ time as desired.

We will now prove the accuracy of the algorithm. Let $\tbX = (\tx_1, \dots, \tx_n)$ and $\bX = (x'_1, \dots, x'_n)$. By applying Theorem~\ref{thm:dim-red-cluster}, the following holds with probability $1 - \beta / 4$:
\begin{align} \label{eq:opt-preserved-dim-red}
\opt^{p, k}_{\tbX} \leq \left(\frac{d'}{d}\right)^{p/2} \cdot (1 + 0.1\alpha) \cdot \opt^{p, k}_{\bX}.
\end{align}
Furthermore, standard concentration implies that $\|\tx_i\| \leq 1 / \Lambda$ with probability $0.1 \beta / n$. By union bound, this means that the following simultaneously holds for all $i \in \{1, \dots, n\}$ with probability $1 - 0.1\beta$:
\begin{align} \label{eq:scale-not-applied}
x'_i = \Lambda \tx_i.
\end{align}
When~\eqref{eq:opt-preserved-dim-red} and~\eqref{eq:scale-not-applied} both hold, we may apply Theorem~\ref{thm:main-apx-non-private-to-private}, which implies that, with probability $1 - \beta/2$, we have
\begin{align}
&\cost^p_{\bX'}(c_1, \dots, c_k) \nonumber \\
&\leq w(1 + 0.1\alpha) \opt^{p, k}_{\bX'} + O_{p, \alpha, w}\left(\frac{k^2 \log^2 n \cdot 2^{O_{p, \alpha}(d)}}{\eps} \log\left(\frac{n}{\beta}\right)\right) \nonumber \\
&= w(1 + 0.1\alpha) \opt^{p, k}_{\bX'} + O_{p, \alpha, w}\left(\frac{(k/\beta)^{O_{p, \alpha}(1)}}{\eps} \cdot \poly\log\left(\frac{n}{\beta}\right)\right) \nonumber \\
&\overset{\eqref{eq:scale-not-applied}}{=} \Lambda^p \cdot w(1 + 0.1\alpha) \opt^{p, k}_{\tbX} + O_{p, \alpha, w}\left(\frac{(k/\beta)^{O_{p, \alpha}(1)}}{\eps} \cdot \poly\log\left(\frac{n}{\beta}\right)\right) \nonumber \\
&\overset{\eqref{eq:opt-preserved-dim-red}}{\leq} \Lambda^p \cdot w(1 + 0.3\alpha) \opt^{p, k}_{\tbX} + O_{p, \alpha, w}\left(\frac{(k/\beta)^{O_{p, \alpha}(1)}}{\eps} \cdot \poly\log\left(\frac{n}{\beta}\right)\right), \label{eq:low-dim-apx-guarantee}
\end{align}
where the first equality follows from $d' = O_{p, \alpha}\left(\log\left(\frac{k}{\beta}\right)\right)$.

Let $\bX'_1, \dots, \bX'_k$ partition of $\bX'$ induced by $c_1, \dots, c_k$, and let $\tbX_1, \dots, \tbX_k$ denote the corresponding partition of $\tbX$. From Theorem~\ref{thm:dim-red-cluster}, the following holds with probability $1 - \beta/4$:
\begin{align} \label{eq:partition-cost-preserved}
\cost^p_{(\bX_1, \dots, \bX_k)} 
&\leq \left(\frac{d}{d'}\right)^{p/2} \cdot (1 + 0.1\alpha) \cdot \cost^p_{(\tbX_1, \dots, \tbX_k)}.
\end{align}
By union bound~\eqref{eq:opt-preserved-dim-red},~\eqref{eq:scale-not-applied},~\eqref{eq:low-dim-apx-guarantee} and~\eqref{eq:partition-cost-preserved} together occur with probability $1 - 3\beta / 4$. When this is the case, we have
\begin{align}
&\cost^p_{(\bX_1, \dots, \bX_k)} \nonumber \\
&\overset{\eqref{eq:partition-cost-preserved}}{\leq} \left(\frac{d}{d'}\right)^{p/2} \cdot (1 + 0.1\alpha) \cdot \cost^p_{(\tbX_1, \dots, \tbX_k)} \nonumber \\
&\overset{\eqref{eq:scale-not-applied}}{=} \frac{1}{\Lambda^p} \cdot \left(\frac{d}{d'}\right)^{p/2} \cdot (1 + 0.1\alpha) \cdot \cost^p_{(\bX'_1, \dots, \bX'_k)} \nonumber \\
&= \frac{1}{\Lambda^p} \cdot \left(\frac{d}{d'}\right)^{p/2} \cdot (1 + 0.1\alpha) \cdot \cost^p_{\bX'}(c_1, \dots, c_k) \nonumber \\
&\overset{\eqref{eq:low-dim-apx-guarantee}}{\leq} \left(\frac{d}{d'}\right)^{p/2} \cdot w (1 + 0.5\alpha) \cdot \opt^{p, k}_{\tbX} + O_{p, \alpha, w}\left(\frac{1}{\Lambda^p} \cdot \left(\frac{d}{d'}\right)^{p/2} \cdot \frac{(k/\beta)^{O_{p, \alpha}(1)}}{\eps} \cdot \poly\log\left(\frac{n}{\beta}\right)\right) \nonumber \\
&\overset{\eqref{eq:opt-preserved-dim-red}}{\leq} w (1 + \alpha) \cdot \opt^{p, k}_{\bX} + O_{p, \alpha, w}\left(\frac{1}{\Lambda^p} \cdot \left(\frac{d}{d'}\right)^{p/2} \cdot \left(\frac{(k/\beta)^{O_{p, \alpha}(1)}}{\eps} \cdot \poly\log\left(\frac{n}{\beta}\right)\right)\right) \nonumber \\
&= w (1 + \alpha) \cdot \opt^{p, k}_{\bX} + O_{p, \alpha, w}\left(\frac{(k/\beta)^{O_{p, \alpha}(1)}}{\eps} \cdot \poly\log\left(\frac{n}{\beta}\right)\right), \label{eq:apx-guarantee-main}
\end{align}
where in the last inequality we use the fact that, by our choice of parameters, $\frac{1}{\Lambda^2} \cdot \frac{d}{d'} = O(\log(1/\beta))$.

Now, using the guarantee from Corollary~\eqref{cor:1-cluster-pure} and the union bound over all $j = 1, \dots, k$, the following holds simultaneously for all $j = 1, \dots, k$ with probability $1 - \beta / 4$:
\begin{align} \label{eq:additive-error-1-cluster}
\cost^p_{\bX_j}(c_j) &\leq \opt^{p, 1}_{\bX_j} + O_p\left(\frac{d}{\eps} \cdot \log\left(\frac{k}{\beta}\right)\right).
\end{align}
When~\eqref{eq:apx-guarantee-main} and~\eqref{eq:additive-error-1-cluster} both occur (with probability at least $1 - \beta$), we have 
\begin{align*}
\cost^{p}_{\bX}(c_1, \dots, c_k) 
&\leq \sum_{j=1}^k \cost^p_{\bX_j}(c_j) \\
&\overset{\eqref{eq:additive-error-1-cluster}}{\leq} \sum_{j=1}^k \left(\opt^{p, 1}_{\bX_j} + O_p\left(\frac{d}{\eps} \cdot \log\left(\frac{k}{\beta}\right)\right)\right) \\
&= \cost^p_{(\bX_1, \dots, \bX_k)} + O_p\left(\frac{k d}{\eps} \cdot \log\left(\frac{k}{\beta}\right)\right) \\
&\overset{\eqref{eq:apx-guarantee-main}}{\leq} w (1 + \alpha) \cdot \opt^{p, k}_{\bX} + O_{p, \alpha, w}\left(\left(\frac{kd + (k/\beta)^{O_{p, \alpha}(1)}}{\eps}\right) \cdot \poly\log\left(\frac{n}{\beta}\right)\right),
\end{align*}
which concludes our proof.
\end{proof}

We will next state the proof for approximate-DP case, which is almost the same as that of the pure-DP case.

\begin{proof}[Proof of Theorem~\ref{thm:main-cluster-apx-high-dim}]
This proof is exactly the same as that of Theorem~\ref{thm:main-cluster-pure-high-dim}, except that we use the $(1, p)$-Clustering algorithm from Corollary~\ref{cor:1-cluster-apx} instead of Corollary~\ref{cor:1-cluster-pure}. Everything in the proof remains the same except that the additive error on the right handside of~\eqref{eq:additive-error-1-cluster} becomes $O_p\left(\frac{\sqrt{d}}{\eps} \cdot \log\left(\frac{k}{\delta \beta}\right)\right)$ (instead of $O_p\left(\frac{d}{\eps} \cdot \log\left(\frac{k}{\beta}\right)\right)$ as in Theorem~\ref{thm:main-cluster-pure-high-dim}), resulting in the new additive error bound.
\end{proof}

We remark that Theorems~\ref{thm:main-cluster-pure-high-dim} and~\ref{thm:main-cluster-apx-high-dim} imply Theorem~\ref{thm:main-apx-non-private-to-private-pure-intro} in Section~\ref{sec:kmeans-kmedian-main-body}.

\paragraph{FPT Approximation Schemes.} Finally, we state the results for FPT algorithms below. These are almost exactly the same as above, except that we use the FPT algorithm from Theorem~\ref{thm:fpt-apx} to solve the low-dimensional $(k, p)$-Clustering, leading to approximation ratio arbritrarily close to one. 

\begin{theorem} \label{thm:fpt-cluster-pure-high-dim}
For every $0 < \eps \leq O(1)$, $0 < \alpha, \beta \leq 1$ and $p \geq 1$, there exists an $\eps$-DP algorithm that runs in $(1 / \beta)^{O_{p, \alpha}(k \log k)} \poly(nd)$ time and, w.p. $1 - \beta$, outputs an $\left(1 + \alpha, O_{p, \alpha}\left(\left(\frac{kd + k^2}{\eps}\right) \cdot \poly\log\left(\frac{n}{\beta}\right)\right)\right)$-approximation for $(k, p)$-Clustering.
\end{theorem}

\begin{proof}
This proof is the same as the proof of Theorem~\ref{thm:main-cluster-pure-high-dim}, except that we use the algorithm from Theorem~\ref{thm:fpt-apx} instead of that from Theorem~\ref{thm:main-apx-non-private-to-private}. Note here that the bottleneck in the running time is from the application of Theorem~\ref{thm:fpt-apx}, which takes $2^{O_{p, \alpha}(d' k + k \log k)} \cdot \poly(n) = (1/\beta)^{O_{p, \alpha}(k \log k)} \cdot \poly(n)$ time because $d = O_{p, \alpha}(\log (k/\beta))$.
\end{proof}

\begin{theorem} \label{thm:fpt-cluster-apx-high-dim}
For every $0 < \eps \leq O(1)$, $0 < \delta, \alpha, \beta \leq 1$ and $p \geq 1$, there exists an $(\eps, \delta)$-DP algorithm that runs in $(1 / \beta)^{O_{p, \alpha}(k \log k)} \poly(nd)$ time and, with probability $1 - \beta$, outputs an $\left(1 + \alpha, O_{p, \alpha}\left(\left(\frac{k\sqrt{d}}{\eps} \cdot \poly\log\left(\frac{k}{\delta \beta}\right) \right) + \left(\frac{k^2}{\eps} \cdot \poly\log\left(\frac{n}{\beta}\right)\right)\right)\right)$-approximation for $(k, p)$-Clustering. 
\end{theorem}

\begin{proof}
This is exactly the same as the proof of Theorem~\ref{thm:main-cluster-apx-high-dim}, except that we use the algorithm from Theorem~\ref{thm:fpt-apx} instead of that from Theorem~\ref{thm:main-apx-non-private-to-private}.
\end{proof}

%Finally, we remark that Theorems~\ref{thm:fpt-cluster-pure-high-dim} and~\ref{thm:fpt-cluster-apx-high-dim} imply Theorems~\ref{thm:fpt-cluster-pure-high-dim-intro} and~\ref{thm:fpt-cluster-apx-high-dim-intro} in the introduction respectively.

\subsection{\densestball}
\label{subsec:dim-red-densest-ball}

We refer to the variant of the \densestball problem where we are promised that \emph{all} points are within a certain radius as the \onecenter problem:

\begin{definition}[\onecenter]
The input of \onecenter consists of $n$ points in the $d$-dimensional unit ball and a positive real number $r$. It is also promised that \emph{all} input points lie in some ball of radius $r$. A $(w, t)$-approximation for
\onecenter is a ball $B$ of radius $w \cdot r$ that contains at least $n - t$ input points.
\end{definition}

\subsubsection{\onecenter Algorithm in High Dimension}

Once again, we will first show how to solve the \onecenter problem in high dimensions:
\begin{lemma} \label{lem:1-center-pure}
For every $\eps > 0$ and $0 < \alpha, \beta \leq 1$, there exists an $\eps$-DP algorithm that runs in time $(nd)^{O_{\alpha}(1)} \poly\log(1/r)$ and, with probability $1 - \beta$, outputs an $\left(1 + \alpha, O_{\alpha}\left(\frac{d}{\epsilon}\cdot\log\left(\frac{d}{\beta r}\right)\right)\right)$-approximation for \onecenter.
\end{lemma}

\begin{lemma} \label{lem:1-center-apx}
For every $0 < \eps \leq O(1)$ and $0 < \alpha, \beta, \delta \leq 1$, there exists an $(\eps, \delta)$-DP algorithm that runs in time $(nd)^{O_{\alpha}(1)} \poly\log(1/r)$ and, w.p. $1 - \beta$, outputs an $\left(1 + \alpha, O_{\alpha} \left(\frac{\sqrt{d}}{\eps}\cdot\poly\log\left(\frac{n d}{\eps \delta \beta}\right)\right)\right)$-approximation for \onecenter.
\end{lemma}

A natural way to solve the \onecenter problem in high dimensions is to use differentially private ERM similarly to the case of $(k, p)$-Clustering, but with a \emph{hinge loss} such as $\ell(c, x) = \frac{1}{r}\max\{0, r - \|c - x\|\}$. In other words, the loss is zero if $c$ is within the ball of radius $r$ aroun the center $c$, whereas the loss is at least one when it is say at a distance $2r$ from $c$. The main issue with this approach is that the Lipchitz constant of this function is as large as $1/r$. However, since the expected error in the loss has to grow linearly with the Lipchitz constant~\cite{BassilyST14}, this will give us an additive error that is linear in $1/r$, which is undesirable.

Due to this obstacle, we will instead take a different path: use a dimensionality reduction argument again! More specifically, we randomly rotate each vector and think of blocks each of roughly $O(\log (nd))$ coordinates as a single vector. We then run our low-dimensional \densestball algorithm from Section~\ref{sec:densest-ball-low-dim} on each block. Combining these solutions together immediately gives us the desired solution in the high-dimensional space. The full pseudo-code of the procedure is given below; here $b$ is the parameter of the algorithm, \textsc{DensestBallLowDimension} is the algorithm for solving \densestball in low dimensions, and we use the notation $y|_{i, \dots, j}$ to denote a vector resulting from the restriction of $y$ to the coordinates $i, \dots, j$.

\begin{algorithm}[h!]
\caption{1-Center Algorithm.}\label{alg:1-center}
\begin{algorithmic}[1]
\Procedure{1-Center$_{b}(x_1, \dots, x_n; r, \alpha)$}{}
\State $R \leftarrow$ Random $(d \times d)$ rotation matrix
\For{$i \in \{1, \dots, n\}$}
\For{$j \in \{1, \dots, b\}$}
\State $x_i^j \leftarrow (Rx_i)|_{1 + \lfloor \frac{(j - 1) d}{b}\rfloor, \dots, \lfloor \frac{j d}{b}\rfloor}$
\EndFor
\EndFor
\For{$j \in \{1, \dots, b\}$}
\State $d^j \leftarrow \lfloor \frac{j d}{b}\rfloor - \lfloor \frac{(j - 1) d}{b}\rfloor$
\State $r^j \leftarrow (1 + 0.1\alpha) \cdot \sqrt{d^j/d} \cdot r$
\State $c^j \leftarrow \textsc{DensestBallLowDimension}(x_1^j, \dots, x_n^j; r^j, 0.1\alpha)$. \label{step:densest-ball-low-dim-app}
\EndFor
\State $\tc \leftarrow$ concatenation of $c^1, \dots, c^t$
\State \Return $R^{-1}(\tc)$
\EndProcedure
\end{algorithmic}
\end{algorithm}

To prove the correctness of our algorithm, we will need the Johnson--Lindenstrauss (JL) lemma~\cite{JL}. The version we use below follows from the proof in~\cite{DasguptaG03}. 

\begin{theorem}[\cite{DasguptaG03}] \label{thm:jl}
Let $v$ be any $d$-dimensional vector. Let $S$ denote a random $d$-dimensional subspace of $\R^d$ and let $\Pi_S$ denote the projection from $\R^d$ onto $S$. Then, for any $\zeta \in (0, 1)$ we have
\begin{align*}
\Pr\left[\|v\|_2 \approx_{1 + \zeta} \sqrt{d/d'} \cdot \|\Pi v\|_2 \right] \geq 1 - 2\exp\left(-\frac{d'\zeta^2}{100}\right).
\end{align*}
\end{theorem}

We are now ready to prove our results for \onecenter, starting with the pure-DP algorithm (Lemma~\ref{lem:1-center-pure}).

\begin{proof}[Proof of Lemma~\ref{lem:1-center-pure}]
We simply run Algorithm~\ref{alg:1-center} with $b = \max\left\{1, \lfloor \frac{d}{10^8\log(nd/\beta)/\alpha^2} \rfloor\right\}$ and with $\textsc{DensestBallLowDimension}$ on Line~\ref{step:densest-ball-low-dim-app} being the algorithm $\cA$ from Theorem~\ref{thm:pure-1-cluster-given-radius} that is $(\eps/b)$-DP, has approximation ratio $w = 1 + 0.1\alpha$ and failure probability $\frac{\beta}{2d}$. Since algorithm $\cA$ is $(\eps/b)$-DP and we apply the algorithm $b$ times, the trivial composition implies that the entire algorithm is $\eps$-DP. Furthermore, it is obvious that every step except the application of $\cA$ runs in polynomial time. From Theorem~\ref{thm:pure-1-cluster-given-radius}, the $j$th application of $\cA$ takes time
\begin{align*}
(1 + 1/\alpha)^{O_{\alpha}(d/b)}\poly\log(1/r') &= (1 + 1/\alpha)^{O_{\alpha}(\log(nd\beta))}\poly\log(\sqrt{d/d^j} \cdot r) \\ 
&= (nd)^{O_{\alpha}(1)} \poly\log(1/r).
\end{align*}
As a result, the entire algorithm runs in time $(nd)^{O_{\alpha}(1)} \poly\log(1/r)$ as desired.

The remainder of this proof is dedicated to proving the accuracy of the algorithm. To do this, let $c_{\opt}$ denote the solution, i.e., the center such that $x_1, \dots, x_n \in \cB(c_{\opt}, r)$. Moreover, for every $j \in \{1, \dots, b\}$, let $c^j_{\opt}$ be $R(c_{\opt})$ restricted to the coordinates $1 + \lfloor \frac{(j - 1) d}{b}\rfloor, \dots, \lfloor \frac{j d}{b}\rfloor$.

Notice that $d^j \geq \frac{d}{10^6\log(nd\beta)/\alpha^2}$ for every $j \in \{1, \dots, b\}$. As a result, by applying Theorem~\ref{thm:jl} and the union bound, the following bounds hold simultaneously for all $j \in \{1, \dots, b\}$ and $i, i' \in \{1, \dots, n\}$ with probability $1 - \beta / 2$:
\begin{align}
\|x_i^j - c_{\opt}^j\| &\leq (1 + 0.1\alpha) \cdot \sqrt{\frac{d^j}{d}} \cdot \|x_i - c_{\opt}\| \leq r^j, \label{eq:projected-center-distance} \\
\|x_i^j - x_{i'}^j\| &\leq (1 + 0.1\alpha) \cdot \sqrt{\frac{d^j}{d}} \cdot \|x_i - x_{i'}\| \leq 2 r^j, \label{eq:projected-points-distance}
\end{align}
where the last inequality follows from the triangle inequality (through $c_{\opt}$).

Observe that, when~\eqref{eq:projected-center-distance} holds, $x^j_1, \dots, x^j_n \in \cB(c_{\opt}^j, r^j)$. As a result, the accuracy guarantee from Theorem~\ref{thm:pure-1-cluster-given-radius} and the union bound implies that the following holds for all $j \in \{1, \dots, b\}$, with probability $1 - \beta / 2$, we have
\begin{align} \label{eq:projected-ball-size}
|\{x_1^j, \dots, x_n^j\} \setminus \cB(c, (1 + 0.1\alpha) r^j)| \leq t^j,
\end{align}
where $t^j = O_{\alpha}\left(\frac{d^j}{(\epsilon / b)}\log\left(\frac{1}{(\beta/2b) r^j}\right)\right) = O_{\alpha}\left(\frac{d}{\eps} \cdot \log\left(\frac{d}{\beta r}\right)\right)$. For convenience, let $t^{\max} = \max_{j \in \{1, \dots, b\}} t^j = O_{\alpha}\left(\frac{d}{\epsilon}\cdot\log\left(\frac{d}{\beta r}\right)\right)$.

We may assume that $n > t^{\max}$ as otherwise the desired accuracy guarantee holds trivially. When this is the case, we have that $\{x_1^j, \dots, x_n^j\} \cap \cB(c, (1 + 0.1\alpha)r^j)$ is not empty. From this and from~(\ref{eq:projected-points-distance}), we have
\begin{align} \label{eq:any-to-projected-center}
\|x_i^j - c^j\| \leq (3 + 0.1\alpha) r^j \leq 3.1 r^j,
\end{align}
for all $j \in \{1, \dots, b\}$ and $i \in \{1, \dots, n\}$.

To summarize, we have so far shown that~\eqref{eq:projected-center-distance},~\eqref{eq:projected-points-distance},~\eqref{eq:projected-ball-size}, and~\eqref{eq:any-to-projected-center} hold simultaneously for all $j \in \{1, \dots, b\}$ and $i, i' \in \{1, \dots, n\}$ with probability at least $1 - \beta$. We will henceforth assume that this ``good'' event occurs and show that we have the desired additive error bound, i.e., $|\{x_1, \dots, x_n\} \setminus \cB(c, (1 + \alpha)r)| \leq O_{\alpha}\left(\frac{d}{\epsilon}\cdot\log\left(\frac{d}{\beta r}\right)\right)$. 

To prove such a bound, let $\bX_{\text{far}} = \{x_1, \dots, x_n\} \setminus \cB(c, (1 + \alpha)r)$ and, for every $j \in \{1, \dots, b\}$, let $\bX_{\text{far}}^j = \{x_1^j, \dots, x_n^j\} \setminus \cB(c, (1 + 0.1\alpha)r^j)$. Notice that, for every input point $x_i$, we have
\begin{align}
\|x_i - c\|^2
&= \|Rx_i - \tc\|^2 \nonumber \\
&= \sum_{j \in \{1, \dots, b\}} \|x^j_i - c^j\|^2 \nonumber \\
&= \sum_{j \in \{1, \dots, b\} \atop x_i \notin \bX^j_{\text{far}}} \|x^j_i - c^j\|^2 + \sum_{j \in \{1, \dots, b\} \atop x_i \in \bX^j_{\text{far}}} \|x^j_i - c^j\|^2 \nonumber \\
&\overset{\eqref{eq:any-to-projected-center}}{\leq} \sum_{j \in \{1, \dots, b\} \atop x_i \notin \bX^j_{\text{far}}} (1 + 0.1\alpha)^2 (r^j)^2 + \sum_{j \in \{1, \dots, b\} \atop x_i \in \bX^j_{\text{far}}} (3.1r^j)^2 \nonumber \\
&\leq (1 + 0.1\alpha)^4 r^2 + \sum_{j \in \{1, \dots, b\} \atop x_i \in \bX^j_{\text{far}}} (3.1r^j)^2, \label{eq:distance-bound-intermediate}
\end{align}
where the last inequality follows from the identity $(r^1)^2 + \cdots + (r^b)^2 = (1 + 0.1\alpha)^2r^2$. Notice also that, since $d^j$ is within a factor of 2 of each other, this implies that $r^j \leq (4(1 + 0.1\alpha)^2r^2) / b \leq \frac{16r^2}{b}$ for all $j \in \{1, \dots, b\}$. Plugging this back to~\eqref{eq:distance-bound-intermediate}, we have
\begin{align*}
\|x_i - c\|^2 &\leq (1 + 0.1\alpha)^4 r^2 + \frac{160 r^2}{b} \cdot |\{j \in \{1, \dots, b\} \mid x_i \in \bX^j_{\text{far}}\}|
\end{align*}
Recall that $x_i \in \bX_{\text{far}}$ iff $\|x_i - \tc\| \geq (1 + \alpha)r$. Hence, for such $x_i$, we must have
\begin{align*}
|\{j \in \{1, \dots, b\} \mid x_i \in \bX^j_{\text{far}}\}|
&\geq \frac{b}{160r^2} \cdot \left( (1 + \alpha)^2 r^2 - (1 + 0.1\alpha)^4 r^2 \right) \\
&\geq \frac{b \alpha}{160}.
\end{align*}
Summing the above inequality over all $x_i \in \bX_{\text{far}}$, we have
\begin{align*}
\sum_{j \in \{1, \dots, b\}} |\bX^j_{\text{far}}| &\geq \frac{b\alpha}{160} \cdot |\bX_{\text{far}}|.
\end{align*}
Recall from~\eqref{eq:projected-ball-size} that $|\bX^j_{\text{far}}| \leq t^{\max}$. Together with the above, we have
\begin{align*}
|\bX_{\text{far}}| \leq \frac{160}{b\alpha} \cdot b \cdot t^{\max} = O_{\alpha}\left(\frac{d}{\epsilon}\cdot\log\left(\frac{d}{\beta r}\right)\right),
\end{align*}
which concludes our proof.
\end{proof}

The proof of Lemma~\ref{lem:1-center-apx} is similar, except we use the approximate-DP algorithm for \densestball (from Theorem~\ref{thm:apx-1-cluster-given-radius}) as well as advanced composition (Theorem~\ref{thm:advance-composition}).

\begin{proof}[Proof of Lemma~\ref{lem:1-center-apx}]
We simply run Algorithm~\ref{alg:1-center} with $b = \max\left\{1, \lfloor \frac{d}{10^6\log(nd/\beta)/\alpha^2} \rfloor\right\}$, and with $\cA$ being the algorithm from Theorem~\ref{thm:apx-1-cluster-given-radius} that is $(\eps', \delta')$-DP with $\eps' = \min\left\{1, \frac{\eps}{100\sqrt{b \ln(2/\delta)}}\right\}$ and $\delta' = 0.5\delta/b$, has approximation ratio $w = 1 + 0.1\alpha$ and failure probability $\frac{\beta}{2d}$. Since algorithm $\cA$ is $(\eps', \delta')$-DP and we apply the algorithm $b$ times, the advanced composition theorem (Theorem~\ref{thm:advance-composition}) implies\footnote{Here we use that fact that, since $\eps' \leq 1$, we have $e^{\eps'} - 1 < 10\eps'$.} that the entire algorithm is $(\eps, \delta)$-DP. The running time analysis is exactly the same as that of Lemma~\ref{lem:1-center-pure}.

Finally, the proof of the additive error bound is almost identical to that of Lemma~\ref{lem:1-center-pure}, except that here, using Theorem~\ref{thm:apx-1-cluster-given-radius} instead of Theorem~\ref{thm:pure-1-cluster-given-radius}, we have 
\begin{align*}
t^{j}
&= O_{\alpha}\left(\frac{d^j}{\eps'}\log\left(\frac{n}{\eps' \delta' \cdot (0.5\beta/b)}\right)\right) \\
&= O_{\alpha}\left(\frac{(d/b)}{\eps / \sqrt{b \log(1/\delta)}}\log\left(\frac{n}{\min\{(\eps / \sqrt{b \log(1/\delta)}), 1\} \cdot(\delta/b) \cdot(0.5\beta/b)}\right)\right) \\
&\leq O_{\alpha}\left(\frac{d}{\sqrt{b}} \cdot \frac{\sqrt{\log(1/\delta)}}{\eps} \cdot\log\left(\frac{n d}{\eps \delta \beta}\right)\right) \\
&= O_{\alpha}\left(\sqrt{d \log(n d / \beta)} \cdot \frac{\sqrt{\log(1/\delta)}}{\eps} \cdot\log\left(\frac{n d}{\eps \delta \beta}\right)\right) \\
&= O_{\alpha} \left(\frac{\sqrt{d}}{\eps}\cdot\poly\log\left(\frac{n d}{\eps \delta \beta}\right)\right),
\end{align*}
which results in a similar bound on the additive error for the overall algorithm.
\end{proof}

\subsubsection{From \onecenter to \densestball via Dimensionality Reduction}
\label{sec:1-center-to-densest-ball}

We are now ready to prove the main theorems regarding \densestball (Theorems~\ref{thm:densest-ball-pure} and~\ref{thm:densest-ball-apx}).

\begin{theorem} \label{thm:densest-ball-pure}
For every $\eps > 0$ and $0 < \alpha, \beta \leq 1$, there exists an $\eps$-DP algorithm that runs in $(nd)^{O_{\alpha}(1)} \poly\log(1/r)$ time and, with probability $1 - \beta$, outputs an $\left(1 + \alpha, O_{\alpha}\left(\frac{d}{\epsilon}\cdot \log\left(\frac{d}{\beta r}\right)\right)\right)$-approximation for \densestball.
\end{theorem}

\begin{theorem} \label{thm:densest-ball-apx}
For every $0 < \eps \leq O(1)$ and $0 < \delta, \alpha, \beta  \leq 1$, there exists an $(\eps, \delta)$-DP algorithm that runs in $(nd)^{O_{\alpha}(1)} \poly\log(1/r)$ time and, with probability $1 - \beta$, solves the \densestball problem with approximation ratio $1 + \alpha$ and additive error $O_{\alpha} \left(\frac{\sqrt{d}}{\eps}\cdot\poly\log\left(\frac{n d}{\eps \delta \beta}\right)\right)$.
\end{theorem}

Note here that Theorems~\ref{thm:densest-ball-pure} and~\ref{thm:densest-ball-apx} imply Theorems~\ref{thm:1-cluster-given-radius-intro} in Section~\ref{sec:densest-ball-main-body}.

With the \onecenter algorithm in the previous subsection, the algorithm for \densestball in high dimension follows the same footprint as its counterpart for $(k, p)$-Clustering. The pseudo-code is given below.

\begin{algorithm}[h!]
\caption{\densestball Algorithm (High Dimension).}\label{alg:densest-ball-high-dim}
\begin{algorithmic}[1]
\Procedure{DensestBallHighDimension$_{d'}(x_1, \dots, x_n; r, \alpha)$}{}
\State $S \leftarrow$ Random $d'$-dimensional subspace of $\R^d$
\For{$i \in \{1, \dots, n\}$}
\State $x'_i \leftarrow $ projection of $x_i$ onto $S$
\State $r' \leftarrow (1 + 0.1\alpha) \cdot \sqrt{d'/d} \cdot r$
\EndFor
\State $c' \leftarrow \textsc{DensestBallLowDimension}(x'_1, \dots, x'_n; r', 0.1\alpha)$
\State $\bX_{\text{cluster}} = \{x_i \mid x'_i \in \cB(c', (1 + 0.1\alpha)r')\}$
\State \Return \textsc{1-Center}$(\bX_{\text{cluster}}; (1 + 0.1\alpha)^3r, 0.1\alpha)$ \label{step:1-center}
\EndProcedure
\end{algorithmic}
\end{algorithm}

To prove Theorems~\ref{thm:densest-ball-pure} and~\ref{thm:densest-ball-apx}, we will also need the following well-known theorem. Its use in our proof below has appeared before in similar context of clustering (see, e.g.,~\cite{MakarychevMR19}).

\begin{theorem}[Kirszbraun Theorem~\cite{kirszbraun1934}] \label{thm:kirszbraun}
Suppose that there exists an $L$-Lipchitz map $\psi$ from $X \subseteq \R^d$ to $\R^{d'}$. Then, there exists an $L$-Lipchitz extension\footnote{Recall that $\tpsi$ is an extension of $\psi$ iff $\tpsi(x) = \psi(x)$ for all $x \in X$.} $\tpsi$ of $\psi$ from $\R^d$ to $\R^{d'}$.
\end{theorem}

\begin{proof}[Proof of Theorem~\ref{thm:densest-ball-pure}]
We simply run Algorithm~\ref{alg:densest-ball-high-dim} where $d' = \min\left\{d, \lceil 10^6\log(nd/\beta)/\alpha^2 \rceil\right\}$, \textsc{DensestBallLowDimension} is the algorithm from Theorem~\ref{thm:pure-1-cluster-given-radius} that is $(\eps/2)$-DP, has approximation ratio $w = 1 + 0.1\alpha$ and the failure probability $\frac{\beta}{3}$, and the \onecenter algorithm on Line~\ref{step:1-center} is the algorithm from Lemma~\ref{lem:1-center-pure} that is $(\eps/2)$-DP, has approximation ratio $w = 1 + 0.1\alpha$ and the failure probability $\frac{\beta}{3}$.  Basic composition immediately implies that the entire algorithm is $\eps$-DP. Furthermore, similar to the proof of Lemma~\ref{lem:1-center-pure}, it is also simple to check that the entire algorithm runs in $(nd)^{O_{\alpha}(1)} \poly\log(1/r)$ time as desired.

We will now argue the accuracy of the algorithm. To do this, let $c_{\opt}$ be the solution, i.e., the center such that $|\{x_1, \dots, x_n\} \cap \cB(c_{\opt}, r)|$ is maximized; we let $T = |\{x_1, \dots, x_n\} \cap \cB(c_{\opt}, r)|$ . Moreover, let $c'_{\opt}$ denote the projection of $c_{\opt}$ onto $S$.

By applying Theorem~\ref{thm:jl} and the union bound, the following holds simultaneously for all $j \in \{1, \dots, t\}$ and $i, i' \in \{1, \dots, n\}$ with probability $1 - \beta / 3$:
\begin{align}
\|x_i^j - c'_{\opt}\| &\leq (1 + 0.1\alpha) \cdot \sqrt{\frac{d'}{d}} \cdot \|x_i - c_{\opt}\| \leq r', \label{eq:projected-center-distance-second} \\
\|x'_i - x'_{i'}\| &\approx_{1 + 0.1\alpha} \sqrt{\frac{d'}{d}} \cdot \|x_i - x_{i'}\|. \label{eq:projected-points-distance-lipchitz}
\end{align}

When~\eqref{eq:projected-center-distance-second} holds, $x'_1, \dots, x'_n \in \cB(c'_{\opt}, r')$. As a result, from the accuracy guarantee from Theorem~\ref{thm:pure-1-cluster-given-radius}, with probability $1 - \beta/3$, we have
\begin{align} \label{eq:projected-ball-size-second}
|\bX_{\text{cluster}}| \geq T - O_{\alpha}\left(\frac{d'}{(\epsilon / 2)}\log\left(\frac{1}{(\beta/2) r'}\right)\right) \geq T - O_{\alpha}\left(\frac{d}{\eps} \cdot \log\left(\frac{d}{\beta r}\right)\right).
\end{align}

Now, consider the map $\psi: \{x'_1, \dots, x'_n\} \to \R^d$ where $\psi(x'_i) = x_i$. From~\eqref{eq:projected-points-distance-lipchitz}, this map is $L$-Lipchitz for $L = (1 + 0.1\alpha) \sqrt{\frac{d}{d'}}$. Thus, from the Kirszbraun Theorem (Theorem~\ref{thm:kirszbraun}), there exists an $L$-Lipchitz extension $\tpsi$ of $\psi$. Consider $\tpsi(c')$. By the $L$-Lipchitzness of $\tpsi$, we have 
\begin{align} \label{eq:original-distance-to-center}
\|x_i - \tpsi(c')\| \leq L \cdot \|x'_i - c'\| \leq (1 + 0.1\alpha) \sqrt{\frac{d}{d'}} \cdot (1 + 0.1\alpha) r' = (1 + 0.1\alpha)^3 r.
\end{align}
for all $x_i \in \bX_{\text{cluster}}$.

When~\eqref{eq:original-distance-to-center} holds, the accuracy guarantee of Lemma~\ref{lem:1-center-pure} implies that with probability $1 - \beta / 3$ the output center $c$ from \textsc{1-Center}, satisfies
\begin{align} \label{eq:1-center-guarantee-final}
|\cB(c, (1 + 0.1\alpha)^4 r)| \geq |\bX_{\text{cluster}}| - O_{\alpha}\left(\frac{d}{\epsilon}\cdot\log\left(\frac{d}{\beta r}\right)\right).
\end{align}
Finally, observe that $(1 + 0.1\alpha)^4 \leq (1 + \alpha)$. Hence, by combining~\eqref{eq:projected-ball-size-second} and~\eqref{eq:1-center-guarantee-final}, the algorithm solves the \densestball problem with approximation ratio $1 + \alpha$ and size error $O_{\alpha}\left(\frac{d}{\epsilon}\cdot\log\left(\frac{d}{\beta r}\right)\right)$.
\end{proof}

\begin{proof}[Proof of Theorem~\ref{thm:densest-ball-apx}]
This proof is exactly the same as that of Theorem~\ref{thm:densest-ball-pure}, except that we use $(\eps/2, \delta/2)$-DP algorithms as subroutines (instead of $\eps/2$-DP algorithms as before). The size error bounds from Theorem~\ref{thm:apx-1-cluster-given-radius} and Lemma~\ref{lem:1-center-apx} can then be used in placed of those from Theorem~\ref{thm:pure-1-cluster-given-radius} and Lemma~\ref{lem:1-center-pure}, resulting in the new $O_{\alpha} \left(\frac{\sqrt{d}}{\eps}\cdot\poly\log\left(\frac{n d}{\eps \delta \beta}\right)\right)$ bound.
\end{proof}

\section{From \densestball to \onecluster}\label{sec:one_cluster_from_densest_ball}
\newcommand{\low}{\mathrm{low}}
\newcommand{\high}{\mathrm{high}}
\newcommand{\opti}{\mathrm{opt}}
\newcommand{\approxi}{\mathrm{approx}}

In this section, we prove Theorem~\ref{th:1_cluster_app}. We start by formally defining the \onecluster problem.
\begin{definition}[\onecluster, e.g., \cite{NissimSV16}]\label{def:1_cluster}
Let $n$, $T$ and $t$ be non-negative integers and let $w \geq 1$ be a real number. The input to \onecluster consists of a subset $S$ of $n \geq T$ points in $\mathbb{B}_{\kappa}^d$, the discretized $d$-dimensional unit ball with a minimum discretization step of $\kappa$ per dimension. An algorithm is said to solve the \onecluster problem with multiplicative approximation $w$, additive error $t$ and probability $1-\beta$ if it outputs a center $c$ and a radius $r$ such that, with probability at least $1-\beta$, the ball of radius $r$ centered at $c$ contains at least $T-t$ points in $S$ and $r \le w \cdot r_{\opti}$ where $r_{\opti}$ is the radius of the smallest ball containing at least $T$ points in $S$.
%if $r_{\opti}$ is the radius of the smallest ball containing at least $T$ points in $S$, then $r \le w \cdot r_{\opti}$.

Moreover, we denote by \onecluster$_{r_{\low}, r_{\high}}$ the corresponding promise problem where $r_{\opti}$ is guaranteed to be between $r_{\low}$ and $r_{\high}$ for given $0 < r_{\low} < r_{\high} < 1$.
\end{definition}

Note that for $r_{\low} = \kappa$ and $r_{\high} = 1$ in Definition~\ref{def:1_cluster}, the \onecluster$_{r_{\low}, r_{\high}}$ problem coincides with the \onecluster problem without promise.

The following lemma allows us to use our DP algorithm for \densestball in order to obtain a DP algorithm for \onecluster.

\begin{lemma}[DP Reduction from \onecluster$_{r_{\low}, r_{\high}}$ to \densestball]\label{le:red_pure_densest_ball_1_cluster}
Let $\epsilon, \delta > 0$. If there is an $(\epsilon, \delta)$-DP algorithm for \densestball with approximation ratio $w$, additive error $t(n, d, w, r, \epsilon, \delta, \beta)$ and running time $\tau(n, d, w, r, \epsilon, \delta, \beta)$, then there is an $(O(\epsilon \cdot \log_{w}(r_{\high}/r_{\low})), O(\delta \cdot \log_{w}(r_{\high}/r_{\low})))$-DP algorithm that, with probability at least $1-O(\beta \log_w(r_{\high}/r_{\low}))$ solves \onecluster$_{r_{\low}, r_{\high}}$ with approximation ratio $w^2$, additive error
$$\max_{i=0,1, \dots, \lfloor \log_{w}(r_{\high}/r_{\low}) \rfloor} t(n, d, w, r/w^i, \epsilon, \delta, \beta)+O\bigg(\frac{\log_w(r_{\high}/r_{\low}) \log(1/\beta)}{\epsilon}\bigg)$$
and running time
$$\max_{i=0,1, \dots, \lfloor \log_{w}(r_{\high}/r_{\low}) \rfloor} \tau(n, d, w, r/w^i, \epsilon, \delta, \beta) \cdot O(\log_{w}(r_{\high}/r_{\low})) +O(\log(1/\epsilon)).$$
\end{lemma}

The following theorem follows directly by combining Lemma~\ref{le:red_pure_densest_ball_1_cluster} (with $r_{\high} = 1$ and $r_{\low} = \kappa$) with our pure DP algorithm for \densestball from Theorem~\ref{thm:densest-ball-pure}.

\begin{theorem}\label{th:1_cluster_pure_DP_full}
For every $0 < \eps \leq O(1)$ and $0 < \alpha, \beta < 1$, there is an $\epsilon$-DP algorithm that runs in time $(nd)^{O_{\alpha}(1)} \poly\log(1/\kappa)$ and with probability at least $1-\beta$, solves \onecluster with approximation ratio $1+\alpha$ and additive error $O_{\alpha}\left(\frac{d}{\epsilon}\log\left(\frac{d}{\beta \kappa}\right)\right)$.
\end{theorem}

% \begin{lemma}[Approximate DP Reduction]\label{le:red_approx_densest_ball_1_cluster}
% Let $\epsilon, \beta > 0$ and $0 < \delta \le 1$. If there is an $(\epsilon, \delta)$-DP algorithm for \textsc{DensestBall} with approximation ratio $w$, additive error $t(n, d, w, \epsilon, \delta, \beta)$ and running time $\tau(n, d, w, \epsilon, \delta, \beta)$, then there is an $(???, ???)$-DP algorithm that, with probability at least $1-O(\beta \log_w(1/\kappa))$ solves \textsc{$1$-Cluster} with approximation ratio $w$, additive error $t(n, d, w, \epsilon, \delta, \beta)+???$ and running time $O(\tau(n, d, w, \epsilon, \beta) \cdot \log_{w}(1/\kappa))$.
% \end{lemma}
% The reduction in Lemmas~\ref{le:red_pure_densest_ball_1_cluster} and~\ref{le:red_approx_densest_ball_1_cluster} is the same (given in Algorithm~\ref{alg:reduction_densest_ball_1_cluster}); only the analysis differs, with Basic Composition (i.e., Theorem~\ref{thm:basic-composition}) being used in the former whereas Advanced Composition (i.e., Theorem~\ref{thm:advance-composition}) being used in the latter.

We now prove Lemma~\ref{le:red_pure_densest_ball_1_cluster}.

\begin{algorithm}[h!]
\caption{\onecluster from \densestball}\label{alg:reduction_densest_ball_1_cluster}
\begin{algorithmic}[1]
\Procedure{$1$-Cluster$(\bX)$ with parameters $\epsilon ,\delta \geq 0$, $\kappa, \beta > 0$, $w > 1$, $\lambda, r_{\low}, r_{\high} > 0$ and $0 < t' \le T$}{}
\State $r \leftarrow r_{\high}$.
\While{$r \geq r_{\low}$}
\State $c_1 \leftarrow$ center output by \densestball$^{\epsilon, \delta, \beta}(\bX; r)$\label{line:c_1}
%\If{$c_1 = \perp$}
%\State \Return $\perp$
%\EndIf
\State $s_1 \leftarrow$ $|\bX \cap \cB(c_1, r)| + \DLap(\lambda)$\label{line:s_1_setting}
\If{$s_1 \le T-t'$}\label{line:s_1_if}
\State \Return $\perp$
\EndIf
\State $c_2 \leftarrow$ center output by \densestball$^{\epsilon, \delta, \beta}(\bX; r/w)$\label{line:c_2}
%\If{$c_2 = \perp$}
%\State \Return $(c_1, w r)$\label{line:first_return}
%\EndIf
\State $s_2 \leftarrow$ $|\bX \cap \cB(c_2, r/w)| + \DLap(\lambda)$\label{line:s_2_setting}
\If{$s_2 \le T-t'$}\label{line:s_2_if}
\State \Return $(c_1, w r)$
\Else \State $r \leftarrow r/w$
\EndIf
\EndWhile
\State \Return $(c_1, r)$\label{line:second_return}
\EndProcedure
\end{algorithmic}
\end{algorithm}

\begin{proof}[Proof of Lemma~\ref{le:red_pure_densest_ball_1_cluster}]
We apply the reduction in Algorithm~\ref{alg:reduction_densest_ball_1_cluster} with $r_{\low}$, $r_{\high}$, $w$, and $T$ set the the values given in the statement of Lemma~\ref{le:red_pure_densest_ball_1_cluster}. We also set $\lambda = \frac{1}{\epsilon}$ and $t' = t(n, d, w, \epsilon, \delta, \beta)+O(\frac{\log_w(r_{\high}/r_{\low}) \log(1/\beta)}{\epsilon})$. We now analyze the properties of the resulting algorithm for \onecluster. On a high level, this algorithm performs differentially private binary search on the possible values of the ball's radius. In fact, in every iteration of the {\bf while} loop in Algorithm~\ref{alg:reduction_densest_ball_1_cluster}, we either return or decrease the radius $r$ by a factor of $w$. Thus, the total number of iterations executed is at most $\lfloor \log_w(r_{\high}/r_{\low}) \rfloor$.

\paragraph{Privacy.}
The DP property directly follows from the setting of $\lambda$, the privacy properties of the \densestball algorithm, and Basic Composition (i.e., Theorem~\ref{thm:basic-composition}).

\paragraph{Accuracy.} Denote $t := t(n, d, w, \epsilon, \delta, \beta)$. The standard tail bound for Discrete Laplace random variables implies that the probability that a $\DLap(\lambda)$ random variable has absolute value larger than some $\eta >0$ is at most $e^{-\Omega(\eta/\lambda)}$. By a union bound, we have that with probability at least $1-O(\beta \log_w(r_{\high}/r_{\low}))$, all the runs of \densestball succeed and each of the added $\DLap(\lambda)$ random variables has absolute value at most $O\bigg(\frac{\log_w(r_{\high}/r_{\low}) \log(1/\beta)}{\epsilon}\bigg)$ in Algorithm~\ref{alg:reduction_densest_ball_1_cluster}. We henceforth condition on this event. In this case, the following holds in each iteration of the {\bf while} loop:
\begin{itemize}
\item If there is a ball of radius $r$ that contains at least $T$ of the points in $\bX$, then %the center $c_1$ output in line~\ref{line:c_1} would not be equal to $\perp$ and 
the ball centered at $c_1$ output in line~\ref{line:c_1} and of radius $w r$ would contain at least $T-t$ points in $\bX$. Moreover, the setting of $s_1$ in line~\ref{line:s_1_setting} will not pass the {\bf if}  statement in line~\ref{line:s_1_if}.
\item If there is a ball of radius $r/w$ that contains at least $T$ of the points in $\bX$, then %the center $c_2$ output in line~\ref{line:c_2} would not be equal to $\perp$ and 
the ball centered at $c_2$ output in line~\ref{line:c_2} and of radius $r$ would contain at least $T-t$ points in $\bX$. Moreover, the setting of $s_2$ in line~\ref{line:s_2_setting} will not pass the {\bf if} statement in line~\ref{line:s_2_if}.
\end{itemize}
Put together, these properties imply that the radius output by Algorithm~\ref{alg:reduction_densest_ball_1_cluster} %in either line~\ref{line:first_return} or
line~\ref{line:second_return} is at most $w^2 \cdot r_{\opti}$ where $r_{\opti}$ is the radius of the smallest ball containing at least $T$ points in $S$. Moreover, the ball of the output radius around the output center is guaranteed to contain $T-t'$ points in $\bX$.

\paragraph{Running Time.}
The running time bound stated in Lemma~\ref{le:red_pure_densest_ball_1_cluster} directly follows from the bound on the number iterations and the facts that in each iteration at most $2$ calls to the \densestball algorithm are made (each with a radius parameter of the form $r/w^i$ for some $i = 0,1,\dots, \lfloor \log_{w}(r_{\high}/r_{\low}) \rfloor$), and that the running time for sampling a Discrete Laplace random variable with parameter $\lambda$ is $O(1+\log(\lambda))$ \cite{bringmann2013exact}.
\end{proof}

We next show that in the case of approximate DP, there is an algorithm with an additive error with better dependence on both the dimension $d$ and the discretization step $\kappa$ per dimension.

\begin{theorem}\label{th:one_cluster_approx_DP}
For every $\alpha, \epsilon, \delta, \beta > 0$, $\kappa \in (0,1)$ and positive integers $n$ and $d$, there is an $(\epsilon, \delta)$-DP algorithm that runs in time $(nd)^{O_{\alpha}(1)} \poly\log(1/\kappa)$ and solves the \onecluster problem with approximation ratio $1+\alpha$ and additive error $O_{\alpha}\left(\frac{\sqrt{d}}{\epsilon} \cdot \poly\log\left(\frac{nd}{\eps \delta \beta}\right)\right) + O\left(\frac{1}{\epsilon} \cdot \log(\frac{1}{\beta \delta}) \cdot 9^{\log^*(d/\kappa)}\right)$.
\end{theorem}

On a high level, the improved dependence of the dimension $d$ will follow from the use of our approximate DP algorithm for \densestball from Theorem~\ref{thm:densest-ball-apx} (instead of our pure DP algorithm for \densestball from Theorem~\ref{thm:densest-ball-pure}). On the other hand, the improved dependence of $\kappa$ will be obtained by applying the following algorithm of Nissim et al. \cite{NissimSV16}.

% To prove Theorem~\ref{th:one_cluster_approx_DP}, we will need the following approximate DP version of the Lemma~\ref{le:red_pure_densest_ball_1_cluster}.

% \begin{lemma}[Approximate DP Reduction from \textsc{$1$-Cluster}$_{r_{\low}, r_{\high}}$ to \textsc{DensestBall}]\label{le:red_approximate_densest_ball_1_cluster}
% Let $\epsilon, \delta > 0$. If there is an $(\epsilon, \delta)$-DP algorithm for \textsc{DensestBall} with approximation ratio $w$, additive error $t(n, d, w, r, \epsilon, \delta, \beta)$ and running time $\tau(n, d, w, r, \epsilon, \delta, \beta)$, then there is an $O(\epsilon \cdot \log_{w}(r_{\high}/r_{\low}))$-DP algorithm that, with probability at least $1-O(\beta \log_w(r_{\high}/r_{\low}))$ solves \textsc{$1$-Cluster} with approximation ratio $w$, additive error
% $$\max_{i=0,1, \dots, \lfloor \log_{w}(r_{\high}/r_{\low}) \rfloor} t(n, d, w, r/w^i, \epsilon, \delta, \beta)+O\bigg(\frac{\log_w(r_{\high}/r_{\low}) \log(1/\beta)}{\epsilon}\bigg)$$
% and running time
% $$\max_{i=0,1, \dots, \lfloor \log_{w}(r_{\high}/r_{\low}) \rfloor} \tau(n, d, w, r/w^i, \epsilon, \delta, \beta) \cdot O(\log_{w}(r_{\high}/r_{\low})) +O(\log(1/\epsilon)).$$
% \end{lemma}

\begin{theorem}[\cite{NissimSV16}]\label{thm:NSV_GoodRadius_cst_approx}
For every $\epsilon, \delta, \beta > 0$, $\kappa \in (0,1)$ and positive integers $n$ and $d$, there is an $(\epsilon, \delta)$-DP algorithm, \goodradius, that runs in time $poly(n, d, \log(1/\kappa))$ and solves the \onecluster problem with approximation ratio $w = 4$ and additive error $t = O\left(\frac{1}{\epsilon} \cdot \log(\frac{1}{\beta \delta}) \cdot 9^{\log^*(d/\kappa)}\right) $.
\end{theorem}

We are now ready to prove Theorem~\ref{th:one_cluster_approx_DP}.

\begin{proof}[Proof of Theorem~\ref{th:one_cluster_approx_DP}]
We proceed by first running the \goodradius algorithm from Theorem~\ref{thm:NSV_GoodRadius_cst_approx} to get a radius $r_{\approxi}$. If $r_{\approxi} = 0$, we run our approximate DP algorithm for \densestball from Theorem~\ref{thm:densest-ball-apx} with $r=0$, round the resulting center to the closest point in $\mathbb{B}_{\kappa}^d$, which we then output along with a radius of $0$. Otherwise, we apply Lemma~\ref{le:red_pure_densest_ball_1_cluster} with $r_{\low} = r_{\approxi}/4$ and $r_{\high} = r_{\approxi}$ and with our approximate DP algorithm for \densestball from Theorem~\ref{thm:densest-ball-apx}.

The privacy of the combined algorithm can be guaranteed by dividing the $(\epsilon, \delta)$-DP budget, e.g., equally among the call to \goodradius and that to Lemma~\ref{le:red_pure_densest_ball_1_cluster} (and ultimately to Theorem~\ref{thm:densest-ball-apx}), and applying Basic Composition (i.e., Theorem~\ref{thm:basic-composition}).

The accuracy follows from the approximation ratio and additive error guarantees of Theorem~\ref{thm:NSV_GoodRadius_cst_approx}, Lemma~\ref{le:red_pure_densest_ball_1_cluster} and Theorem~\ref{thm:densest-ball-apx}, and by dividing the failure probability $\beta$, e.g., equally among the two algorithms, and then applying the union bound.

The running time is simply the sum of the running times of the two procedures, and can thus be directly bounded using the running time bounds in Theorem~\ref{thm:NSV_GoodRadius_cst_approx}, Lemma~\ref{le:red_pure_densest_ball_1_cluster} and Theorem~\ref{thm:densest-ball-apx}.
\end{proof}

\section{Sample and Aggregate}\label{sec:sample_and_aggregate}
This section is devoted to establishing Theorem~\ref{thm:sample_and_agg_ptas_app}. As mentioned in Section~\ref{sec:app_sample_and_aggregate}, one of the basic techniques in DP is the Sample and Aggregate framework of \cite{nissim2007smooth}. Consider a universe $\cal U$ and functions $f: {\cal U}^* \to \mathbb{B}_{\kappa}^d$ mapping databases to points in $\mathbb{B}_{\kappa}^d$. Intuitively, the premise of the Sample and Aggregate framework is that, for sufficiently large databases $S \in {\cal U}^*$, evaluating the function $f$ on a random subsample of $S$ can yield a good approximation to the point $f(S)$. The following definition quantifies how good such approximations are.

\begin{definition}[\cite{NissimSV16}]
Let $\kappa \in (0,1)$. Consider a function $f: {\cal U}^* \to \mathbb{B}_{\kappa}^d$ and a database $S \in U^*$. A point $c \in \mathbb{B}_{\kappa}^d$ is said to be an \emph{$(m,r,\zeta)$-stable point} of $f$ on $S$ if for $S'$ a database consisting of $m$ i.i.d. samples $S$, it holds that $\Pr[\|f(S')-c\|_2 \le r] \geq \zeta$. If such a point $c$ exists, the function $f$ is said to be \emph{$(m,r,\zeta)$-stable} on $S$, and $r$ is said to be a radius of the stable point $c$.
\end{definition}
Nissim et al.~\cite{NissimSV16} obtained the following DP reduction from the problem of finding a stable point of small radius to \textsf{$1$-Cluster}.
\begin{lemma}[\cite{NissimSV16}]\label{le:red_1_cluster_sample_and_agg}
Let $d$ and $n \geq m$ be positive integers, and $\epsilon > 0$ and $0 < \zeta, \beta, \delta < 1$ be real numbers satisfying $\epsilon \le \zeta/72$ and $\delta \le \frac{\beta \epsilon}{3}$. If there is an $(\epsilon, \delta)$-DP algorithm for \textsf{$1$-Cluster} on $k$ points in $d$ dimensions with approximation ratio $w$, additive error $t$, error probability $\beta/3$, and running time $\tau(k, d, w, \epsilon, \delta, \beta/3)$, then there is an $(\epsilon, \delta)$-DP algorithm that takes as input a function $f: {\cal U}^* \to \mathbb{B}_{\kappa}^d$ along with the parameters $m$, $\zeta$, $\epsilon$, and $\delta$, runs in time $\tau(n/(9m), d, w, \epsilon, \delta, \beta/3)$ plus $O(n/m)$ times the running time for evaluating $f$ on a dataset of size $m$, and whenever $f$ is $(m, r, \zeta)$-stable on $S$, with probability $1-\beta$, the algorithm outputs an $(m, wr, \frac{\zeta}{8})$-stable point of $f$ on $S$, provided that $n \geq m \cdot O\bigg( \frac{t}{\zeta} + \frac{1}{\zeta^2} \log\bigg(\frac{12}{\beta}\bigg)\bigg)$.
\end{lemma}

By combining Lemma~\ref{le:red_1_cluster_sample_and_agg} and our Theorem~\ref{th:one_cluster_approx_DP}, we obtain the following algorithm.

\begin{theorem}\label{thm:sample_and_agg_ptas}
Let $d$ and $n \geq m$ be positive integers, and $\epsilon > 0$ and $0 < \zeta, \alpha, \beta, \delta, \kappa < 1$ be real numbers satisfying $\epsilon \le \zeta/72$ and $\delta \le \frac{\beta \epsilon}{3}$. There is an $(\epsilon, \delta)$-DP algorithm that takes as input a function $f: {\cal U}^* \to \mathbb{B}_{\kappa}^d$ as well as the parameters $m$, $\zeta$, $\epsilon$ and $\delta$, runs in time $(nd/m)^{O_{\alpha}(1)} \poly\log(1/\kappa)$ plus $O(n/m)$ times the running time for evaluating $f$ on a dataset of size $m$, and whenever $f$ is $(m, r, \zeta)$-stable on $S$, with probability $1-\beta$, the algorithm outputs an $(m, (1+\alpha) r, \frac{\zeta}{8})$-stable point of $f$ on $S$, provided that $n \geq m \cdot O_{\alpha}\left(\frac{\sqrt{d}}{\epsilon} \cdot \poly\log\left(\frac{nd}{\eps \delta \beta}\right) + \frac{1}{\epsilon} \cdot \log(\frac{1}{\beta \delta}) \cdot 9^{\log^*(d/\kappa)}\right)$.
\end{theorem}

We point out that our Theorem~\ref{thm:sample_and_agg_ptas} obtains a $1+\alpha$ approximation to the radius (where $\alpha$ is an arbitrarily small positive constant) whereas \cite{NissimSV16} obtained an approximation ratio of $O(\sqrt{\log{n}})$, the prior work of \cite{nissim2007smooth} had obtained an approximation ratio of $O(\sqrt{d})$, and a constant factor is subsequently implied by \cite{NissimS18}.

\section{Agnostic Learning of Halfspaces with a Margin}\label{sec:agnostic_halfspaces_margin}
In this section, we prove Theorem~\ref{thm:DP_agn_learn_halfspaces_app}. 
%Using our algorithm for \densestball, we obtain the first efficient DP algorithm for agnostic learning of halfspaces with a margin. 
We start with some definitions.
\paragraph{Halfspaces.} Let $\sgn(x)$ be equal to $+1$ if $x \geq 0$, and to $-1$ otherwise. A \emph{halfspace} (aka \emph{hyperplane} or \emph{linear threshold function}) is a function $h_{u, \theta}(x) =  \sgn(u \cdot x - \theta)$ where $u \in \mathbb{R}^d$ and $\theta \in \mathbb{R}$, and where $u \cdot x = \langle u, x \rangle$ denotes the dot product of the vectors $u$ and $x$. Without loss of generality, we henceforth focus on the case where $\theta = 0$.\footnote{As a non-homogeneous halfspace (i.e., one with $\theta \neq 0$) can always be thought of as a homogeneous halfspace (i.e., with $\theta = 0$) with an additional coordinate whose value is $\theta$.} A halfspace $h_{u}$ correctly classifies the labeled point $(x, y) \in \mathbb{R}^d \times \{\pm 1\}$ if $h_{u}(x) = y$.

\paragraph{Margins.} The \emph{margin} of a point $x$ with respect to a hypothesis $h$ is defined as the largest distance $r$ such that any point of $x$ at distance $r$ is classified in the same class as $x$ by hypothesis $h$. In the special case of a halfspace $h_u(x) = \sgn(u \cdot x)$, the margin of point $x$ is equal to $\frac{|\langle u, x \rangle|}{\|u\| \cdot \|x\|}$.

\paragraph{Error rates.} For a distribution $D$ on $\mathbb{R}^d \times \{\pm 1 \}$,
\begin{itemize}
\item the \emph{error rate} of a halfspace $h_u$ on $D$ is defined as $\err^D(u) := \Pr_{(x,y) \sim D}[h(x) \neq y]$,
\item for any $\mu > 0$, the \emph{$\mu$-margin error rate} of a halfspace $h_u$ on $D$ is defined as
$$\err^D_{\mu}(u) := \Pr_{(x,y) \sim D}\left[y \frac{\langle u , x \rangle}{\|u\| \cdot \|x\|} \le \mu\right].$$
\end{itemize}

Furthermore, let $\opt^D_{\mu} := \min_{u \in \mathbb{R}^d} \err^D_{\mu}(u)$. For the ease of notation, we may write $\err^S(u)$ where $S \subseteq \R^d \times \{\pm 1\}$ to denote the error rate on the uniform distribution of $S$; $\err^S_{\mu}(u)$ is defined similarly.

We study the problem of learning halfspaces with a margin in the agnostic PAC model~\cite{haussler1992decision,kearns1994toward}, as stated below.

\begin{definition}[Proper Agnostic PAC Learning of Halfspaces with Margin]
Let $d \in \mathbb{N}$, $\beta \in (0,1)$, and $\mu, t \in \mathbb{R}^{+}$. An algorithm \emph{properly agnostically PAC learns halfspaces} with margin $\mu$, error $t$, failure probability $\beta$ and sample complexity $m$, if given as input a training set $S = \{(x^{(i)}, y^{(i)})\}_{i=1}^m$ of i.i.d. samples drawn from an unknown distribution $D$ on $\cB(0, 1) \times \{\pm 1\}$, it outputs a halfspace $h_u:\mathbb{R}^d \to \{\pm 1\}$ satisfying $\err^D(u) \le \opt^D_{\mu} + t$ with probability $1-\beta$.
\end{definition}

When not explicitly stated, we assume that $\beta = 0.01$, it is simple to decrease this failure probability by running the algorithm $\log(1/\beta)$ times and picking the best.

\paragraph{Related Work.} In the non-private setting, the problem has a long history~\cite{Ben-DavidS00,bartlett2002rademacher,mcallester2003simplified,SSS09,BirnbaumS12,diakonikolas2019nearly,DKM20}; in fact, the perceptron algorithm~\cite{Rosenblatt:58} is known to PAC learns halfspaces with margin $\mu$ in the realizable case (where $\opt^D_{\mu} = 0$) with sample complexity $O_{t}(1/\gamma^2)$~\cite{novikoff62convergence}. In the agnostic setting (where $\opt^D_{\mu}$ might not be zero), Ben-David and Simon~\cite{Ben-DavidS00} gave an algorithm that uses $O\left(\frac{1}{t^2 \gamma^2}\right)$ samples and runs in time $\poly(d) \cdot (1/t)^{O(1/\gamma^2)}$. This is in contrast with the perceptron algorithm, which runs in $\poly\left(d/t\right)$ time. It turns out that this is not a coincidence: the agnostic setting is NP-hard even for constant $t > 0$~\cite{Ben-DavidEL03,Ben-DavidS00}. Subsequent works~\cite{SSS09,diakonikolas2019nearly,DKM20} managed to improve this running time, albeit at certain costs. For example, the algorithm in~\cite{SSS09} is \emph{improper}, meaning that it may output a hypothesis that is not a halfspace, and those in~\cite{diakonikolas2019nearly,DKM20} only guarantee that $\err^D(h_u) \le (1 + \eta) \cdot \opt^D_{\mu} + t$ for an arbritrarily small constant $\eta > 0$.

Nguyen et al.~\cite{le2020efficient} were the first to study the problem of learning halfspaces with a margin in conjunction with differential privacy. In the realizable setting, they give an $\eps$-DP (resp. $(\eps, \delta)$-DP) algorithm with running time $(1/t)^{O(1/\gamma^2)} \cdot \poly\left(\frac{d \log(
1/\delta)}{\eps t}\right)$ (resp. $\poly\left(\frac{d \log(1/\delta)}{\eps t}\right)$) and sample complexity $O\left(\poly\left(\frac{1}{\epsilon t \gamma}\right) \cdot \poly\log\left(\frac{1}{\epsilon t \gamma}\right) \right)$ (resp. $O\left(\poly\left(\frac{1}{\epsilon t \gamma}\right) \cdot \poly\log\left(\frac{1}{\epsilon t \delta \gamma}\right) \right)$). Due to the aforementioned NP-hardness of the problem, their efficient $(\eps, \delta)$-DP algorithm cannot be extended to the agnostic setting. On the other hand, while not explicitly analyzed in the paper, their $\eps$-DP algorithm also works in the agnostic setting with similar running time and sample complexity. 

Here, we provide an alternative proof of the agnostic learning result, as stated below. This will be shown via our \densestball algorithm together with a known connection between \densestball and learning halfspaces with a margin~\cite{Ben-DavidS00, Ben-DavidES02}.

\begin{theorem}\label{thm:DP_agn_learn_halfspaces_full}
For every $0 < \eps \leq O(1)$ and $0 < \beta, \mu, t < 1$, there is an $\eps$-DP algorithm that runs in time $\left(\frac{\log(1/\beta)}{\epsilon t}\right)^{O_{\mu}(1)} + \poly\left(O_{\mu}\left(\frac{d}{\epsilon t}\right)\right)$, and properly agnostically PAC learns halfspaces with margin $\mu$, error $t$, failure probability $\beta$ and sample complexity $O_{\mu} \left(\frac{1}{\eps t^2}\cdot\poly\log\left(\frac{1}{\eps \beta t}\right)\right)$.
\end{theorem}

To prove Theorem~\ref{thm:DP_agn_learn_halfspaces_full}, we will use the following reduction\footnote{This reduction is implicit in Claim 2.6 and Lemma 4.1 of \cite{Ben-DavidES02}.}:

\begin{lemma}[\cite{Ben-DavidS00, Ben-DavidES02}]\label{le:reduction_hyperplane_ball}
Let $\mu \in (0,1)$ and $\alpha, t >0$ such that $1+\alpha < 1/\sqrt{1-\mu^2}$. There is a polynomial-time transformation that, given as input a set $S = \{(x^{(i)}, y^{(i)})\}_{i=1}^m$ of labeled points, separately transforms each $(x^{(i)}, y^{(i)})$ into a point $z^{(i)}$ in the unit ball such that a solution to \densestball on the set $\{z^{(i)}\}_{i=1}^m$ with radius $\sqrt{1-\mu^2}$, approximation ratio $1+\alpha$ and additive error $t$ yields a halfspace with $\mu'$-margin error rate on $S$ at most $\opt_{\mu}^{S} + \frac{t}{m}$ where $\mu' = \sqrt{1-(1-\mu^2)(1+\alpha)^2}$.
\end{lemma}

%The proof of Lemma~\ref{le:reduction_hyperplane_ball} follows along the same lines as the proofs of Claim 2.6 and Lemma 4.1 of \cite{Ben-DavidES02}.
By combining Lemma~\ref{le:reduction_hyperplane_ball} and our Theorem~\ref{thm:densest-ball-pure}, we immediately obtain the following:
\begin{lemma}\label{lem:best_sep_halfspace_pure_DP}
For every $\eps, \beta > 0$ and $0 < \mu < 1$, there exists an $\eps$-DP algorithm that runs in time $(md)^{O_{\mu}(1)}$, takes as input a set $S = \{(x^{(i)}, y^{(i)})\}_{i=1}^m$ of labeled points, and with probability $1 - \beta$, outputs a halfspace with $\mu'$-margin error rate on $S$ at most $\opt_{\mu}^{S} + \frac{t}{m}$ where $\mu' = \sqrt{1-(1-\mu^2)(1+\alpha)^2}$ and $t = O_{\alpha}\left(\frac{d}{\epsilon}\cdot\log\left(\frac{d}{\beta}\right)\right)$.
\end{lemma}

% By combining Lemma~\ref{le:reduction_hyperplane_ball} and our Theorem~\ref{thm:apx-1-cluster-given-radius}, we can similarly get the following approximate DP algorithm.

% \begin{theorem}\label{thm:best_sep_halfspace_approx_DP}
% For every $\eps, \beta > 0$ and $0 < \delta, \mu  < 1$, there exists an $(\eps, \delta)$-DP algorithm that runs in time $(nd)^{O_{\mu}(1)} \poly\log(\frac{1}{1-\mu})$, and with probability at least $1 - \beta$, solves the Best Separating Halfspace problem with margin $\mu$ and additive error $O_{\mu} \left(\frac{\sqrt{d}}{\eps}\cdot\poly\log\left(\frac{n d}{\eps \delta \beta}\right)\right)$.
% \end{theorem}

As is usual in PAC learning results, we will need a generalization bound:

\begin{lemma}[Generalization Bound for Halfspaces with Margin \cite{bartlett2002rademacher,mcallester2003simplified}]\label{le:gen_bd_halfspaces_margin}
Let $S = \{(x^{(i)}, y^{(i)})\}_{i=1}^m$ be a multiset of i.i.d. samples from a distribution $D$ on $\mathbb{R}^d \times \{\pm 1\}$, where $m = \Omega(\log(1/\beta)/(t^2 \mu^2))$. Then, with probability $1-\beta$ over $S$, for all vectors $u \in \mathbb{R}^d$, it holds that $\err^D(u) \le \err^{\mathbb{U}(S)}_{\mu}(u) + t$.
\end{lemma}

The above lemmas do not yet imply Theorem~\ref{thm:DP_agn_learn_halfspaces_full}; applying them directly will lead to a sample complexity that depends on $d$. To prove Theorem~\ref{thm:DP_agn_learn_halfspaces_full}, we will also need the following dimensionality-reduction lemma from~\cite{le2020efficient} which allows us to focus on the low-dimensional case.

\begin{lemma}[Properties of JL Lemma~\cite{le2020efficient}]\label{le:JL_properties}
Let $A \in \mathbb{R}^{d' \times d}$ be a random matrix such that $d' = \Theta\left(\frac{\log(1/\beta_{JL})}{\mu^2}\right)$ and $A_{i,j} =  \begin{cases}
               +\frac{1}{\sqrt{d'}}               & \text{w.p. } \frac{1}{2}\\
               -\frac{1}{\sqrt{d'}} &\text{w.p. } \frac{1}{2}
           \end{cases}$ independently over $(i,j)$.\\ 
           Let $u \in \mathbb{R}^d$ be a fixed vector. Then, for any $(x, y) \in \R^d \times \{\pm 1\}$ such that $y \cdot \frac{\langle u , x \rangle}{\|u\| \cdot \|x \|} \geq \mu$, we have
           \begin{align*}
           \Pr_A\left[y \cdot \frac{\langle Au , Ax \rangle}{\|Au\| \cdot \|Ax \|}  > 0.9\mu\right] \geq 1 - 4\beta_{JL}.
           \end{align*}
\end{lemma}

\begin{proof}[Proof of Theorem~\ref{thm:DP_agn_learn_halfspaces_full}]
Our algorithm works as follows. We first draw a set $S$ of $m$ training samples, and, then apply the JL lemma (with a matrix $A$ sampled as in Lemma~\ref{le:JL_properties}) in order to project to $d'$ dimensions, where $m, d'$ are to be specified below. Let $S_A$ be the projected training set (i.e., $S_A$ is the multiset of all pairs $(Ax, y)$ where $(x,y) \in S$). We then use the algorithm from Lemma~\ref{lem:best_sep_halfspace_pure_DP} with $\alpha = 0.01\mu^2$ to obtain a halfspace $u' \in \R^{d'}$. Finally, we output $A^T u'$.

We will now prove the algorithm's correctness. Consider any $u^* \in \arg\min_{u \in \mathbb{R}^d} \err^D_{\mu}(u)$. Let $D'$ denote the distribution of $(x, y) \sim D$ conditioned on $(x, y)$ being correctly classified by $u^*$ with margin at least $\mu$. (Note that $\err^{D'}_\mu(u) = 0$.) Furthermore, let $D_A$ denote the distribution of $(Ax, y)$ where $(x, y) \sim D$, and $D'_A$ denote the distribution of $(Ax, y)$ where $(x, y) \sim D'$.

Let $\beta_{JL} = 0.01 t \beta$ and $d' = \Theta\left(\frac{\log(1/\beta_{JL})}{\mu^2}\right) = \Theta\left(\frac{\log(1/(t \beta))}{\mu^2}\right)$ be as in Lemma~\ref{le:JL_properties}, which implies that $\E_A[\err^{D'_A}_{0.9\mu}(Au^*)] \leq 0.04t \beta$. Hence, by Markov's inequality, we have $\Pr_A[\err^{D'_A}_{0.9\mu}(Au^*) > 0.2t] \leq 0.2\beta$. Combining this with the definitions of $u^*$ and $D'$, we have
\begin{align}\label{ineq:jl_and_markov}
\Pr_A\left[\err^{D_A}_{0.9 \mu}(Au^*) > \opt^D_{\mu} + 0.2t\right] \leq 0.2\beta.
\end{align}
When $m \geq \Omega(\log(1/\beta) / (t^2\mu^2))$, the Chernoff bound implies that
\begin{equation}\label{ineq:chernoff_emp_pop}
\Pr_S\left[\err^{S_A}_{0.9\mu}(Au^*) > \err^{D_A}_{0.9\mu}(Au^*) + 0.2t\right] \leq 0.2\beta.
\end{equation}
Combining~\eqref{ineq:jl_and_markov} and~\eqref{ineq:chernoff_emp_pop}, we have
\begin{equation}\label{ineq:opt-halfspace-emp-err}
\Pr_{A, S}\left[\err^{S_A}_{0.9 \mu}(Au^*) \leq \opt^D_{\mu} + 0.4t\right] \geq 1 - 0.4\beta.
\end{equation}
Lemma~\ref{lem:best_sep_halfspace_pure_DP} then ensures that, with probability $1 - 0.2\beta$, we obtain a halfspace $u' \in \R^{d'}$ satisfying
     \begin{equation}\label{ineq:algo_best_sep_hyp}
         \err^{S_A}_{0.5\mu}(u') \le \err^{S_A}_{0.9\mu}(Au^*) + t',
     \end{equation}
     where $t' = O_{\mu}\left(\frac{d'}{\epsilon m}\cdot\log\left(\frac{d'}{\beta}\right)\right)$. When we select $m = \Theta_{\mu}\left(\frac{d'}{\epsilon t}\cdot\log\left(\frac{d'}{\beta}\right)\right) = \Theta_{\mu}\left(\frac{1}{\eps t^2}\cdot\poly\log\left(\frac{1}{\eps \beta t}\right)\right)$, we have $t' \leq 0.1t$.

Next, we may apply the generalization bound from Lemma~\ref{le:gen_bd_halfspaces_margin}, which implies that
\begin{equation} \label{ineq:gen_found_u}
\Pr_S[\err^{D_A}(u') \leq \err^{S_A}_{0.5\mu}(u') + 0.1t] \geq 1 - 0.2\beta.
\end{equation}
Using the union bound over~\eqref{ineq:opt-halfspace-emp-err},~\eqref{ineq:algo_best_sep_hyp} and~\eqref{ineq:gen_found_u}, the following holds with probability at least $1 - \beta$:
\begin{align*}
\err^{D}(A^Tu') = \err^{D_A}(u') \leq \opt^D_{\mu} + t,
\end{align*}
which concludes the correctness proof. The claimed running time follows from Lemma~\ref{lem:best_sep_halfspace_pure_DP}.
\end{proof}

\section{\cp} \label{sec:closest-pair}
In this section, we give our history-independent data structure for \cp (Theorem~\ref{thm:dynamic-cp}). Before we do so, let us briefly discuss related previous work.

\paragraph{Related Work.} \cp is among the first problems studied in computational geometry~\cite{ShamosH75,BentleyS76,Rabin76} and there have been numerous works on lower and upper bounds for the problem since then. Dynamic \cp has also long been studied~\cite{Salowe91,smid1992maintaining,LenhofS92,KapoorS96,Bespamyatnikh98}. To the best of our knowledge, each of these data structures is either history-dependent or has update time $2^{\omega(d)} \cdot \poly\log n$. We will not discuss these results in detail. As alluded to in the main body of the paper, the best known history-independent data structure in the ``small dimension'' regime is that of Aaronson et al.~\cite{Aaronson-cp} whose running time is $d^{O(d)} \poly\log n$. Our result improves the running time to $2^{O(d)} \poly\log n$. We also remark that, due to a result of~\cite{SM19}, the update time cannot\footnote{Specifically,~\cite{SM19} shows, assuming SETH, that (offline) \cp cannot be solved in $O(n^{1.499})$ time even for $d = O(\log n)$. If one had a data structure for dynamic \cp with update time $2^{o(d)}\poly\log n$, then one would be able to solve (offline) \cp in $n \cdot 2^{o(d)}\poly\log n = n^{1 + o(1)}$ time for $d = O(\log n)$.} be improved to $2^{o(d)} \poly\log n$ assuming the strong exponential time hypothesis (SETH); in other words, our update time is essentially the best possible.

We finally note that, in the literature, \cp is sometimes referred to the optimization variant, in which we wish to determine $\min_{1 \leq i < j \leq n} \|x_i - x_j\|_2^2$. In the offline setting, the two versions have the same running time complexity to within a factor of $\poly(L)$ (both in the quantum and classical settings) because, to solve the optimization variant, we may use binary search on $\xi$ and apply the algorithm for the decision variant. However, our dynamic data structure (Section~\ref{subsec:online-cp}) does not naturally extend to the optimization variant and it remains an interesting open question to extend the algorithm to this case. 

\subsection{History-Independent Dynamic Data Structure}
\label{subsec:online-cp}

As stated in the proof overview, we will use a history-independent data structure for maintaining a map $M: \{0, 1\}^{\ell_k} \to \{0, 1\}^{\ell_v}$, where $\ell_k, \ell_v$ are positive integers. In this setting, the map starts of as the trivial map $k \mapsto 0\dots0$. Each update is of the form: set $M[k]$ to $v$, for some $k \in \{0, 1\}^{\ell_k}, v \in \{0, 1\}^{\ell_v}$. The data structure should support a lookup of $M[k]$ for a given $k$.

Similarly to before, we say that a randomized data structure is history-independent if, for any two sequences of updates that result in the same map, the distributions of the states are the same.

Ambainis~\cite{Ambainis07} gives a history-independent data structure for maintaining a map, based on skip lists. However, this results in probabilistic guarantees on running time. As a result, we will use a different data structure due to~\cite{BernsteinJLM13} based on radix trees, which has a deterministic guarantee on the running time. (See also~\cite{Stacey2014} for a more detailed description of the data structure.)

\begin{theorem}\cite{BernsteinJLM13} \label{thm:dynamic-map}
Let $\ell_k, \ell_v$ be positive integers. There is a history-independent data structure for maintaining a map $M: \{0, 1\}^{\ell_k} \to \{0, 1\}^{\ell_v}$ for up to $n$ updates, such that each update and lookup takes $\poly(\ell_k, \ell_v)$ time and the required memory is $O(n \cdot \poly(\ell_k, \ell_v))$. 
\end{theorem}

With the above in mind, we are now ready to prove our main result of this section.

\begin{proof}[Proof of Theorem~\ref{thm:dynamic-cp}]
Let $C := C_{0.5\sqrt{\xi}} \subseteq \R^d$ be the lattice cover from Lemma~\ref{lem:cover-main} with $\Delta = 0.5\sqrt{\xi}$. It follows from the construction of Micciancio~\cite{Micciancio04} that every point $c \in C$ satisfies $\frac{3^{d + 1}}{\sqrt{\xi}} c \in \Z^n$ (i.e., every coordinate of $c$ is an integer multiple of $\frac{\sqrt{\xi}}{3^{d + 1}}$). As a result, we have that every point $c \in C^* := C \cap \cB(0, 10\sqrt{d 2^L})$ can be represented as an $\ell_k = \poly(L, d)$ bit integer.

Our data structure maintains a triple $p^{\text{total}}_{\leq \xi}, q^{\text{marked-cell}}$ and $\cH$, where $p^{\text{total}}_{\leq \xi}, q^{\text{marked-cell}}$ are integers between 0 and $n$ (inclusive) and $\cH$ is the data structure from Theorem~\ref{thm:dynamic-map} for maintaining a map $M$ with $\ell_k$ as above and $\ell_v = 2\lceil \log n \rceil + dL$. Each key of $M$ is thought of as an encoding of a point $c$ in the cover $C^*$. Furthermore, each value is a triplet $(n_{count}, p_{\leq \xi}, x_{\oplus})$ where $n_{count}$ is an integer between $0$ and $n$ (inclusive), $p_{\leq \xi}$ is an integer between $0$ and $n$ (inclusive), and $x_{\oplus}$ is a $dL$-bit string. 

Let $\psi: (\Z \cap [0, 2^L])^d \to C$ denote the mapping from $x$ to $\argmin_{c \in C} \|x - c\|_2$ where ties are broken arbitrarily, and let $\cV_c := \psi^{-1}(c)$ denote the Voronoi cell of $c$ (with respect to $C$). Observe that $\psi$ can be computed in time $2^{O(d)} \cdot \poly(L)$ using the CVP algorithm from Theorem~\ref{thm:cvp}. Furthermore, since $C$ is a $0.5\sqrt{\xi}$ cover, we have that $\|\psi(x) - x\|_2 \leq 0.5\sqrt{\xi}$, which implies that $\psi(x) \in C^*$.

For a set $S$ of input points and $c \in C^*$, if $|\cV_c \cap S| = 1$, we use $x(c, S)$ to denote the unique element of $\cV_c \cap S$. When $S$ is clear from the context, we simply write $x(c)$ as a shorthand for $x(c, S)$.

We will maintain the following invariants for the entire run of the algorithm (where $S$ is the current set of points): 
\begin{itemize}
\item First, for all $c \in C^*$, $M[c] = (n_{count}, p_{\leq \xi}, x_{\oplus})$ where the values of $n_{count}, p_{\leq \xi}, x_{\oplus}$ are as follows:
\begin{itemize}
\item $n_{count} = |\cV_c \cap S|$,
\item $x_{\oplus} = \bigoplus_{x \in \cV_c \cap S} x$ where each $x \in \cV_c \cap S$ is thought of as a $dL$-bit string resulting from concatenating each bit representation of the coordinate,
\item $p_{\leq \xi}$ depends on whether $|\cV_c \cap S| = 1$. If $|\cV_c \cap S| \ne 1$, $p_{\leq \xi} = 0$. Otherwise, i.e., if $|\cV_c \cap S| = 1$, then $p_{\leq \xi} = |\{c' \in C \setminus \{c\} \mid |\cV_{c'} \cap S| = 1, \|x(c) - x(c')\|_2^2 \leq \xi\}|$, i.e., the number of other cells $c'$ with unique input point $x(c')$ such that $x(c)$ and $x(c')$ are within $\sqrt{\xi}$ in Euclidean distance.
\end{itemize}
\item $q^{\text{marked-cell}}$ is equal to $|\{c \in C^* \mid |\cV_c \cap S| \geq 2\}|$.
\item $p^{\text{total}}_{\leq \xi}$ is equal to $|\{c, c' \in C^* \mid c \ne c', |\cV_c \cap S| = |\cV_{c'} \cap S| = 1, \|x(c) - x(c')\|_2^2 \leq \xi\}|$, i.e., the number of pairs of cells with unique input points such that the corresponding pair of input points are within $\sqrt{\xi}$ in Euclidean distance.
\end{itemize}

We now describe the operations on the data structure. Throughout, we use the following notation:
\begin{align*}
\Lambda((n_{count}, p_{\leq \xi}, x_{\oplus}), (n'_{count}, p'_{\leq \xi}, x'_{\oplus})) :=
\begin{cases}
1 & \text{ if } n_{count} = n'_{count} = 1 \text{ and } \|x_{\oplus} - x'_{\oplus}\|_2^2 \leq \xi, \\
0 & \text{otherwise.}
\end{cases}
\end{align*}
Note that, when these two states correspond to cells $c$ and $c'$, this is the contribution of $c, c'$ to $p^{\text{total}}_{\leq \xi}$. Notice also that $\Lambda$ does not depend on $p_{\leq \xi}$ and $p'_{\leq \xi}$, but we leave them in the expression for simplicity. %, starting with an auxiliary operation that will be used as a subroutine. %Here and throughout, we use ``time'' to denote the number of arithematic operations performed on $O(dL \log n)$ bits. Since we are allowed additional 

%\paragraph{Auxiliary Operation.} 

\paragraph{Lookup.} To determine whether the current point set $S$ contains two distinct points that are at most $\sqrt{\xi}$ apart, we simply check whether $q^{\text{marked-cell}} \geq 1$ or $p^{\text{total}}_{\leq \xi} \geq 1$.

\paragraph{Insert.} To insert a point $x$ into the data structure, we perform the following:
\begin{enumerate}
\item Use the algorithm for Closest Vector Problem (Theorem~\ref{thm:cvp}) to compute $c = \psi(x)$.
\item Let $(n^{old}_{count}, p^{old}_{\leq \xi}, x^{old}_{\oplus}) = M[c]$.
\item Let $n^{new}_{count} = n^{old}_{count} + 1, p^{new}_{\leq \xi} = 0$ and $x^{new}_{\oplus} = x^{old}_{\oplus} \oplus x$.
\item Using the list-decoding algorithm (from Lemma~\ref{lem:cover-main}), compute the set $C_{\text{close}}$ of all $c' \in C$ within distance $2\sqrt{\xi}$ of $c$. Then, for each $c' \in C_{\text{close}}$, do the following:
\begin{enumerate}
\item Compute $\Lambda^{old} = \Lambda(M[c'], (n^{old}_{count}, p^{old}_{\leq \xi}, x^{old}_{\oplus}))$.
\item Compute $\Lambda^{new} = \Lambda(M[c'], (n^{new}_{count}, p^{new}_{\leq \xi}, x^{new}_{\oplus}))$.
\item If $\Lambda^{old} - \Lambda^{new} \ne 0$, increase $p_{\leq \xi}$ of $M[c']$ by $\Lambda^{old} - \Lambda^{new}$.
\item Increase $p_{\xi}^{new}$ by $\Lambda^{new}$.
\end{enumerate}
\item Update $M[c]$ to $(n^{new}_{count}, p^{new}_{\leq \xi}, x^{new}_{\oplus})$
\item If $n^{new}_{count} = 2$, increase $q^{\text{marked-cell}}$ by one.
\end{enumerate}

\paragraph{Delete.} To remove a point $x$ from the data structure, we perform the following:
\begin{enumerate}
\item Use the algorithm for the Closest Vector Problem (Theorem~\ref{thm:cvp}) to compute $c = \psi(x)$.
\item Let $(n^{old}_{count}, p^{old}_{\leq \xi}, x^{old}_{\oplus}) = M[c]$.
\item Let $n^{new}_{count} = n^{old}_{count} - 1, p^{new}_{\leq \xi} = 0$ and $x^{new}_{\oplus} = x^{old}_{\oplus} \oplus x$.
\item Using the list-decoding algorithm (from Lemma~\ref{lem:cover-main}), compute the set $C_{\text{close}}$ of all $c' \in C$ within distance $2\sqrt{\xi}$ of $c$. Then, for each $c' \in C_{\text{close}}$, do the following:
\begin{enumerate}
\item Compute $\Lambda^{old} = \Lambda(M[c'], (n^{old}_{count}, p^{old}_{\leq \xi}, x^{old}_{\oplus}))$.
\item Compute $\Lambda^{new} = \Lambda(M[c'], (n^{new}_{count}, p^{new}_{\leq \xi}, x^{new}_{\oplus}))$.
\item If $\Lambda^{old} - \Lambda^{new} \ne 0$, increase $p_{\leq \xi}$ of $M[c']$ by $\Lambda^{old} - \Lambda^{new}$.
\item Increase $p_{\xi}^{new}$ by $\Lambda^{new}$.
\end{enumerate}
\item Update $M[c]$ to $(n^{new}_{count}, p^{new}_{\leq \xi}, x^{new}_{\oplus})$
\item If $n^{new}_{count} = 1$, decrease $q^{\text{marked-cell}}$ by one.
\end{enumerate}

\paragraph{Time and memory usage.}
It is obvious that a lookup takes $\poly(d, L, \log n)$ time. For an insertion or a deletion, recall that the CVP algorithm and the list-decoding algorithm run in time $2^{O(d)}\poly(L, \log n)$. Furthermore, from the list size bound, $C_{\text{close}}$ is of size at most $2^{O(d)}$, which means that we only invoke at most $2^{O(d)}$ lookups and updates of the map $M$. As a result, from the running time guarantee in Theorem~\ref{thm:dynamic-map}, we can conclude that the total runtime for each update is only $2^{O(d)} \poly(L, \log n)$.

\paragraph{Correctness.} It is simple to verify that the claimed invariants hold. Notice also that these invariants completely determine $p^{\text{total}}_{\leq \xi}, q^{\text{marked-cell}}$ and $M$ based on the current point set $S$ alone (regardless of the history). As a result, from the history-independence of $\cH$, we can conclude that our data structure is also history-independent.
\end{proof}

%\subsection{Implication on Quantum Algorithm for \cp}
%\label{subsec:quantum-cp}

\end{document}